\documentclass{article}

    \PassOptionsToPackage{numbers, compress}{natbib}

\usepackage[final]{neurips_2022}

\usepackage[utf8]{inputenc} 
\usepackage[T1]{fontenc}    
\usepackage{hyperref}       
\usepackage{url}            
\usepackage{booktabs}       
\usepackage{amsfonts}       
\usepackage{nicefrac}       
\usepackage{microtype}      
\usepackage{xcolor}         

\usepackage{subcaption}

\usepackage{cite}
\usepackage{amsmath,amssymb,amsfonts}
\usepackage{graphicx}
\usepackage{textcomp}

\def\BibTeX{{\rm B\kern-.05em{\sc i\kern-.025em b}\kern-.08em
    T\kern-.1667em\lower.7ex\hbox{E}\kern-.125emX}}

\usepackage[toc,page,header]{appendix}
\usepackage{minitoc}

\usepackage{dsfont,empheq}
\usepackage{mathrsfs}
\usepackage{algorithm}
\usepackage{algpseudocode}
\usepackage{algorithmicx}
\usepackage[textsize=tiny]{todonotes}
\usepackage{mathtools}
\usepackage{syntonly,bm,wrapfig}
\usepackage{cases,balance}

\usepackage{cleveref}

\usepackage{makecell}

\newcommand\numberthis{\addtocounter{equation}{1}\tag{\theequation}}

\def \bX {{\mathbf{X}}}
\def \bx {{\mathbf{x}}}
\def \bZ {{\mathbf{Z}}}

\def \bV {{\mathbf{V}}}
\def \bu {{\mathbf{u}}}
\def \bW {{\mathbf{W}}}
\def \bA {{\mathbf{A}}}
\def \bB {{\mathbf{B}}}
\def \bC {{\mathbf{C}}}
\def \bD {{\mathbf{D}}}
\def \bE {{\mathbf{E}}}
\def \bI {{\mathbf{I}}}
\def \bP {{\mathbf{P}}}
\def \bLam {{\mathbf{\Lambda}}}

\usepackage{amsthm,thmtools, thm-restate}
 \declaretheorem[name=Proposition]{proposition}

\newtheorem{example}{Example}
\newtheorem{theorem}{Theorem}
\newtheorem{lemma}{Lemma}
\newtheorem{remark}{Remark}
\newtheorem{corollary}{Corollary}
\newtheorem{definition}{Definition}
\newtheorem{assumption}{Assumption}
\newtheorem*{assumption*}{Assumption}

\usepackage{float}
\usepackage{placeins}
\usepackage{multirow}

\usepackage[most]{tcolorbox}
\usepackage{pgf}

\graphicspath{{./figures/}}

\newcommand{\emphblockoption}{drop shadow,
    colframe=black!60,
    colback=black!10,
    coltitle=white!, 
    left=.2pt,
    right=.2pt,
    boxrule=1pt,
    arc=1pt}
\tcbset{emphblock/.style={code={\pgfkeysalsofrom{\emphblockoption}}}}

\title{Understanding Benign Overfitting in \\
Gradient-Based Meta Learning}

\author{%
\begin{minipage}[t]{0.33\textwidth}
\centering
   Lisha Chen \\
   \textnormal{Rensselaer Polytechnic Institute \\
   Troy, NY, USA} \\
   \texttt{chenl21@rpi.edu} 
\end{minipage}  
\begin{minipage}[t]{0.35\textwidth}
\centering
   Songtao Lu \\
   \textnormal{IBM Research   \\
   Yorktown Heights, NY, USA} \\
   \texttt{songtao@ibm.com} 
\end{minipage}  
\begin{minipage}[t]{0.33\textwidth}
\centering
   Tianyi Chen \\
   \textnormal{Rensselaer Polytechnic Institute \\
   Troy, NY, USA} \\
   \texttt{chentianyi19@gmail.com} 
\end{minipage}  
}

\begin{document}

\maketitle
\doparttoc 
\faketableofcontents 

\begin{abstract}
  Meta learning has demonstrated tremendous success in few-shot learning with  limited supervised data. In those settings, the meta model is usually overparameterized. While the conventional statistical learning theory suggests that overparameterized models tend to overfit, empirical evidence reveals that overparameterized meta learning methods still work well -- a phenomenon often called ``benign overfitting.'' To understand this phenomenon, we focus on the meta learning settings with a challenging bilevel structure that we term the gradient-based meta learning, and analyze its generalization performance under an overparameterized meta linear regression model. While our analysis uses the relatively tractable linear models, our theory contributes to understanding the delicate interplay among data heterogeneity, model adaptation and benign overfitting in gradient-based meta learning tasks. We corroborate our theoretical claims through numerical simulations. 
\end{abstract}

\section{Introduction} 
\label{sec:introduction}

\begin{wrapfigure}{R}{0.43\textwidth}
  \vspace{-0.2cm}
   \begin{minipage}{0.42\textwidth}
    \centering
    \includegraphics[width=0.95\textwidth]{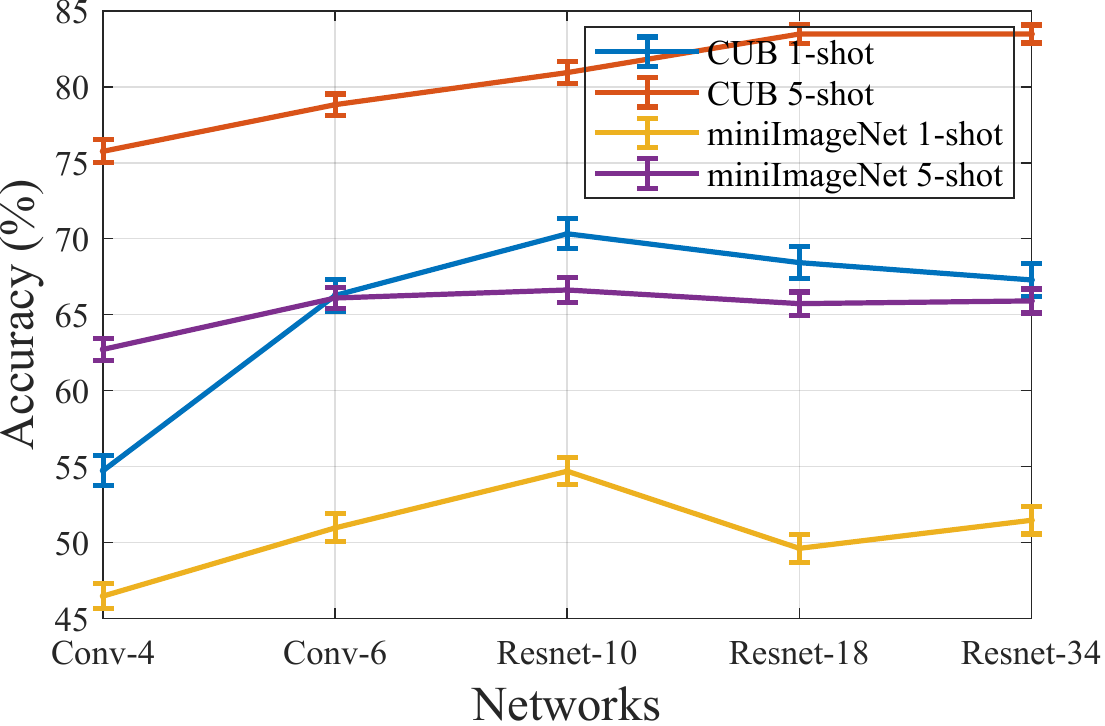}
    \caption{\small Accuracy vs networks with increasing dimensions for MAML on few-shot image classification with different datasets~\citep{chen2018closer}.}
    \label{fig:MAML_acc_vs_depth}
   \end{minipage}
\end{wrapfigure}
Meta learning, also referred to as ``learning to learn'', usually learns a prior model from multiple tasks so that the learned model is able to quickly adapt to unseen tasks \citep{schmidhuber1993_recurrent,hochreiter2001_l2l}. 
Meta learning has been successfully applied to few-shot learning   learning~\citep{andrychowicz2016_l2l,chen2017_l2l}, image recognition \citep{zoph2018learning}, federated learning \citep{jiang2019improving}, reinforcement learning \citep{Finn2017_maml} and communication systems \citep{chen2022fnt}. 
While there are many exciting meta learning methods today, in this paper, we will study a representative meta learning setting where the goal is to learn a shared initial model that can quickly adapt to task-specific models. 
This adaptation may take an explicit form such as the output of one gradient descent step, which is referred to as the model agnostic meta learning (MAML) method~\citep{Finn2017_maml}. Alternatively, the adaptation step may take an implicit form such as the solution of another optimization problem, which is referred to as the implicit MAML (iMAML) method \citep{rajeswaran2019_imaml}. Since both MAML and iMAML will solve a bilevel optimization problem, we term them the \emph{gradient-based meta learning} thereafter. 
In many cases, overparameterized models are  used as the initial models in meta learning for quick adaptation. For example, Resnet-based MAML models typically have around $6$ million parameters, but are trained on $1$-$3$ million meta-training data \citep{chen2018closer}.
Training such initial models is often difficult in meta learning because the number of training data is much smaller than the dimension of the model parameter.

Previous works on meta learning mainly focus on addressing the optimization challenges or analyzing the generalization performance with sufficient data~\citep{fallah2020_convergence_maml,fallah2021_generalization_unseen,chua2021fine}. 
Different from these works, we are particularly interested in the generalization performance of the sought initial model in practical scenarios where the total number of data from all tasks is \emph{smaller than} the dimension of the initial model, which we term  \emph{overparameterized meta learning}.
Empirical studies have demonstrated that the  MAML with overparameterized models generally perform better than MAML with underparameterized models~\citep{chen2018closer} -- a phenomenon often called ``benign overfitting.'' 
 To show this, we plot in Figure~\ref{fig:MAML_acc_vs_depth} the empirical results of MAML from Table~A5 in~\citep{chen2018closer}. Resnets are overparameterized, and Convnets are underparameterized in the meta learning settings. MAML with Resnets generally perform better than MAML with Convnets on different datasets.
 However, in those overparameterized regimes, the generalization error and benign overfitting condition of gradient-based meta learning models are not fully understood. 
Motivated by this, we ask:
\begin{center}
  \emph{
  If and when overparameterized MAML models lead to benign overfitting, provably?
  }
\end{center}
Complementing the empirical observations, we take an initial step by answering this theoretical question in the meta linear regression setting.

\subsection{Prior art} 
\label{sec:related_work}
We review prior art that we group in the following three categories.

\paragraph{Benign overfitting analysis.}
The empirical success of overparameterized deep neural networks has inspired theoretical studies of overparameterized learning.
The most closest line of work is \emph{benign overfitting} in linear regression~\citep{Bartlett_benign_linear}, which provides excess risk
that measures the difference between expected population risk of the empirical solution and the optimal population risk.
Analysis of  overparameterized linear regression model with the minimum-norm solution. It concludes that certain data covariance matrices lead to benign overfitting, explaining why overparameterized models that perfectly fit the noisy training data can work well during testing.
The analysis has been extended to ridge regression~\citep{Tsigler2020Benign_ridge}, multi-class classification~\citep{wang2021benign}, and  adversarial learning with linear models~\citep{chen2021benign_adverse}.
While previous theoretical efforts on benign overfitting largely focused on linear models,
most recently, the analysis of benign overfitting has been extended to two-layer neural networks~\citep{cao2022benign_nn,li2021towards_nn,frei2022benign}. 
However, existing works mainly study benign overfitting for empirical risk minimization problems, rather than bilevel problems such as gradient-based meta learning, which is the focus of this work. 

\paragraph{Meta learning.}
Early works of meta learning build black-box recurrent models that can make predictions based on a few examples from new tasks~\citep{schmidhuber1993_recurrent,hochreiter2001_l2l,andrychowicz2016_l2l,chen2017_l2l}, or learn shared feature representation among multiple tasks~\citep{snell2017prototypical,vinyals2016matching}. 
More recently, meta learning approaches aim to find the initialization of model parameters that can quickly adapt to new tasks with a few number of  optimization steps such as MAML 
~\citep{Finn2017_maml,nichol2018first,rothfuss2018promp}.
The empirical success of meta learning has also stimulated recent interests on building the theoretical foundation of meta learning methods.

\paragraph{Generalization of meta learning.}
The \emph{excess risk},  {as a metric of generalization ability} of gradient-based meta learning has been analyzed recently~\citep{denevi2018_l2l_linear_centroid,bai2021_trntrn_trnval,chen2022_bamaml,wang2020_global_converge_maml_dnn,balcan2019_gbml_online_convex,fallah2021_generalization_unseen}.
The generalization of meta learning has been studied in \citep{kong2020meta} in the context of mixed linear regression, where the focus is on investigating when abundant tasks with small data can compensate for lack of tasks with big data. 
Generalization performance has also been studied in a relevant but different setting - representation based meta learning ~\citep{chua2021fine,du2020few}.
Information theoretical  bounds have been proposed in~\citep{jose2021information,chen2021generalization}, which bound the generalization error in terms of mutual information between the input training data and the output of the meta-learning algorithms.
The PAC-Bayes framework has been extended to meta learning to provide a PAC-Bayes meta-population risk bound~\citep{amit2018_pac_bayes_ml_prior,rothfuss2021_pacoh_pac_bayes_ml,ding2021bridging,farid2021generalization}. 
These works mostly focus on the case where the meta learning model is underparameterized; that is, the total number of meta training data from all tasks is larger than the dimension of the model parameter. 
Recently, \emph{overparameterized} meta learning has attracted more attention.  Bernacchia~\citep{bernacchia2020meta} suggests that in overparameterized MAML, negative learning rate in the inner loop is optimal during meta training for linear models with Gaussian data.
Sun et al.~\citep{sun2021towards} shows that the optimal representation in representation-based meta learning is overparameterized and provides sample complexity for the method of moment estimator.  Our work and a concurrent work \citep{huang2022_overpara_maml_sgd} study a common setting where the meta learning models incur overparameterization in the meta level, and we both cover the nested MAML method. However, the two studies differ in terms of how the empirical solution of the meta parameter is obtained. In our case, we consider the minimum $\ell$-2 norm solution, while \citep{huang2022_overpara_maml_sgd} consider the solution trained with $T$-step stochastic gradient descent (SGD). Furthermore, our analysis covers both MAML and iMAML, while \citep{huang2022_overpara_maml_sgd} only considers MAML.

Our work differs with the most relevant works in the following aspects. 
Compared to the works that also analyze generalization error or sample complexity in linear meta learning models such as~\citep{denevi2018_l2l_linear_centroid,bai2021_trntrn_trnval,chen2022_bamaml}, we focus on the overparameterized case when the total number of training data is smaller than the dimension of the model parameter.
Compared to the work that focus on \emph{representation-based} meta learning with a bilinear structure~\citep{sun2021towards}, we consider initialization-based meta learning methods with a bilevel structure such as MAML and iMAML. 
Furthermore, we provide tight analysis of the excess risk and explicitly consider the benign overfitting condition.

A summary of key differences compared to prior art is provided in Table~\ref{tab:compare_prior}. 
We distinguish two different overparameterization settings: i) the \emph{per-task level overparameterization} where the dimension of model parameter is larger than the number of training data per task, but smaller than the total number of data across all tasks; and, 
ii) the \emph{meta level overparameterization} where the dimension of model parameter is larger than the total number of training data from all tasks.

\begin{table}[bt]
  \vspace{-0.4cm}
\small
  \caption{A comparison with closely related prior work on meta learning  {with linear models}. ``Reps.'' and ``Gradient'' refer to representation based methods  and gradient-based methods; ``Per-task'' refers to the per-task level overparameterization and ``Meta'' refers to the meta level overparameterization. 
}
  \label{tab:compare_prior}
  \centering
  \begin{tabular}{c c c c c c c}
  \toprule
   Prior work &\multicolumn{2}{c}{\!\!\!\!\!\!Type of meta learning} & \multicolumn{2}{c}{\!\!\!Overparameterization\!\!\!} & Methods & Focus of analysis \\
   &Reps. &\!\!\!Gradient &Per-task &\!\!\!Meta &\\
  \midrule 
  Bai et al.~\citep{bai2021_trntrn_trnval} 
  & &\checkmark &\checkmark 
  & & iMAML & Train-validation split\\
  Bernacchia~\citep{bernacchia2020meta} 
  & &\checkmark & & \checkmark 
  & MAML & Optimal step size \\
  Chen et al.~\citep{chen2022_bamaml} 
  & -  & \checkmark &\checkmark & 
  & MAML, BMAML &\!\!\!\!  {Test risk comparison} \!\!\!\! \\
  Huang et al.~\citep{huang2022provable}    
  & &\checkmark & &\checkmark 
  & MAML &\!\!\!\! Excess risk of SGD solution \!\!\\
  Kong et al.~\citep{kong2020meta} 
  & -  &  - &\checkmark & 
  & - &\!\!\!\!  {Effect of small data tasks }\!\!\!\! \\
  \!\!\!Saunshi et al.~\citep{saunshi2021representation}  
  &\checkmark & &\checkmark 
  & &- & Train-validation split\\
  Sun et al.~\citep{sun2021towards}    
  &\checkmark& & &\checkmark 
  & - &\!\! Optimal representation \!\!\\
  \hline
  Ours & &\checkmark & & \checkmark 
  & \!\!\!MAML, iMAML \!\!\!& Benign overfitting\\
  \bottomrule
  \end{tabular}
  \vspace{-0.5cm}
\end{table}

\subsection{This work}
This paper provides a unifying analysis of the generalization performance for meta learning problems with overparameterized meta linear models. 
To our best knowledge, 
this is the first work that analyzes the condition of benign overfitting for gradient-based meta learning including MAML and iMAML.

\textbf{Technical challenges.} Before we introduce the key result of our paper, we first highlight the challenges of analyzing the generalization of gradient-based meta learning and characterizing its benign overfitting condition, compared to the non-bilevel setting such as in~\citep{Bartlett_benign_linear,Tsigler2020Benign_ridge,sun2021towards}. 

\textbf{T1)} Due to the bilevel structure of gradient-based meta learning, the solution to the meta training objective involves polynomial functions of data covariance. 
As a result, the dominating term in the excess risk propagated from the label noise contains higher order moment terms, which is harder to quantify and can potentially lead to much higher excess risk than the linear regression case \citep{Bartlett_benign_linear,Tsigler2020Benign_ridge,sun2021towards}.

\textbf{T2)} 
The existing analysis of benign overfitting in single-level problems~\citep{Bartlett_benign_linear,Tsigler2020Benign_ridge} has a solution that is directly related to the data covariance matrix. However, due to the nested structure of gradient-based meta learning and thus the solution matrix, the solution matrix is a function of both the data covariance matrix and the hyperparameters such as the step size.
Therefore, what kind of data matrices can satisfy the benign overfitting condition cannot be directly implied.

\textbf{T3)} Due to the multi-task learning nature of meta learning, the excess risk of MAML depends on the  heterogeneity across different tasks in terms of both the task data covariance and the ground truth task parameter. 
As a result, the data covariance matrices from different tasks have different eigenvectors.
This is in contrast to the linear regression case where all the data follow the same distribution.

\textbf{Contributions.}
In view of challenges, our contributions can be summarized as follows.
\vspace{-0.1cm}
\begin{itemize}
  \item[\bf C1)] 
Focusing on the relatively tractable linear models, we derive the excess risk for the minimum-norm solution to overparameterized gradient-based meta learning including MAML and iMAML.
  Specifically, the excess risk upper bound adopts the following form
  \begin{align*}
    &\text{Cross-task variance } 
    +\text{Per-task variance }
    +\text{Bias }
  \end{align*}
  where the \emph{cross-task variance} quantifies the error caused by finite task number and the variation of the ground truth task specific parameter, which is a unique term compared to single task learning.
  The \emph{bias} quantifies the bias resulting from the minimum $\ell$-2 norm solution.
  And the \emph{per-task variance} quantifies the error caused by noise in the training data. 
  \item[\bf C2)] We compare the benign overfitting condition for the overparameterized gradient-based meta learning models and that for the empirical risk minimization (ERM) which learns a single shared parameter for all tasks. We show that overfitting is more likely to happen in MAML and its variants such as iMAML than in ERM.
  In addition, larger data heterogeneity across tasks will make overfitting more likely to happen.
  \item[\bf C3)] We discuss the choice of hyperparameter, e.g., the step size in MAML and the weight of the regularizer in iMAML, such that if the data leads to benign overfitting in ERM, it also leads to benign overfitting in MAML and iMAML.
\end{itemize}

\section{Problem Formulation and Methods} 
\label{sec:problem_definition_and_solutions}

In this section, we will introduce the problem setup and the considered   meta learning methods. 

\textbf{Problem setup.} 
In the meta-learning setting, 
assume task $m$ is drawn from a task distribution, i.e. $m \sim \mathcal{M}$.
For each task $m$, we observe $N$ samples with input feature ${x}_{m} \in \mathcal{X}_{m} \subset \mathbb{R}^d$ and target label $y_{m} \in \mathcal{Y}_{m} \subset \mathbb{R}$ drawn i.i.d. from a task-specific data distribution $\mathcal{P}_{m}$. These samples are collected in the dataset
$\mathcal{D}_{m} = \{({x}_{m,n},y_{m,n})\}_{n=1}^N$, 
which is divided into the train and validation datasets, denoted as $\mathcal{D}_{m}^{\rm tr}$ and $\mathcal{D}_{m}^{\rm va}$. And $|\mathcal{D}_{m}^{\rm tr}|=N_{\rm tr}$ and $|\mathcal{D}_{m}^{\rm va}|=N_{\rm va}$ with $N=N_{\rm tr}+N_{\rm va}$.
We use the empirical loss  ${\ell_{m}}(\theta_{m}, \mathcal{D}_{m})$ of per-task parameter $\theta_{m} \in {\Theta}_{m}$ as a measure of the performance. In this paper, we consider regression problems, where $\ell_{m}$ is defined as the mean squared error.

The goal for gradient-based meta learning methods, such as MAML~\citep{Finn2017_maml} and iMAML~\citep{rajeswaran2019_imaml},
is to learn an initial parameter $\theta_0 \in {\Theta}_0$,
which, with an adaptation method ${\mathcal{A}}: {\Theta}_0 \times (\mathcal{X}_{m}\times \mathcal{Y}_{m})^{N_{\rm tr}} \rightarrow {\Theta}_{m}$, can generate a per-task parameter $\theta_{m}$ that performs well on the validation data for task $m$.
Given $M$ tasks, our meta-learning objective is computed as the average of the per-task objective, given by
\begin{equation}\label{eq:emp_loss}
\text{Meta training objective} ~~~~~~~~~~  {\mathcal{L}}^{\cal A}(\theta_0, \mathcal{D})
\coloneqq 
\frac{1}{M} \sum_{m=1}^{M} \ell_{m}({\mathcal{A}}(\theta_0,\mathcal{D}_{m}^{\rm tr}),\mathcal{D}_{m}^{\rm va}).
\end{equation}
Obtaining the empirical solution $\hat{\theta}_0^{\mathcal{A}}$ by minimizing~\eqref{eq:emp_loss} under a meta learning method $\mathcal{A}$, in the meta testing stage, we evaluate  $\hat{\theta}_0^{\mathcal{A}}$ on  the population risk, given by
\begin{equation}
\label{eq:R}
\text{Meta testing objective}~~~ ~~~\mathcal{R}^{\cal A}(\hat{\theta}_0^{\mathcal{A}})
\coloneqq  
\mathbb{E}_{m } 
\left[ \mathbb{E}_{\mathcal{D}_{m}} \big[\ell_{m}({\mathcal{A}}(\hat{\theta}_0^{\mathcal{A}},\mathcal{D}_{m}^{\rm tr}), \mathcal{D}_{m}^{\rm va})\big]\right]. 
\end{equation}

\begin{wrapfigure}{R}{0.3\textwidth}
\vspace{-0.2cm}
 \begin{minipage}{0.29\textwidth} 
  \centering
  \includegraphics[width=0.95\textwidth]{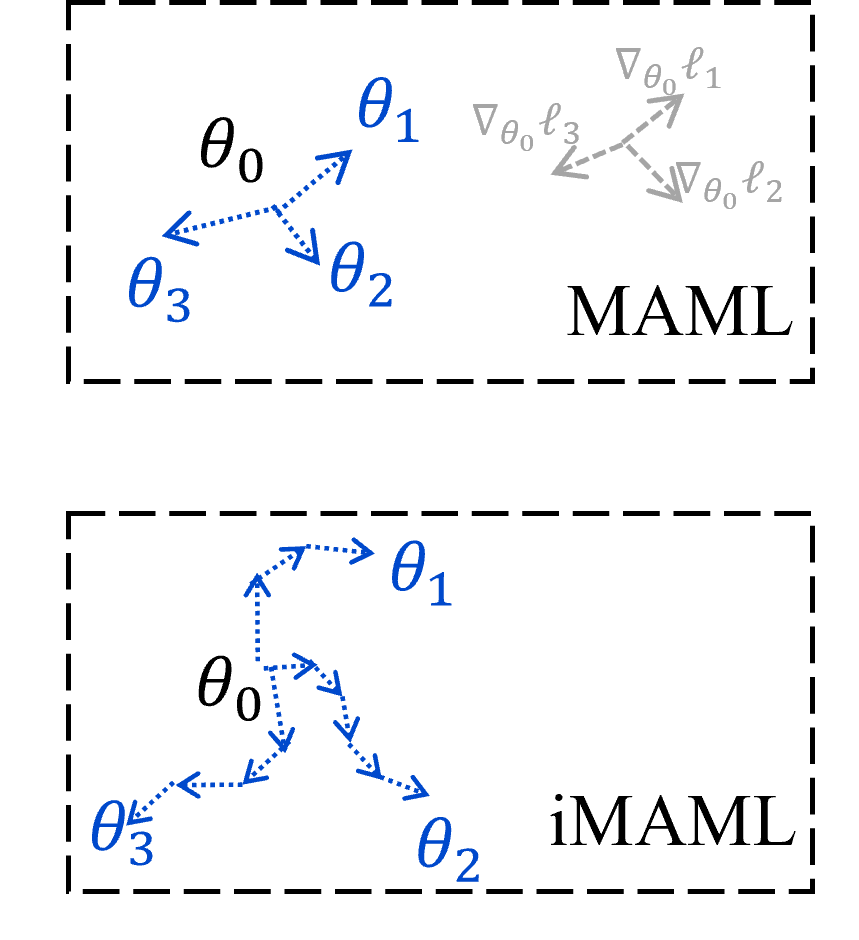}
  \caption{\small Two types of meta learning.}
  \label{fig:maml_imaml}
\end{minipage}
\end{wrapfigure}
\textbf{Methods.} 
We focus on understanding the generalization of two representative gradient-based meta learning methods MAML~\citep{Finn2017_maml} and iMAML~\citep{rajeswaran2019_imaml} in the overparameterized regime. 
MAML obtains the task-specific parameter $\hat{\theta}_{m}(\theta_0)$  by taking one step gradient descent with step size $\alpha$ of the per-task loss function $\ell_{m}$ from the initial parameter $\theta_0$, that is
\begin{equation}\label{eq:MAML_problem}
  {\mathcal{A}}(\theta_0,\mathcal{D}_{m}^{\rm tr}) = \theta_0  - \alpha   \nabla_{\theta_0} {\ell}_{m}(\theta_0,\mathcal{D}_{m}^{\rm tr}).
\end{equation}
On the other hand, iMAML obtains  
the task-specific parameter $\hat{\theta}_{m}$ from the initial parameter $\theta_0$ by optimizing the task-specific loss regularized by the distance between $\hat{\theta}_{m}$ and $\theta_0$, that is
\begin{equation}
\label{eq:iMAML_problem}
  {\mathcal{A}}(\theta_0,\mathcal{D}_{m}^{\rm tr}) 
  =\underset{\theta}{\arg \min }\, 
  ~\ell_{m} (\theta, \mathcal{D}_{m}^{{\rm tr}})+\frac{\gamma}{2}\|\theta-\theta_0\|^{2}  
\end{equation}
where $\gamma>0$ is the weight of the regularizer. 
As summarized in Figure \ref{fig:maml_imaml}, MAML has smaller computation complexity than iMAML
since iMAML requires solving an inner problem during adaptation, while iMAML may achieve smaller test error  since it explicitly minimize the loss.

\section{Main Results: Benign Overfitting for Gradient-based Meta Learning} 
\label{sec:benign_overfitting_analysis}

In this section, we introduce the meta linear regression model and some necessary assumptions for the analysis. We present the main results, highlight the key steps of the proof and conduct simulations to verify our results. Due to space limitations, we will defer  the proofs to the supplementary document.

\subsection{Meta linear regression setting}
To make a precise analysis, we will assume the following linear regression data model.
Denoting the ground truth parameter on task $m$ as $\theta_{m}^{\star} \in \mathbb{R}^d$, 
and the noise as $\epsilon_m$,
we assume the data model for task $m$ is 
\begin{equation}\label{eq:linear_data_generate}
  y_{m}={\theta}^{\star \top}_{m} {x}_{m}+\epsilon_{m}.
\end{equation}

Given the linear model~\eqref{eq:linear_data_generate}, the meta training problem~\eqref{eq:emp_loss} with adaptation method \eqref{eq:MAML_problem} or \eqref{eq:iMAML_problem} generally have unique solutions when $d \leq NM$.
However, when the meta model $\theta_{0}$ and thus the per-task model $\theta_m$ are overparameterized, i.e. $d > NM$, the training problem~\eqref{eq:emp_loss} may have multiple solutions. In the subsequent analysis, we will analyze the performance of the \emph{minimum norm} solution because recent advances in training overparameterized models reveal  that gradient descent-based methods converge to the minimum norm  solution~\citep{gunasekar18a,Vidya2020harmless_interpolation}. We provide a formal definition below.

\begin{definition}[Minimum $\ell$-2 norm solution]
  \label{def:mne}
  Denote 
$
\mathbf{X}_{m}^{\rm va} := 
[{x}_{m, 1},\ldots,
{x}_{m,N_{\rm va}}]^{\top} \in \mathbb{R}^{N_{\rm va}\times d}$, 
$\mathbf{y}_{m}^{\rm va} := 
[{y}_{m, 1},\ldots,
{y}_{m,N_{\rm va}}]^{\top} \in \mathbb{R}^{N_{\rm va}} 
$. 
With $\mathcal{A}(\theta, \mathcal{D}_m^{\rm tr})$ being either \eqref{eq:MAML_problem} or \eqref{eq:iMAML_problem}, the minimum norm  solution to the meta training problem~\eqref{eq:emp_loss} under the linear regression loss is expressed by 
\begin{align*}
\label{eq:mne}
\hat{\theta}_{0}^{\mathcal{A}} 
\coloneqq \mathop{\arg\min}_{\theta_{0}} 
\|\theta_{0}\|^2
\quad 
 {\rm s.t. }~~~
 & \theta_{0}~
 \in
 \mathop{\arg\min}_{{\theta}} 
 {\mathcal{L}}^{\cal A}(\theta, \mathcal{D})
 = \frac{1}{M} \sum_{m=1}^{M} \left\|\mathbf{X}_{m}^{\rm va} \mathcal{A}(\theta, \mathcal{D}_m^{\rm tr})-\mathbf{y}_{m}^{\rm va}\right\|^{2}.~~~
\numberthis
\end{align*}
\end{definition}

In our analysis, we make the following basic assumptions.
\begin{assumption}[Overparameterized model]\label{assmp:overpara} 
  The total number of meta training data is smaller than the dimension of the model parameter; i.e. $NM < d$.
\end{assumption}
\begin{assumption}[SubGaussian data]\label{assmp:subGaussian_data} 
The noise $\epsilon_m$ is subGaussian with $\mathbb{E} [\epsilon_m] = 0$ and $ \mathbb{E}[\epsilon_m ^2] = \sigma^2$. For the $m$-th task, data ${x}_m = \mathbf{V}_m \mathbf{\Lambda}_m^{\frac{1}{2}} \mathbf{z}_m$, where $\mathbf{z}_m$ has centered, independent, $\sigma_x$-subGaussian entries;
$\mathbb{E}[\mathbf{z}_m] = \mathbf{0}, 
\mathbb{E}[\mathbf{z}_m\mathbf{z}_m^{\top}] = \mathbf{I}_d$, with $\mathbf{I}_d$ being a $d\times d$ identity matrix. 
\end{assumption}
\begin{assumption}[Data covariance matrix]\label{assmp:V} 
1) Assume  for all
$ m \in [M], i \in [d], \lambda_{m,i} > 0 $, 
$\mathrm{Tr}(\mathbf{\Lambda}_m), \mathrm{Tr}(\mathbf{\Lambda}) $ are bounded, i.e. for all
$ m \in [M]$, $\mathrm{Tr}(\mathbf{\Lambda}_m) \leq c_{\lambda}$. 
2) Cross-task  data heterogeneity $ \mathbb{V}(\{\mathbf{Q}_m\}_{m=1}^M ) \coloneqq \max_{i,m}| (\lambda_{i} -  \lambda_{m,i})/\lambda_i|$ is bounded above and below. 
\end{assumption}
\begin{assumption}[Task parameter]\label{assmp:task_para} 
The ground truth parameter 
$\theta_{m}^{\star}$ is independent of $\mathbf{X}_{m}$ and satisfies
$\mathrm{Cov}[\theta_m^{\star}] = (R^2/d) \mathbf{I}_d$,
where $R$ is a constant, and the   entries  of $\theta_m^{\star}$ 
are i.i.d.   $\mathcal{O}( R / \sqrt{d})$-subGaussian. 
\end{assumption}

Assumption~\ref{assmp:overpara} defines the setting that the meta level is overparameterized, which has also been used in~\citep{sun2021towards}. 
Note that Assumptions~\ref{assmp:subGaussian_data}-\ref{assmp:subGaussian_data} are common in the analysis of meta learning  in~\citep{denevi2018_l2l_linear_centroid,bai2021_trntrn_trnval,chen2022_bamaml,gao2020_model_opt_tradeoff_ml}.

With the linear data model \eqref{eq:linear_data_generate}, the (minimum norm) solutions to the meta training objective \eqref{eq:emp_loss} and the meta testing objective \eqref{eq:R} can be computed analytically which we will summarize next.

\begin{wraptable}{R}{0.6\textwidth}
\vspace{0.1cm}
 \begin{minipage}{0.6\textwidth} 
  \caption{Weight matrices under different method $\mathcal{A}$.}
  \label{tab:weight_matrices}
  \centering
  \begin{tabular}{ c c }
  \toprule
  Method & Weight matrices  \\
  \hline
  ERM 
  &$\mathbf{W}_{m}^{\mathrm{er}} =  \mathbf{Q}_{m}$ \\
  &$\hat{\mathbf{W}}_{m}^{\mathrm{er}} =  \hat{\mathbf{Q}}_{m}$ \\
  \hline
  MAML 
  & $\mathbf{W}_{m}^{\mathrm{ma}} = (\mathbf{I}-\alpha \mathbf{Q}_{m}) \mathbf{Q}_{m}(\mathbf{I}-\alpha \mathbf{Q}_{m})$ \\
  & $\hat{\mathbf{W}}_{m}^{\mathrm{ma}} 
  = 
  (\mathbf{I}-
  {\alpha}\hat{\mathbf{Q}}_{m}^{\rm tr}) 
  \hat{\mathbf{Q}}_m^{\rm va}
  (\mathbf{I}- 
  {\alpha}\hat{\mathbf{Q}}_{m}^{\rm tr})$ \\
  \hline
  iMAML
  &\!\! $\mathbf{W}_{m}^{\mathrm{im}} = (\gamma ^{-1}\mathbf{Q}_{m}+\mathbf{I})^{-1}\mathbf{Q}_{m}(\gamma ^{-1}\mathbf{Q}_{m}+\mathbf{I})^{-1}$ \!\!\!\\ 
  &\!\! $\hat{\mathbf{W}}_{m}^{\mathrm{im}} 
    = (\gamma ^{-1}\hat{\mathbf{Q}}_{m}^{\rm tr}+\mathbf{I})^{-1}\hat{\mathbf{Q}}_m^{\rm va}(\gamma ^{-1}\hat{\mathbf{Q}}_{m}^{\rm tr}+\mathbf{I})^{-1}$\!\!\!\\
  \bottomrule
\end{tabular}
\end{minipage}
\vspace{-0.2cm}
\end{wraptable}

\begin{restatable}{proposition}{propone}
\label{prop:solutions}
\emph{\textbf{(Empirical and population level solutions)}} Under the meta linear regression model \eqref{eq:linear_data_generate},
the meta testing objective of method $\mathcal{A}$ in  \eqref{eq:R} can be equivalently written as
\begin{align}
   &\mathcal{R}^{\mathcal{A}}({\theta_0} )
    \label{eq:meta_risk_form}
    = \mathbb{E}_{m}\big[\|\theta_0-\theta^{\star}_{m}\|^2_{\mathbf{W}_{m}^{\mathcal{A}}}\big] 
 \end{align} 
 where the matrix $\mathbf{W}_{m}^{\mathcal{A}}$ and its empirical version $\hat{\mathbf{W}}_{m}^{\mathcal{A}}$ are given in Table~\ref{tab:weight_matrices} with $\hat{\mathbf{Q}}_{m}^{\rm al} \coloneqq \frac{1}{N} \mathbf{X}^{\rm al \top}_{m}\mathbf{X}_{m}^{\rm al}$. 
The optimal solutions to the meta-test risk and the minimum-norm solutions to the empirical meta training loss are given below respectively
\begin{subequations}
\begin{align}\label{eq:theta_0_star_A_sln}
  &\theta_{0}^{\mathcal{A}} 
  \coloneqq \mathop{\arg\min}_{\theta_0}
  \mathcal{R}^{\mathcal{A}}({\theta_0} )
  =\mathbb{E}_{m}\big[\mathbf{W}_{m}^{\mathcal{A}}\big]^{-1} \mathbb{E}_{m}\big[\mathbf{W}_{m}^{\mathcal{A}} {\theta}_{m}^{\star} \big] \\
  \label{eq:theta_0_hat_A_sln}
  &\hat{\theta}_{0}^{\mathcal{A}} 
  = \Big({\sum}_{m=1}^{M}
  \hat{\mathbf{W}}_{m}^{\mathcal{A}}\Big)^{\dag}
  \Big({\sum}_{m=1}^{M}\hat{\mathbf{W}}_{m}^{\mathcal{A}}\theta_{m}^{\star} 
  \Big)
  + \Delta_{M}^{\mathcal{A}} 
\end{align}
\end{subequations}
where $^{\dag}$ denotes the Moore-Penrose pseudo inverse; $\Delta_{M}^{\mathcal{A}}$ is an error term that depends on $\mathbf{X}_m, \epsilon_m$, and specified in the supplementary document.  
\end{restatable}

To study overfitting in the meta learning model, we quantify its generalization ability via the widely used metric - \emph{excess risk}.
The excess risk of method $\mathcal{A}$ (which can be ``$\rm ma$'' for MAML and ``$\rm im$'' for iMAML), with an empirical solution $\hat{\theta}_0^{\cal A}$ and population solution ${\theta}_0^{\cal A}$, is defined as
\begin{align}
  \label{eq:def_excess_risk}
  &\mathcal{E}^{\cal A}(\hat{\theta}_0^{\cal A}) \coloneqq \mathcal{R}^{\cal A}(\hat{\theta}_0^{\cal A}) - \mathcal{R}^{\cal A}({\theta}_0^{\cal A}). 
\end{align}
In \eqref{eq:def_excess_risk}, the excess risk measures the difference between the population risk of the empirical solution, $\hat{\theta}_0$ and the optimal population risk. 
Given total number of training samples $MN$, if $d\rightarrow \infty$, the classic learning theory implies that the excess risk $\mathcal{E}^{\cal A}(\hat{\theta}_0^{\cal A})$ also grows, which leads to overfitting~\citep{hastie2009elements}. 
The larger the excess risk, the further the empirical solution $\hat{\theta}_0^{\cal A}$ is from the optimal population solution $\theta_0^{\cal A}$, indicating more severe \emph{overfitting}.

\subsection{Main results}
With the closed-form solutions given in Proposition~\ref{prop:solutions}, we are ready to bound the excess risk of MAML and iMAML in the overparameterized linear regime. 
For notation brevity, we first introduce some universal constants such as $c_0, c_1, c_2, \dots$, and only present the dominating terms in the subsequent results. The precise presentation of remaining terms are deferred to the supplementary document.

We first decompose the excess risk into three terms in Proposition~\ref{prop:excess_risk}.
\begin{proposition}
  \label{prop:excess_risk}
  Define $\mathbf{W}^{\cal A} 
  \coloneqq  {\mathbb{E}_m[\mathbf{W}_m^{\cal A}]}$. The excess risk of a meta learning method $\mathcal{A}$ can be bounded by 
  \begin{align*}\label{eq:excess_risk_bound}
 &\mathcal{E}^{\cal A} (\hat{\theta}_0^{\cal A})
      \lesssim 
      \mathcal{E}_{\theta_m^{\star}}
      + \mathcal{E}_{\epsilon_m}
      + \mathcal{E}_b 
    \numberthis
  \end{align*}
 where the first term $\mathcal{E}_{\theta_m^{\star}}$  is a function of $\theta_m^{\star}, \theta_0^{\cal A}, \mathbf{W}^{\cal A}, \hat{\mathbf{W}}_m^{\cal A}$, which quantifies
  the weighted variance of the ground truth task specific parameters $\theta_m^{\star}$; the second term $\mathcal{E}_{\epsilon_m}$, as a function of $\epsilon_m$, is the weighted noise variance; and the third term $\mathcal{E}_b$, as a function of $\theta_0^{\cal A}, \mathbf{W}^{\cal A}, \hat{\mathbf{W}}_m^{\cal A}$, is the bias of the minimum-norm solution in overparameterized MAML or iMAML. 
  \end{proposition}
  Based on this decomposition, as we will show in Section \ref{sec.proof-sketch}, the bound of the excess risk can be derived from the bound of these three terms  $\mathcal{E}_{\theta_m^{\star}}, \mathcal{E}_{\epsilon_m^{\star}}, \mathcal{E}_b$, respectively, which gives Theorem~\ref{thm:maml_excess_risk_bound}.

\begin{theorem}[Excess risk bound]
\label{thm:maml_excess_risk_bound}
  Suppose Assumptions~\ref{assmp:overpara}-\ref{assmp:task_para} hold. 
Let $\mu_1(\cdot) \geq \mu_2(\cdot) \dots$ denote the eigenvalues of a matrix in the descending order.
  For the meta linear regression problem with the minimum-norm solution \eqref{eq:mne}, for $0 \leq k \leq d$,
  define the effective ranks as 
\begin{equation}\label{eq.eff-rank}
r_k\left(\bW^{\cal A}\right) \coloneqq 
  \frac{\sum_{i>k} {\mu}_{i}\left(\bW^{\cal A}\right)}{{\mu}_{k+1}\left(\bW^{\cal A}\right)};~~~~~~~~~
    R_k\left({\mathbf{W}}^{\cal A}\right) \coloneqq 
  \frac{\left(\sum_{i>k} {\mu}_{i}({\mathbf{W}}^{\cal A})\right)^{2}}{\sum_{i>k} {\mu}_{i}^{2}\left({\mathbf{W}}^{\cal A}\right)}.
\end{equation}
With the cross-task  data heterogeneity $\mathbb{V}$ defined in Assumption~\ref{assmp:V}, if there exist universal constants $c_1, c_2, c_3 > 1$ such that the effective dimension $k^{*}=\min \{k \geq 0: r_{k}(\bW^{\cal A}) \geq c_1 NM \}$, $c_2 \log(1/\delta) < NM$ and $k^* < NM/c_3$,
then with probability at least $1 - \delta$, the excess risk  satisfies 
\begin{align}\label{eq:maml_excess_risk}
 \!\!\! \mathcal{E}^{\cal A}(\hat{\theta}_0^{\cal A}) \lesssim 
  \|\mathbb{E}[\theta_m^{\star}]\|^2 \|\bW^{\cal A}\| \sqrt{\frac{r_0(\mathbf{W}^{\cal A})} {MN}} 
  +\sigma^{2} \Bigg(\frac{k^{*}}{MN}+\frac{MN}{R_{k^{*}}({\mathbf{W}}^{\cal A} )}\Bigg) \Bigg(1+\mathbb{V}(\{{\mathbf{W}}_m^{\cal A}\}_{m=1}^M) \Bigg).\!
\end{align}
\end{theorem}

Theorem~\ref{thm:maml_excess_risk_bound} provides the excess risk bound via the effective ranks.  In \eqref{eq.eff-rank}, the effective ranks $r_k$ and $R_k$ of a matrix capture the distribution of the eigenvalues of this matrix, and the effective dimension $k^*$ determines the above upper bound by considering the asymmetry of the eigenvalues of the solution matrix. The idea is to choose $k^*$ that makes $R_{k^*}$ large enough and keeps $k^*$ small enough compared to $MN$ so that the variance term of the excess risk is controlled.
For example, $r_0$ is the trace normalized by the largest eigenvalue, which is bounded above by $R_0$. And both $r_0$ and $R_0$
are no larger than the rank of the matrix, and they are equal to the rank only when all non-zero eigenvalues are equal. If the eigenvalues distribute more uniformly, the effective rank will be larger, otherwise smaller.

\begin{remark} 

\emph{1) The definition of effective rank has been also given in~\citep{Bartlett_benign_linear}  but only on the data matrix $\mathbf{Q}$.
And our setting reduces to the single task ERM learning, 
or the linear regression case in~\citep{Bartlett_benign_linear},
when $M=1$, $\theta_m^{\star} = \theta_0$,
$\mathbf{W}_m^{\cal A} = \mathbf{Q}$, which implies that the cross-task variance in \eqref{eq:excess_risk_bound} as well as the data heterogeneity $\mathbb{V}(\cdot)$ reduces to zero.
Accordingly, Theorem~\ref{thm:maml_excess_risk_bound} reduces to Theorem 4 in \citep{Bartlett_benign_linear}.}\\ 
\emph{2) Given Theorem~\ref{thm:maml_excess_risk_bound}, in order to control the excess risk of solution $\hat{\theta}_0^{\cal A}$, we want $r_0(\bW^{\cal A})$ to be small compared to the total number of training samples $MN$, but $r_{k^*}(\bW^{\cal A})$ and $R_{k^*}(\bW^{\cal A})$ to be large compared to $MN$.
In addition, the cross-task heterogeneity $\mathbb{V}$ should be small.
Since for a matrix $\mathbf{W}$, $r_k(\mathbf{W}) \leq R_k(\mathbf{W}) \leq d$, this suggests the model benefits from overparameterization.} 
\end{remark}

Building upon Theorem~\ref{thm:maml_excess_risk_bound}, we now discuss the conditions for ``benign overfitting'', which refers to the situation that overparameterization does not ``harm'' the excess risk, or the excess risk still vanishes when $d > MN$ and $N,M,d$ increase.
\begin{figure*}[t]
\centering
\begin{subfigure}{0.45\textwidth}
\centering
\includegraphics[width=0.9\textwidth]{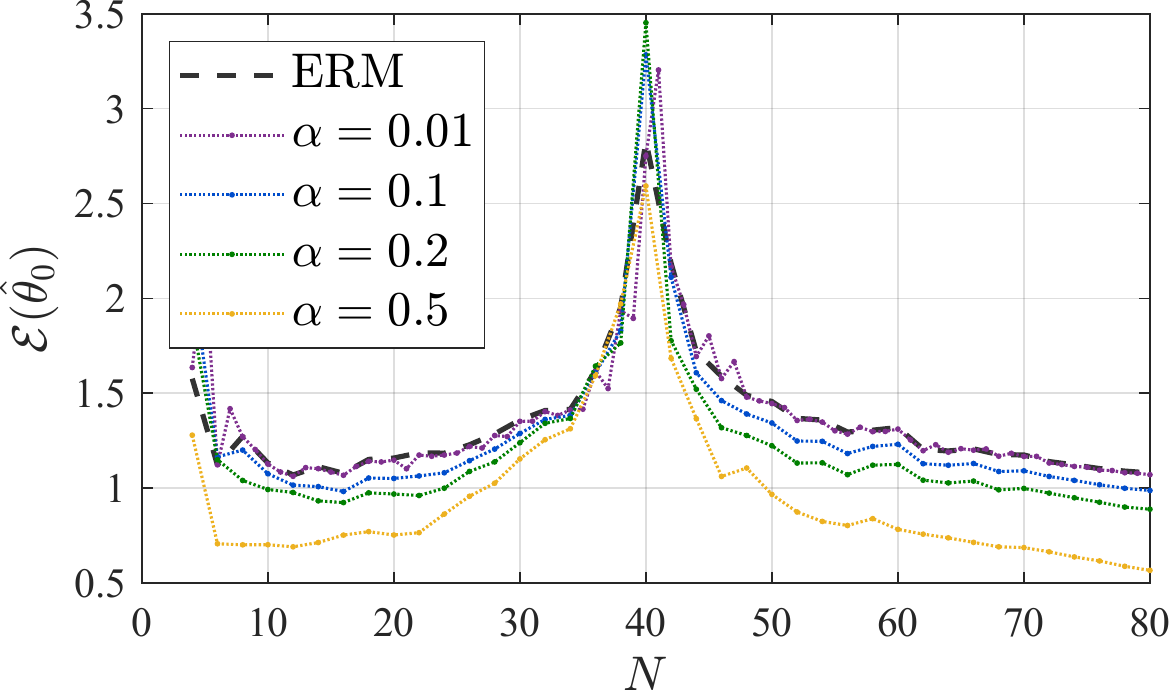}
\vspace{-0.1cm}
\caption{MAML with different $\alpha$. 
}
\end{subfigure}
~~
\begin{subfigure}{0.45\textwidth}
\centering
\includegraphics[width=0.9\textwidth]{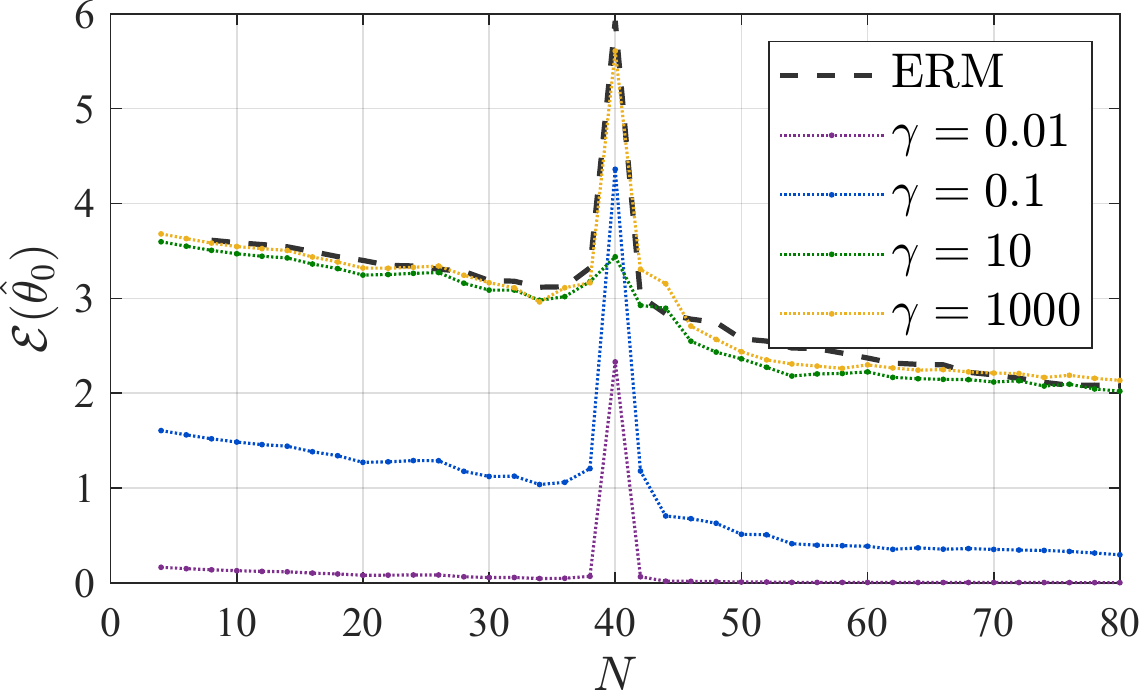}
\vspace{-0.1cm}
\caption{iMAML with different $\gamma$. }
\end{subfigure}
\caption{Excess risk vs number of samples ($N$) with different hyperparameters ($M=10, d=200$). 
}
\label{fig:test_risk_vs_N}
\end{figure*}

\begin{tcolorbox}[emphblock]
\begin{definition}[Benign overfitting condition in meta learning]
Under Assumptions~\ref{assmp:overpara}-\ref{assmp:task_para}, the weight matrices  $\bW^{\cal A}$ for method $\mathcal{A}$ satisfy the \emph{benign overfitting condition} in gradient-based meta learning, if and only if 
\begin{equation}
  \label{eq:benign}
  \lim _{NM,d \rightarrow \infty} \frac{r_{0}(\bW^{\cal A} )}{NM} =
  \lim _{NM,d \rightarrow \infty} \frac{k^{*}}{NM}
  =\lim _{NM,d \rightarrow \infty} \frac{NM}{R_{k^{*}}(\bW^{\cal A}  )} =0.
\end{equation}
\end{definition}
\end{tcolorbox}

This guarantees the excess risk~\eqref{eq:maml_excess_risk} goes to zero in overparameterized meta learning models with sufficient training data from all tasks.
To provide an intuitive explanation,
Figure~\ref{fig:test_risk_vs_N} plots the population risk versus the number of the training data, which demonstrates the ``double descent'' curve. Namely, as $N$ increases,  $\mathcal{E}(\hat{\theta}_0)$ first decreases, then increases and then decreases again, as is discovered in overparameterized neural networks \citep{nakkiran2021deep}. 
The trend in Figure~\ref{fig:test_risk_vs_N} is similar to the trend observed in~\citep{nakkiran2020optimal}.
When $d/(NM) > 1$, the model is overparameterized, which can overfit the training data, leading to larger excess risk as $N$ decreases. However, Figure~\ref{fig:test_risk_vs_N} shows the excess risk does not become too large as $N$ decreases, indicating that overfitting does not severely harm the population risk in this case.
%

\subsection{Examples and discussion} 
\label{sec:experiments}
In this section, we discuss how the benign overfitting condition \eqref{eq:benign} in gradient-based meta learning reduces to
that in single task linear regression; e.g.,  in~\citep{Bartlett_benign_linear,Tsigler2020Benign_ridge}.
We also provide examples to show 
\begin{enumerate}
    \item [\bf Q1)] {\em how certain properties of meta training data affect the excess risk; and,}
    \item [\bf Q2)] {\em how to choose the hyperparameters that preserve benign overfitting.}    
\end{enumerate}

\textbf{Data covariance and cross-task heterogeneity.}
Theorem~\ref{thm:maml_excess_risk_bound} reveals that the excess risk depends on both the eigenvalues of the data covariance matrix $\mathbf{Q}_m$, and the cross-task data heterogeneity, measured by $\mathbb{V}(\{\mathbf{Q}_m\}_{m=1}^M )$.
We give an example below to better demonstrate how these two properties of gradient-based meta training data affect the excess risk. 

\begin{example}[Data covariance]
\label{exmp:simple_eigenvalue}
Suppose $\mathbf{Q}_m=\operatorname{diag}(\mathbf{I}_{d_1}, \beta \mathbf{I}_{d-d_1}),\,\forall m$. 
Set $M=10, d = 200, d_1 = 20, $ $\alpha = 0.1 $ for MAML and $\gamma = 10^3$ for iMAML.
Then the benign overfitting condition \eqref{eq:benign} is satisfied by MAML and iMAML.
We plot the excess risk under different $\beta$
  in Figure~\ref{fig:example_data_covariance}.
\end{example}
\begin{figure}
  \centering
  \begin{subfigure}{0.45\textwidth}
  \centering
  \includegraphics[width=0.9\textwidth]{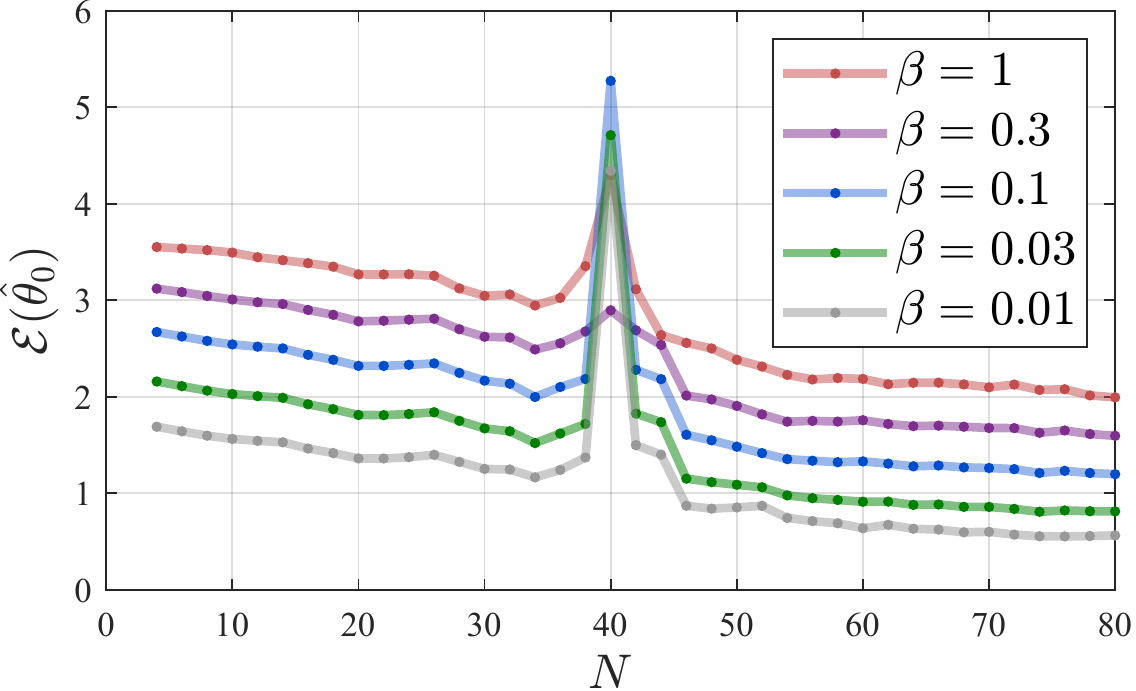}
  \vspace{-0.1cm}
  \caption{MAML}
  \label{fig:exmp1-onestep-maml}
\end{subfigure}
~~
\begin{subfigure}{0.45\textwidth}
  \centering
  \includegraphics[width=0.9\textwidth]{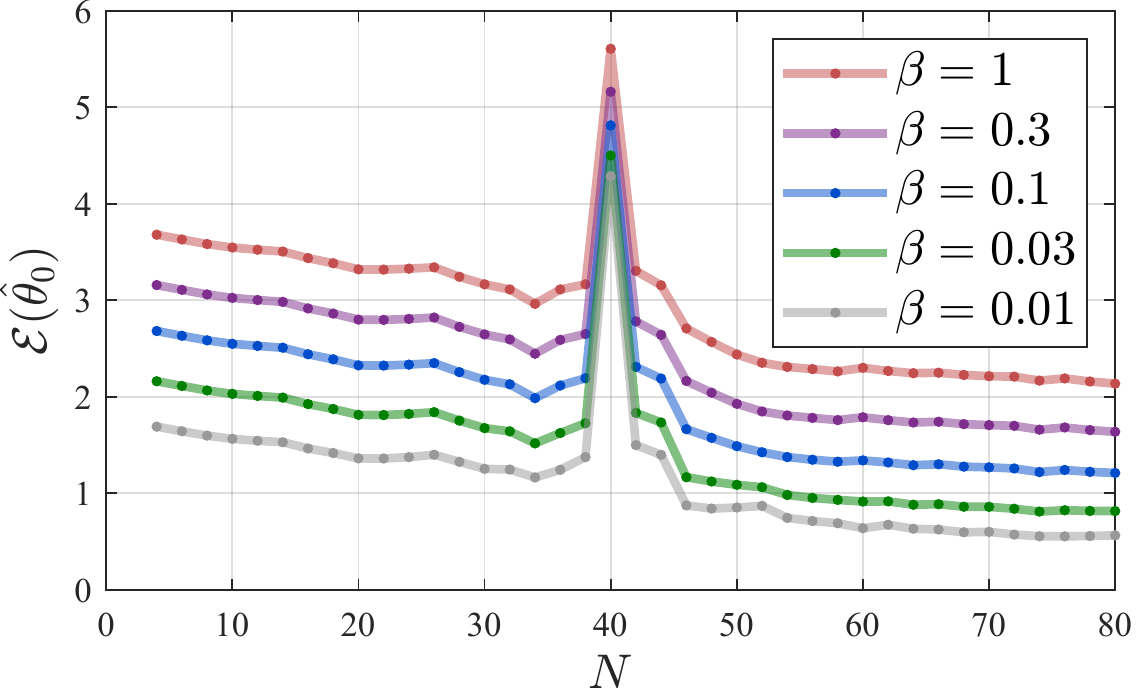}
  \vspace{-0.1cm}
  \caption{iMAML}
  \label{fig:exmp1-imaml}
\end{subfigure}
  \caption{Excess risks vs number of samples ($N$) for  $\mathbf{Q}_m=\operatorname{diag}(\mathbf{I}_{d_1}, \beta \mathbf{I}_{d-d_1})$ with different $\beta$.}
  \label{fig:example_data_covariance}
\end{figure}

\begin{figure}[t]
  \centering
  \begin{subfigure}{0.45\textwidth}
  \centering
  \includegraphics[width=0.9\textwidth]{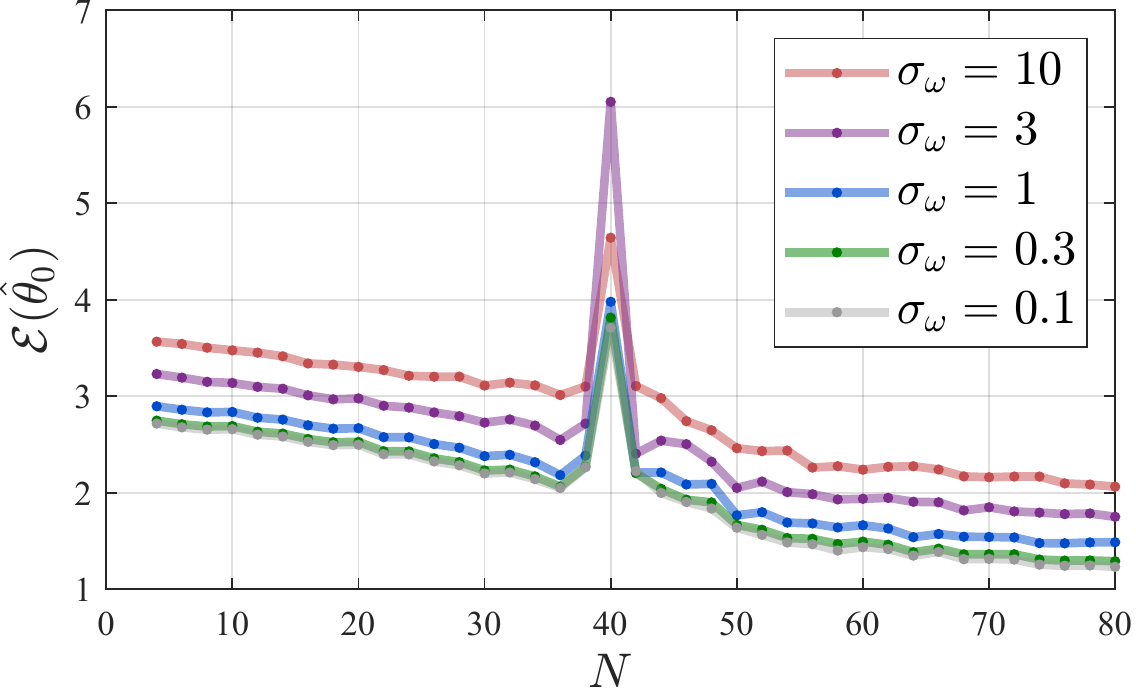}
  \vspace{-0.1cm}
  \caption{MAML.}
  \label{fig:exmp2-maml}
\end{subfigure}
~~
\begin{subfigure}{0.45\textwidth}
\centering
\includegraphics[width=0.9\textwidth]{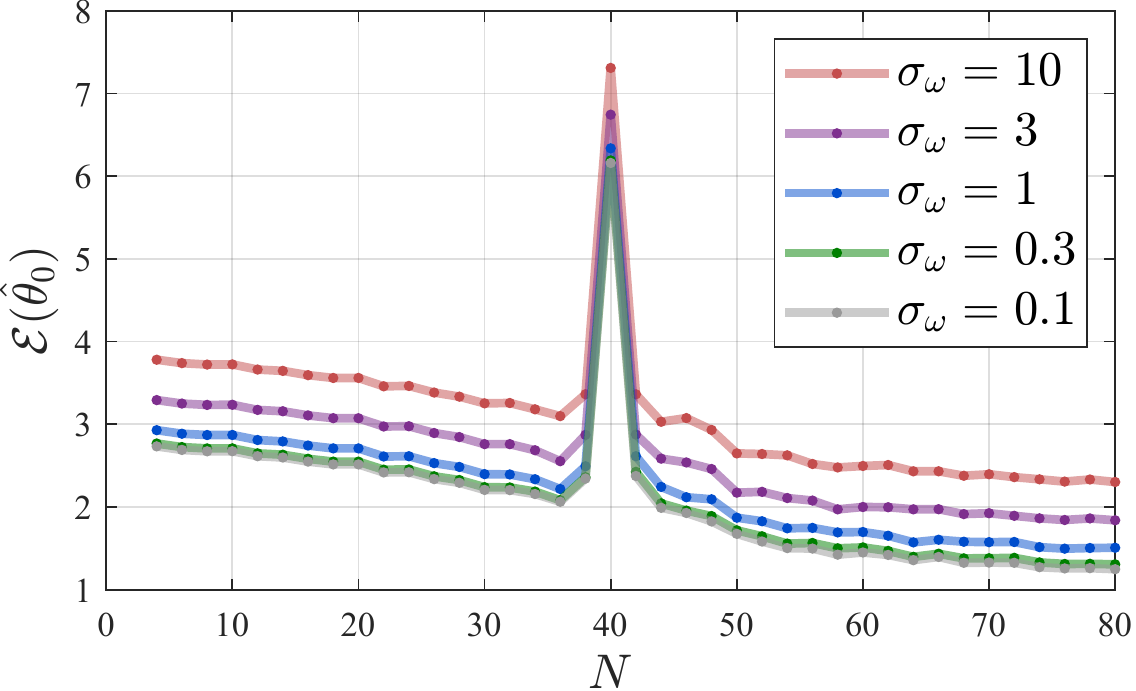}
\vspace{-0.1cm}
\caption{iMAML}
\label{fig:exmp2-imaml}
\end{subfigure}
  \caption{Excess risks of MAML and iMAML vs number of training samples ($N$) for  $\mathbf{Q}_m=|1+\omega_m|\operatorname{diag}(\mathbf{I}_{d_1}, \beta \mathbf{I}_{d-d_1}), \omega_m \sim \mathcal{N}(0, \sigma_{\omega}^2)$ with different $\sigma_{\omega}$.}
  \label{fig:example2_data_covariance}
  \vspace{-0.2cm}
\end{figure}
From Figure~\ref{fig:example_data_covariance} we can  observe that given a fixed number of training data $N$, the excess risk increases with $\beta$ for both MAML and iMAML. This observation verifies our theory since larger $\beta$ results in a smaller $R_k^{\cal A}({\mathbf{W}}^{\cal A} )$, leading to a larger upper bound on the variance term in~\eqref{eq:maml_excess_risk}.

Example~\ref{exmp:simple_eigenvalue} demonstrates how the per-task data matrix $\mathbf{Q}_m$ affects the excess risk.
We consider another example that demonstrates how the data heterogeneity across tasks affects the excess risk.
\begin{example}[Data heterogeneity]
  Suppose $\mathbf{Q}_m= |\omega_m + 1| \operatorname{diag}(\mathbf{I}_{d_1},  \beta \mathbf{I}_{d-d_1}) $ with $\omega_m \sim \mathcal{N}(0, \sigma_{\omega}^2)$  for all $m$. 
Set $M=10, d=200, d_1 = 20, \beta = 0.3, $ $\alpha = 0.1$ for MAML and $\gamma = 0.1$ for iMAML.
Then it satisfies the benign overfitting condition \eqref{eq:benign} for MAML and iMAML.
Figure~\ref{fig:example2_data_covariance} plots the excess risk with different choices of $\sigma_{\omega}$.
\end{example}

Observing from Figure~\ref{fig:example2_data_covariance} that the larger $\sigma_{\omega}^2$, the higher the excess risk, and the more difficult for the benign overfitting condition to be satisfied for both MAML and iMAML.
Therefore, compared to ERM with a single task, the benign overfitting condition for MAML is more restrictive as it imposes constraints for both the expected data covariance $\mathbf{Q}_m$, and the data heterogeneity $\mathbb{V}(\{{\mathbf{W}}_m^{\cal A}\}_{m=1}^M )$.
 
\begin{figure*}[t]
  \centering
  \begin{subfigure}{0.45\textwidth}
    \centering
    \includegraphics[width=0.9\linewidth]{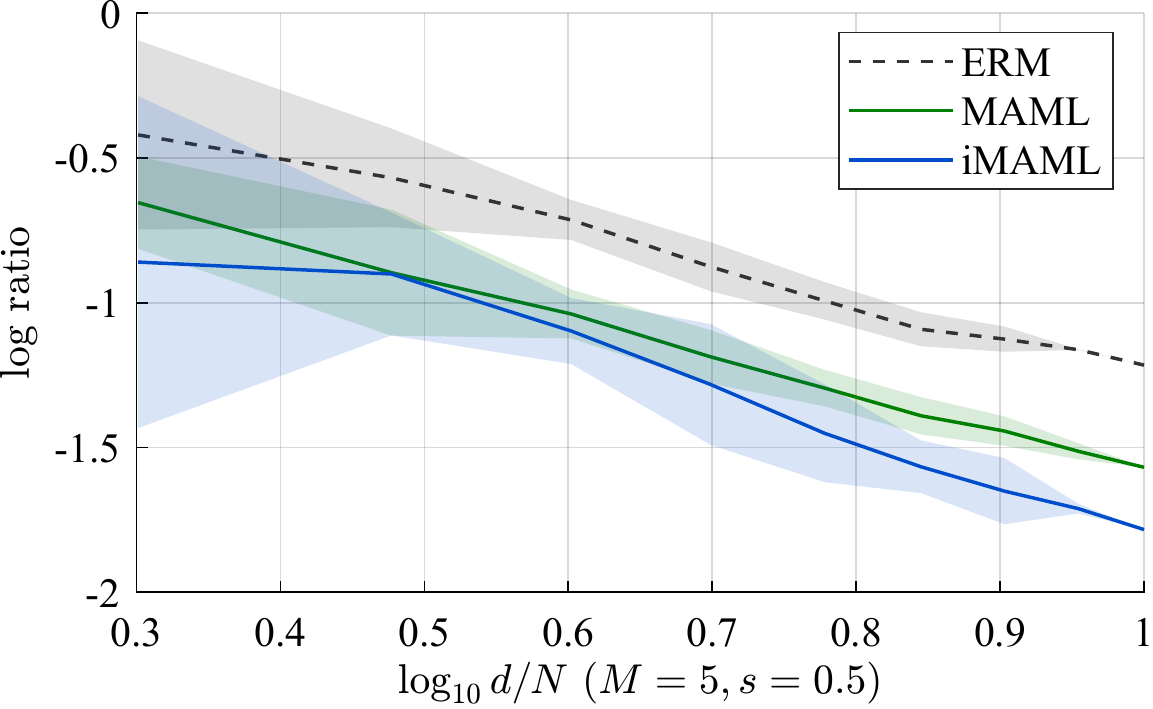}
    \caption{Cross-task variance ratio vs model dimension}
    \label{fig:task_para_dist_M}
  \end{subfigure}~~
  \begin{subfigure}{0.45\textwidth}
    \centering
    \includegraphics[width=0.9\linewidth]{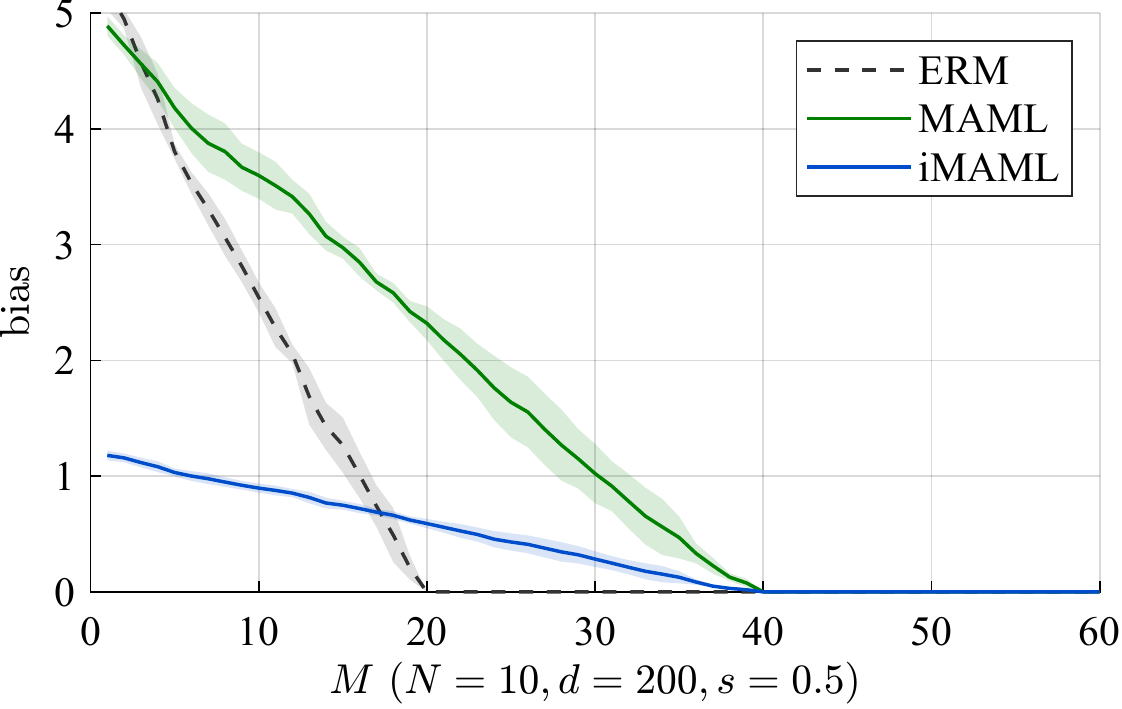}
    \caption{Bias vs task number}
    \label{fig:bias_M}
  \end{subfigure}
  \caption{Cross-task variance and bias versus task number to elaborate Lemma~\ref{lm:bound_task_para_dist} and Lemma~\ref{lm:bound_B}.}
  \label{fig:verify_}
\end{figure*}

\textbf{Connection to multi-task ERM.}
To compare benign overfitting in the gradient-based meta learning with that in the conventional ERM, where $\theta_m =\theta_0$, we can set the step size $\alpha = 0$ in MAML, or $\gamma \to \infty$ in iMAML, and $N_{\rm va} = N$,  which reduces to conventional ERM without adaptation. 

Compared to that of MAML and iMAML in~\eqref{eq:benign}, the benign overfitting condition is less restrictive for ERM since it does not impose constraints on $\alpha$ or $\gamma$. 
Intuitively, benign overfitting is more likely to happen in MAML or iMAML than in ERM.
The hyperparameters $\alpha$ and $\gamma$ will affect the eigenvalues of $\mathbf{W}_m^{\rm ma}$, $\mathbf{W}_m^{\rm im}$, respectively, thus affecting their corresponding excess risk.
Here we provide a sufficient condition where the benign overfitting condition in ERM is preserved in MAML or iMAML.
We summarize the results in the corollary below.

\begin{corollary}[Hyperparameters that preserve benign overfitting]
\label{crlr:hyperpara}
Recall $\lambda_1$ is the largest eigenvalue of $\mathbf{Q}$.
For MAML, when $0 < \alpha \leq \frac{1}{3\lambda_1 }$, 
and for iMAML, when $\gamma \geq \lambda_1$, 
then the effective ranks of $\mathbf{W}^{\rm ma}$ and $\mathbf{W}^{\rm im}$ are bounded above and below by a positive constant times the effective rank of $\mathbf{Q}$, and
therefore the benign overfitting condition holds for MAML and iMAML if it holds for ERM. 
To summarize, there are constants $c_1,c_2,c_3, c$ 
such that for $k^{*}=\min \{k \geq 0: r_{k}(\mathbf{Q}) \geq c_1 NM \}$.
For $\delta < 1$, $c_2 \log(1/\delta) < NM$ and $k^* < NM/c_3$, with probability at least $1 - 7e^{-2N M/c}$, it follows
\begin{align}\label{eq:maml_excess_risk_}
  \mathcal{E}^{\cal A}(\hat{\theta}_0^{\cal A}) \lesssim 
  \|\mathbb{E}[\theta_m^{\star}]\|^2 \bar{\lambda} \sqrt{\frac{r_0(\mathbf{Q})} {MN}} 
  +\sigma^{2} \Bigg(\frac{k^{*}}{MN}+\frac{MN}{R_{k^{*}}(\mathbf{Q} )}\Bigg)  \Bigg(1+\mathbb{V}(\{\mathbf{Q}_m\}_{m=1}^M) \Bigg).
\end{align}
\end{corollary}
\begin{remark}
\emph{
For MAML, let the unordered eigenvalues $\tilde{\mu}_i(\mathbf{W}^{\rm ma}) = \lambda_i(1 - \alpha \lambda_i)^2$. 
One challenge to control $\tilde{\mu}_i(\mathbf{W}^{\rm ma})$ is that  $\tilde{\mu}_i(\mathbf{W}^{\rm ma})$ are not necessarily monotonic w.r.t. $\lambda_i$; that is, it does not necessarily hold that
$\tilde{\mu}_1 \geq \tilde{\mu}_2 \geq \dots \geq \tilde{\mu}_d$.
For any $\lambda_i \geq \lambda_j$, if $\tilde{\mu}_i(\mathbf{W}^{\cal A}) \geq \tilde{\mu}_j(\mathbf{W}^{\cal A})$, then we say the order of the eigenvalues is preserved. 
For this to hold, it requires  $\tilde{\mu}_i(\lambda_i)$ to be a monotonically non-decreasing function of $\lambda_i$, which 
 yields $\alpha \leq \frac{1}{3\lambda_1}$.
Similar results can be obtained for iMAML by controlling the value of $\gamma$.
And the bound on $\alpha$ or $\gamma$ further ensures that $\tilde{\mu}_i(\mathbf{W^{\cal A}})$ is bounded above and below by a positive constant times the effective rank of $\mathbf{Q}$.} 
\end{remark}

\section{Proof Outline}\label{sec.proof-sketch}

In this section, we highlight the key steps of the proof for Theorem~\ref{thm:maml_excess_risk_bound}. We achieve so by analyzing the three terms in Proposition~\ref{prop:excess_risk} respectively.

  The first two terms in \eqref{eq:excess_risk_bound} can be bounded based on the concentration inequalities on subGaussian variables, given in Lemmas~\ref{lm:bound_task_para_dist} and \ref{lm:bound_B}.
  
\begin{lemma}[Bound on cross-task variance]
  \label{lm:bound_task_para_dist}
  With probability at least $1 - \delta$, it follows
  \begin{equation}
    \mathcal{E}_{\theta_m^{\star}} \!=\! \Bigg\|\Big(\sum_{m=1}^{M} \hat{\mathbf{W}}_{m}^{\cal A}\Big)^{\dag}\Big(\sum_{m=1}^{M} \hat{\mathbf{W}}_{m}^{\cal A}(\theta_{m}^{\star}-\theta_{0}^{\cal A})\Big)\Bigg\|_{\mathbf{W}^{\cal A}}^{2}
    \leq 
    \tilde{\mathcal{O}}\left(\frac{N}{d}\right) \mathcal{E}_{\epsilon_m}
  \end{equation}
  where $\widetilde{O}(\cdot)$ hides the log polynomial dependence on $N,M,d$.
\end{lemma}

The cross-task variance term analzyed in  Lemma~\ref{lm:bound_task_para_dist} is unique in meta learning, which captures the data heterogeneity across different tasks.
To elaborate Lemma~\ref{lm:bound_task_para_dist}, we plot  the cross-task variance versus the  task number in Figure~\ref{fig:task_para_dist_M} with task number $M = 5$, training validation split parameter $s =N_{\rm tr}/N = 0.5$, 
per-task data number $N=10$. This figure demonstrates that the ratio of cross-task variance and per-task variance decreases with $d/N$, which is consistent with  Lemma~\ref{lm:bound_task_para_dist}.

\begin{lemma}[Bound on bias]
  \label{lm:bound_B}
  For any $1<\log(1/\delta)<M N_{\rm va}$, with probability at least $1-\delta$, we have 
\begin{equation}
  {\mathcal{E}_b} 
  \lesssim \|\theta_{0}^{\cal A}\|^{2}\|\mathbf{W}^{\cal A}\| \max \left\{\sqrt{\frac{r_0(\mathbf{W}^{\cal A} )} {MN_{\rm va}}}, \frac{r_0(\mathbf{W}^{\cal A} )}{MN_{\rm va}}, \sqrt{\frac{\log(1/\delta)}{MN_{\rm va}}}\right\}.
\end{equation}
\end{lemma}

This term is similar to the bias term in the linear regression case, but directly depending on the solution matrix $\mathbf{W}$ instead of the data matrix $\mathbf{Q}$.
To elaborate Lemma~\ref{lm:bound_B}, Figure~\ref{fig:bias_M} demonstrates that the bias term decays with $M$ until it reaches zero when the model is underparameterized.
These two terms in Lemma~\ref{lm:bound_task_para_dist} and Lemma~\ref{lm:bound_B} 
do not go to infinity as $N,M,d$ increase. 

Note that, the key step is the bound on $\mathcal{E}_{\epsilon_m}$, which is the dominating term in the decomposition of excess risk \eqref{eq:excess_risk_bound} in the overparameterized regime.
We will bound it below. 
\begin{lemma}[Bound on per-task variance]
  \label{lm:bound_C}
There exist  constants $c_1, c_2,  c_3$ 
such that for $0 \leq k \leq 2N M/c_1$, $r_k({\mathbf{W}}^{\cal A} ) \geq c_2 N M$, 
and $k_0 \leq k$, with probability at least $1 - 7e^{-2N M/c_3}$, it follows 
\begin{equation}
\mathcal{E}_{\epsilon_m}
  \lesssim 
  \left(\frac{k_0}{M N_{\rm va} } 
  + \frac{M N_{\rm va} }{R_{k_0}(\bW^{\cal A} )}
  \right)
\left(1+ \mathbb{V}(\{{\mathbf{W}}_m^{\cal A}\}_{m=1}^M )\right).
\end{equation}
\end{lemma}
Note that, in the single task linear regression case, the there is no cross-task data heterogeneity, i.e., $\mathbb{V} = 0$. This term is unique in the meta learning setting with multiple tasks.
Plugging  the results of Lemmas~\ref{lm:bound_task_para_dist}, \ref{lm:bound_B} and~\ref{lm:bound_C} into \eqref{eq:excess_risk_bound}, we will reach Theorem~\ref{thm:maml_excess_risk_bound}.

\section{Conclusions and Limitations} 
\label{sec:conclusions}

This paper studies the generalization performance of the gradient-based meta learning with an overparameterized model. For a precise analysis, we focus on linear models where the total number of data from all tasks is smaller than the dimension of the model parameter.
We show that when the data heterogeneity across tasks is relatively small, the per-task data covariance matrices with certain properties lead to benign overfitting for gradient-based meta learning with the minimum-norm solution.
This explains why overparameterized meta learning models can generalize well in new data and new tasks.
Furthermore, our theory shows that overfitting is more likely to happen in meta learning than in ERM, especially when the data heterogeneity across tasks is relatively high. 

One limitation of this work is that the analysis focuses on the meta linear regression case. While this analysis can capture practical cases where we reuse the feature extractor from pre-trained models and only meta-train the parameters in the last linear layer, it is also promising to extend our analysis to nonlinear cases via means of random features and neural tangent kernels in the future work.

\section*{Acknowledgments}
 This work was partially supported by National Science Foundation MoDL-SCALE Grant 2134168 and the Rensselaer-IBM AI Research Collaboration (\url{http://airc.rpi.edu}), part of the IBM AI Horizons Network (\url{http://ibm.biz/AIHorizons}). 


{
\small
\bibliographystyle{plain}
\bibliography{myabrv,bmaml,maml_theory,maml,benign,statistics}

}

\section*{Checklist}


\begin{enumerate}

\item For all authors...
\begin{enumerate}
  \item Do the main claims made in the abstract and introduction accurately reflect the paper's contributions and scope?
    \answerYes{}
  \item Did you describe the limitations of your work?
    \answerYes{}
  \item Did you discuss any potential negative societal impacts of your work?
    \answerNA{Work of theoretical nature, no potential negative societal impacts.}
  \item Have you read the ethics review guidelines and ensured that your paper conforms to them?
    \answerYes{}
\end{enumerate}

\item If you are including theoretical results...
\begin{enumerate}
  \item Did you state the full set of assumptions of all theoretical results?
    \answerYes{}
  \item Did you include complete proofs of all theoretical results?
    \answerYes{See Supplementary material for the complete proofs.} 
\end{enumerate}

\item If you ran experiments...
\begin{enumerate}
  \item Did you include the code, data, and instructions needed to reproduce the main experimental results (either in the supplemental material or as a URL)?
    \answerNA{Work of theoretical nature.}
  \item Did you specify all the training details (e.g., data splits, hyperparameters, how they were chosen)?
    \answerYes{}
  \item Did you report error bars (e.g., with respect to the random seed after running experiments multiple times)?
    \answerYes{}
  \item Did you include the total amount of compute and the type of resources used (e.g., type of GPUs, internal cluster, or cloud provider)?
    \answerYes{}
\end{enumerate}

\item If you are using existing assets (e.g., code, data, models) or curating/releasing new assets...
\begin{enumerate}
  \item If your work uses existing assets, did you cite the creators?
    \answerNA{}
    Work of theoretical nature. We cite the authors who propose the baseline models.
  \item Did you mention the license of the assets?
    \answerNA{}
  \item Did you include any new assets either in the supplemental material or as a URL?
    \answerNA{}
  \item Did you discuss whether and how consent was obtained from people whose data you're using/curating?
    \answerNA{}
  \item Did you discuss whether the data you are using/curating contains personally identifiable information or offensive content?
    \answerNA{}
\end{enumerate}

\item If you used crowdsourcing or conducted research with human subjects...
\begin{enumerate}
  \item Did you include the full text of instructions given to participants and screenshots, if applicable?
    \answerNA{}
  \item Did you describe any potential participant risks, with links to Institutional Review Board (IRB) approvals, if applicable?
    \answerNA{}
  \item Did you include the estimated hourly wage paid to participants and the total amount spent on participant compensation?
    \answerNA{}
\end{enumerate}

\end{enumerate}

\clearpage
\appendix

\begin{center}
  {\Large \textbf{Supplementary Material}} \\

\end{center}


In this supplementary document, we  present  the missing derivations of some claims, as well as the proofs of all the lemmas and theorems in the paper.

\vspace{-1cm}
\addcontentsline{toc}{section}{} 
\part{} 
\parttoc 


\section{Notations} 
\label{sec:notations}
We use $[\mathbf{X}_m]$ to represent row stack of matrices $\mathbf{X}_m$ with indices $m$, i.e.
\begin{align*}
  [\mathbf{X}_m] 
= \begin{bmatrix}
\mathbf{X}_1^{\top}, \mathbf{X}_2^{\top}, \dots, \mathbf{X}_M^{\top}
\end{bmatrix}^{\top}.
\end{align*}

For a given square matrix $\bD_m$, define 
\begin{align*}
\mathrm{diag}[\bD_m ]
= \begin{bmatrix}
  \bD_1 & \mathbf{0} &\dots & \mathbf{0} \\
  \mathbf{0} &\bD_2 & &\vdots\\
  \vdots & &\ddots  & \mathbf{0}\\
  \mathbf{0}&\dots & \mathbf{0} &\bD_M
\end{bmatrix}.  
\end{align*}

We use $\mu_i(\cdot)$ 
to denote the $i$-th eigenvalue of a matrix with descending order, $\|\cdot\|$ to denote the operator norm, and $\|\cdot\|_{\rm F}$ to denote the Frobenious norm.

For any matrix $\mathbf{M} \in \mathbb{R}^{n \times d}$, denote $\mathbf{M}_{0: k}$ to be the matrix which is comprised of the first $k$ columns of $\mathbf{M}$, and $\mathbf{M}_{k: d}$ to be the matrix comprised of the rest of the columns of $\mathbf{M}$.
For any vector $\eta \in \mathbb{R}^{d}$ denote $\eta_{0: k}$ to be the vector comprised of the first $k$ components of $\eta$, and $\eta_{k: \infty}$ to be the vector comprised of the rest of the coordinates of $\eta$.
Denote $\bLam_{0: k}=\operatorname{diag}(\lambda_{1}, \ldots, \lambda_{k})$, and $\bLam_{k: \infty}=\operatorname{diag}(\lambda_{k+1}, \lambda_{k+2}, \ldots)$,
$\bLam_{k: d}=\operatorname{diag}(\lambda_{k+1}, \lambda_{k+2}, \ldots
\lambda_d )$.

For $t\geq 0$, $N \in \mathbb{Z}^{+}$, define $c_{r_0}({r}_0({\mathbf{\Lambda}}), N, t) \coloneqq \max \Big\{ \sqrt{\frac{{r}_0(\mathbf{\mathbf{\Lambda}})}{N}}, \frac{{r}_0({\mathbf{\Lambda}})}{N}, \sqrt{\frac{t}{N}}, \frac{t}{N} \Big\}$.

We use $\mathbb{E}[\cdot]$ to denote expectation and $\mathrm{Cov}[\cdot]$ to denote covariance.

{We use superscript ``$\rm ma$'' and ``$\rm im$'' to represent quantities related to the MAML and iMAML algorithms, respectively. For notation simplicity, we omit the superscript $\cal A$ when the arguments hold for both MAML and iMAML.}


\section{Proof of Proposition \ref{prop:solutions}}
\label{sec:proof_solutions}

\begin{proposition}[Empirical and population level solutions]
\label{prop:solutions_app}
Under the data model \eqref{eq:linear_data_generate},
the meta-test risk of method $\mathcal{A}$ defined in \eqref{eq:R} can be computed by
\begin{align*}
  &\mathcal{R}^{\mathcal{A}}({\theta_0} )
    = \mathbb{E}_{m}\big[\|\theta_0-\theta^{\star}_{m}\|^2_{\mathbf{W}_{m}^{\mathcal{A}}}\big] + c.
 \end{align*} 
The optimal solutions to the meta-test risk and the minimum-norm solution are given below respectively
\begin{subequations}
\begin{align}
  &\theta_{0}^{\mathcal{A}} 
  \coloneqq \mathop{\arg\min}_{\theta_0}
  \mathcal{R}^{\mathcal{A}}({\theta_0} )
  =\mathbb{E}_{m}\big[\mathbf{W}_{m}^{\mathcal{A}}\big]^{-1} \mathbb{E}_{m}\big[\mathbf{W}_{m}^{\mathcal{A}} {\theta}_{m}^{\star} \big] \\
  &\hat{\theta}_{0}^{\mathcal{A}} 
  \coloneqq \mathop{\arg\min}_{\theta_0}
  \mathcal{L}^{\mathcal{A}} (\theta_0, \mathcal{D} ) 
  = \Big(\sum_{m=1}^{M}
  \hat{\mathbf{W}}_{m}^{\mathcal{A}}\Big)^{\dag}
  \Big(\sum_{m=1}^{M}\hat{\mathbf{W}}_{m}^{\mathcal{A}}\theta_{m}^{\star} 
  \Big)
  + \Delta_{M}^{\mathcal{A}}
\end{align}
\end{subequations}
where $^{\dag}$ denotes the Moore-Penrose pseudo inverse, the error term $\Delta_{M}^{\mathcal{A}}$ is a polynomial function of $M,N,d$, which will be specified in the following sections for MAML and iMAML. 
And $\hat{\mathbf{Q}}_{m}^{\rm al} \coloneqq \frac{1}{N} \mathbf{X}^{\rm al \top}_{m}\mathbf{X}_{m}^{\rm al}$.
The weight matrices of different methods, $\mathbf{W}_{m}^{\mathcal{A}}$ and $\hat{\mathbf{W}}_{m}^{\mathcal{A}}$, are given in Table~\ref{tab:weight_matrices}.

\end{proposition}

\subsection{Model agnostic meta learning method} 
\label{app_sub:maml_method}

Without loss of generality, assume $\sigma = 1$ to simplify notation.  
We use meta-test risk $\mathcal{R}_N^{\cal A}$ to represent expected test risk with finite number of adaptation data $N$ during testing, which is slightly different compared to population risk $\mathcal{R}^{\cal A} = \lim_{N\to \infty}\mathcal{R}_N^{\cal A}$.
The MAML meta-test risk is defined as 
~\citep{gao2020_model_opt_tradeoff_ml}
\begin{align}\label{eq:maml_pop_risk}
  \mathcal{R}_N^{\rm ma}({\theta_0} )
  \coloneqq &
  \mathbb{E}\left[\left(y_{m}-\hat{\theta}_{m}^{\rm ma} (\theta_0, \mathcal{D}_{m, N}) ^{\top} {x}_{m}\right)^{2} \right]\nonumber\\
  =&
  \mathbb{E}_{m}\left[\|\theta_0-\theta^{\star}_{m}\|^2_{\mathbf{W}_{m, N}^{\rm ma}}\right] 
  + 1 + \frac{\alpha^2}{N} \mathbb{E}_{m}[\mathrm{Tr}(\mathbf{Q}^2_{m})] 
\end{align}
where the matrix is defined as 
\begin{align}\label{eq:W_mau_N_ma}
  \mathbf{W}_{m,N}^{\rm ma} 
  =&
  \mathbb{E}_{\hat{\mathbf{Q}}_{m}}\left[(\mathbf{I} - {\alpha} \hat{\mathbf{Q}}_{m}) \mathbf{Q}_{m}(\mathbf{I} - {\alpha} \hat{\mathbf{Q}}_{m})\right]  \nonumber \\
  =& 
  \left(\mathbf{I}-\alpha \mathbf{Q}_{m}\right) \mathbf{Q}_{m}\left(\mathbf{I}-\alpha  \mathbf{Q}_{m}\right)+\frac{\alpha^{2}}{N}\left(\mathbb{E}_{{x}_{m, i}}\left[{x}_{m, i} {x}_{m, i}^{\top} \mathbf{Q}_{m} {x}_{m, i} {x}_{m, i}^{\top}\right]-\mathbf{Q}_{m}^{3}\right). 
\end{align}

Assume during meta testing, we have infinite adaptation data, i.e., $N \to \infty$,
then the optimal population risk of MAML  is
 \begin{align*}
  \mathcal{R}^{\rm ma}(\theta_0)& 
  = \lim_{N \to \infty} 
  \mathcal{R}_N^{\rm ma}(\theta_0)
  = \mathbb{E}_{m}\left[\|\theta_0-\theta^{\star}_{m}\|^2_{\mathbf{W}_{m}^{\rm ma}}\right] 
  + 1
  \numberthis. 
 \end{align*}

In MAML, define $\theta_{0}^{\rm ma}$ as the minimizer of the  optimal population risk of MAML, given by
\begin{align}\label{eq:theta_0_star_MAML_sln_app}
  \theta_{0}^{\rm ma} 
  = \mathop{\arg\min}_{\theta_0} 
  \mathcal{R}^{\rm ma}({\theta_0} )
  = \mathop{\arg\min}_{\theta_0} 
  \mathbb{E}_{m}\big[\|\theta_0-\theta^{\star}_{m}\|^2_{\mathbf{W}_{m}^{\rm ma}}\big]
  =\mathbb{E}_{m}\left[\mathbf{W}_{m}^{\rm ma}\right]^{-1} \mathbb{E}_{m}\left[\mathbf{W}_{m}^{\rm ma} {\theta}_{m}^{\star}\right]
  .
\end{align}

Using the optimality condition of $\mathcal{L}^{\rm ma}(\theta_0, \mathcal{D})$ given in \eqref{eq:emp_loss}, we have
\begin{subequations}
\begin{align}\label{eq:theta_0_hat_MAML_sln_app}
  &
  \hat{\theta}_{0}^{\rm ma} 
  = \Big(\sum_{m=1}^{M}
  \hat{\mathbf{W}}_{m}^{\rm ma}\Big)^{\dag} 
  \Big(\sum_{m=1}^{M}\hat{\mathbf{W}}_{m}^{\rm ma}\theta_{m}^{\star} + 
  \big(\mathbf{I}- 
  {\alpha }\hat{\mathbf{Q}}_{m}^{\rm tr}
  \big)
  \big(
  \frac{1}{N_{\rm va}}\mathbf{X}_{m}^{\rm va\top} \mathbf{e}_{m}^{\rm va} - \frac{\alpha  }{N_{\rm tr}}
  \hat{\mathbf{Q}}_{m}^{\rm va}\mathbf{X}_{m}^{\rm tr\top} \mathbf{e}_{m}^{\rm tr}
  \big)
  \Big) \\
  \label{eq:W_hat_mau_ma}
  &
  \hat{\mathbf{W}}_{m}^{\rm ma} 
  = 
  (\mathbf{I}-
  {\alpha}\hat{\mathbf{Q}}_{m}^{\rm tr}) 
  \hat{\mathbf{Q}}_{m}^{\rm va}
  (\mathbf{I}- 
  {\alpha}\hat{\mathbf{Q}}_{m}^{\rm tr}). 
\end{align}
\end{subequations}

Therefore, we can arrive at \eqref{eq:theta_0_hat_A_sln} by defining 
\begin{align}
  \Delta_M^{\rm ma} \coloneqq 
  \Big(\sum_{m=1}^{M} \hat{\mathbf{W}}_{m}^{\rm ma}
  \Big)^{\dag}
  \Big(\sum_{m=1}^{M}(\mathbf{I} -\alpha \hat{\mathbf{Q}}_{m}^{\rm tr}) \frac{1}{N_{2}} \mathbf{X}_{m}^{\rm va \top} \mathbf{e}_{m}^{\rm va}
  - (\mathbf{I}-\alpha \hat{\mathbf{Q}}_{m}^{\rm tr}) \hat{\mathbf{Q}}_{m}^{\rm va} \frac{\alpha}{N_{\rm tr}} \mathbf{X}_{m}^{\rm tr{\top}} \mathbf{e}_{m}^{\rm tr}
  \Big).
\end{align}

\subsection{Implicit model agnostic meta learning method} 
\label{app_sub:bimaml_method}

  For the iMAML method, 
  the task-specific parameter $\hat{\theta}_{m}^{\mathrm{im}}$ is computed from the initial parameter $\theta_0$ by optimizing the regularized task-specific empirical loss, given by
  \begin{align}\label{eq:theta_tau_biMAML_theta0}
    \hat{\theta}_{m}^{\mathrm{im}} (\theta_0, \mathcal{D}_m) 
    = \mathop{\arg\min}_{\theta_{m}} 
    \frac{1}{ N} 
    \left\|\mathbf{y}_{m}- \mathbf{X}_{m} {\theta}_{m} \right\|^{2} 
     + \gamma \|\theta_{m} - \theta_0\|^2
  \end{align}
  where $\gamma $ is the weight of the regularizer,
  and $\mathcal{D}_{m}$ is the adaptation data during meta-testing or training data during meta-training. 

  The estimated task-specific parameter can be computed by 
  \begin{align}\label{eq:theta_tau_hat_biMAML}
    \hat{\theta}_{m}^{\mathrm{im}}(\theta_0,\mathcal{D}_{m})
    &= (\hat{\mathbf{Q}}_{m}^{\rm al} +  \gamma  \mathbf{I})^{-1}\Big(\frac{1}{N} \mathbf{X}_{m}^{\top}\mathbf{y}_{m} + \gamma  \theta_0\Big). 
  \end{align}

  The empirical loss of iMAML is defined as the average per-task loss, given by
  \begin{align}\label{eq:biMAML_emp_risk_app}
    {\mathcal{L}}_{M,N}^{\mathrm{im}}({\theta_0}, \mathcal{D})
    =\frac{1}{MN_{\rm va}} \sum_{m=1}^{M} 
    \left\|\mathbf{y}^{\text{va}}_{m} - \mathbf{X}^{\text{va}}_{m}\hat{\theta}^{\mathrm{im}}_{m}({\theta_0}, \mathcal{D}_{m}^{\text{tr}})\right\|^2
  \end{align}
  whose minimizer is 
  \begin{align*}\label{eq:theta_0_hat_biMAML_derive}
    \hat{\theta}_{0}^{\mathrm{im}} 
    = \mathop{\arg\min}_{\theta_0} \frac{1}{MN_{\rm va}} \sum_{m=1}^{M} 
    \Big\|\mathbf{X}_{m, N}^{\text{va}}\theta_{m}^{\star} + \mathbf{e}^{\text{val}}_{m,N_{\rm va}} - \mathbf{X}^{\text{va}}_{m}\hat{\theta}^{\mathrm{im}}_{m}({\theta_0}, \mathcal{D}_{m}^{\text{tr}})\Big\|^2. \numberthis
  \end{align*}
Using the optimality condition of the above problem, we obtain
  \begin{subequations}
    \begin{align}\label{eq:theta_0_hat_biMAML_sln_app}
      &\hat{\theta}_{0}^{\mathrm{im}} 
      = \Big(\sum_{m=1}^{M}\hat{\mathbf{W}}_{m}^{\mathrm{im}}\Big)^{\dag}
      \Big(\sum_{m=1}^{M}
      \hat{\mathbf{W}}_{m}^{\mathrm{im}} \theta_{m}^{\star} \Big) + \Delta_{M}^{\rm im} \\
{\rm with}~~~~       &\Delta_{M}^{\rm im}
      =\Big(\sum_{m=1}^{M}\hat{\mathbf{W}}_{m}^{\mathrm{im}}\Big)^{\dag}
      \Big(\sum_{m=1}^{M}\gamma {\Sigma}_{\theta_{m}}
      \frac{1}{N_{\rm va}}\mathbf{X}^{\text{va}\top}_{m}\mathbf{e}_{m, N}^{\text{va}}
      -\gamma^{-1}\hat{\mathbf{W}}_{m}^{\mathrm{im}}\frac{1}{N_{\rm tr}}  \mathbf{X}_{m}^{\text{tr} \top} \mathbf{e}_{m}^{\text{tr}}\Big) 
    \end{align}          
    where we define 
      \begin{align}     
      &{\Sigma}_{\theta_{m}} 
      \coloneqq \Big(\frac{1}{N_{\rm tr}} \mathbf{X}_{m}^{\mathrm{Tr}\top} \mathbf{X}_{m}^{\mathrm{Tr}}
      + \gamma  \mathbf{I}\Big)^{-1} 
      = ( \hat{\mathbf{Q}}_{m}^{\rm tr}
      + \gamma  \mathbf{I})^{-1} \\
      \label{eq:W_hat_bi}
      &\hat{\mathbf{W}}_{m}^{\mathrm{im}} \coloneqq\gamma^2    
      {\Sigma}_{\theta_{m}}
      \frac{1}{N_{\rm va}}\mathbf{X}^{\text{va}\top}_{m}\mathbf{X}^{\text{va}}_{m}
      {\Sigma}_{\theta_{m}}
      =\gamma^2  {\Sigma}_{\theta_{m}}
      \hat{\mathbf{Q}}_{m}^{\rm va}
      {\Sigma}_{\theta_{m}}. 
    \end{align}
    \end{subequations}

  The meta-test risk of iMAML is defined as 
\begin{subequations}\label{eq:bimaml_pop_risk}
  \begin{align}\label{eq:bimaml_pop_risk_R}
  \hspace{-5mm}
  &\mathcal{R}_{N_a}^{\mathrm{im}}({\theta_0} )
  =
  \mathbb{E}\big[ 
  \big(y_{m}-\hat{\theta}_{m}^{\mathrm{im}} (\theta_0, \mathcal{D}_{m,N_a}) ^{\top} {x}_{m}\big)^{2} \big] \nonumber \\
  &\quad\quad \quad ~~ =
  \mathbb{E}_{m}\big[\|\theta_0-\theta^{\star}_{m}\|^2_{\mathbf{W}_{m,N_a}^{\mathrm{im}}}\big] 
  + 1 
  +\frac{1}{N_a} \mathbb{E}[ \gamma^{-2} \mathrm{Tr}(\mathbf{W}_{m,N_a}^{\mathrm{im}} \hat{\mathbf{Q}}_{m,N_a})] 
 \end{align} 
where the weight matrix is defined as
 \begin{align}
\mathbf{W}_{m,N_a}^{\mathrm{im}}
 &  =
  \mathbb{E}_{{{x}}_{m}}\big[(\hat{\mathbf{Q}}_{m,N_a} + \gamma \mathbf{I} )^{-1} \mathbf{Q}_{m}
  (\hat{\mathbf{Q}}_{m,N_a} + \gamma \mathbf{I} )^{-1}\big] 
  \nonumber \\
&  = \mathbf{W}_{m}^{\mathrm{im}} 
 \mathbb{E}_{{{x}}_{m}}\big[{\Sigma}_{\theta_{m}}\big(\mathbf{Q}_{m}-\hat{\mathbf{Q}}_{m,N_a}\big) \mathbf{W}_{m}^{\mathrm{im}}\big(\mathbf{Q}_{m}-\hat{\mathbf{Q}}_{m,N_a}\big) {\Sigma}_{\theta_{m}} +{\Sigma}_{\theta_{m}}\big(\mathbf{Q}_{m}-\hat{\mathbf{Q}}_{m,N_a}\big) \mathbf{W}_{m}^{\mathrm{im}}   \nonumber \\
  &~~~~+\mathbf{W}_{m}^{\mathrm{im}}\big(\mathbf{Q}_{m}-\hat{\mathbf{Q}}_{m,N_a}\big) {\Sigma}_{\theta_{m}}\big]
\end{align}
\end{subequations}
where $\mathbf{W}_{m}^{\mathrm{im}} =
(\gamma ^{-1}\mathbf{Q}_{m}+\mathbf{I})^{-1}\mathbf{Q}_{m}(\gamma ^{-1}\mathbf{Q}_{m}+\mathbf{I})^{-1} $.

Simplify the notation of $\mathbf{X}_{m,N_a}, \mathbf{y}_{m,N_a}, \hat{\mathbf{Q}}_{m, N_a}$ as
$\mathbf{X}_{m}, \mathbf{y}_{m},  \hat{\mathbf{Q}}_{m}$.
The derivation of \eqref{eq:bimaml_pop_risk} is given by
\begin{align*}
\label{eq:total_emp_risk_imaml_theta0_a}
    \mathcal{R}_{N_a}^{\mathrm{im}}({\theta}_0)
    =& \mathbb{E}\big[ %
    \|\hat{\theta}^{\mathrm{im}}_{m}(\theta_0, \mathcal{D}_{m,N_a})-\theta^{\star}_{m}\|^2_{\mathbf{Q}_{m}}\big] + 1 \numberthis\\
    =& \mathbb{E}\big[ \|(\hat{\mathbf{Q}}_{m} + \gamma \mathbf{I})^{-1}(\frac{1}{N_a} \mathbf{X}_{m}^{\top}\mathbf{y}_{m} + \gamma {\theta}_0)-\theta^{\star}_{m}\|^2_{\mathbf{Q}_{m}}\big] + 1 \\
    \stackrel{(a)}{=}& \mathbb{E}\Big[
     {\theta}_0^{\top} 
     \mathbf{W}_{m,N_a}^{\mathrm{im}} {\theta}_0
    +2 \gamma(\frac{1}{N_a}\mathbf{y}_{m}^{\top}\mathbf{X}_{m} 
    {\Sigma}_{\theta_{m}} - \theta_{m}^{\star\top})
     \mathbf{Q}_{m} {\Sigma}_{\theta_{m}}{\theta}_0 
    + \\
    &\frac{1}{N_a}\mathbf{y}_{m}^{\top}\mathbf{X}_{m}
    {\Sigma}_{\theta_{m}}  \mathbf{Q}_{m} {\Sigma}_{\theta_{m}} \frac{1}{N_a} \mathbf{X}_{m}^{\top}\mathbf{y}_{m} 
     - 2\theta_{m}^{\star\top}
     \mathbf{Q}_{m} {\Sigma}_{\theta_{m}}
    \frac{1}{N_a} \mathbf{X}_{m}^{\top}\mathbf{y}_{m}
    + \theta_{m}^{\star\top} \mathbf{Q}_{m} \theta_{m}^{\star}
    \Big] + 1 
\end{align*}
where $(a)$ follows from the definition of ${\Sigma}_{\theta_{m}} = (\hat{\mathbf{Q}}_{m}^{\rm al}+\gamma \mathbf{I})^{-1}$, 
and $\mathbf{W}_{m,N_a}^{\mathrm{im}} = \gamma^2{\Sigma}_{\theta_{m}}  \mathbf{Q}_{m} {\Sigma}_{\theta_{m}} $.

Applying the fact that $\mathbf{y}_{m} = \mathbf{X}_{m}\theta_{m}^{\star} + \mathbf{e}_{m}$ and \(\mathbb{E}_{\mathbf{e}_{m}}[\mathbf{e}_{m}] = \mathbf{0}\), one can further derive from \eqref{eq:total_emp_risk_imaml_theta0_a} that
\begin{align*}
  \mathcal{R}_{N_a}^{\mathrm{im}}({\theta}_0)
  {=}& 
  \mathbb{E}\Big[  {\theta}_0^{\top} 
  \mathbf{W}_{m,N_a}^{\mathrm{im}} {\theta}_0
  +2 \gamma(\theta_{m}^{\star\top}\hat{\mathbf{Q}}_{m}
  {\Sigma}_{\theta_{m}} - \theta_{m}^{\star\top})
  \mathbf{Q}_{m} {\Sigma}_{\theta_{m}}{\theta}_0+ \theta_{m}^{\star\top}
  \hat{\mathbf{Q}}_{m}
  {\Sigma}_{\theta_{m}}  \mathbf{Q}_{m} {\Sigma}_{\theta_{m}} \hat{\mathbf{Q}}_{m}\theta_{m}^{\star} \\
  \label{eq:R_im}
  &  - 2\theta_{m}^{\star\top}
  \mathbf{Q}_{m} {\Sigma}_{\theta_{m}}
  \hat{\mathbf{Q}}_{m}\theta_{m}^{\star}
  + \theta_{m}^{\star\top} \mathbf{Q}_{m} \theta_{m}^{\star} +\frac{1}{N_a^2}
  \mathbf{e}_{m}^{\top}
  \mathbf{X}_{m}
  {\Sigma}_{\theta_{m}}  \mathbf{Q}_{m} {\Sigma}_{\theta_{m}} \mathbf{X}_{m}^{\top}
  \mathbf{e}_{m}
  \Big] + 1. \numberthis
\end{align*}
Based on the linearity of trace and expectation, and the cyclic property of trace, the last term inside the expectation in the above equation can be computed as 
\begin{align*}
  &\mathbb{E}_{\mathbf{e}_{m}}[\mathbf{e}_{m}^{\top}
  \mathbf{X}_{m}
  {\Sigma}_{\theta_{m}}  \mathbf{Q}_{m} {\Sigma}_{\theta_{m}} \mathbf{X}_{m}^{\top}
  \mathbf{e}_{m}^{\cal A}] 
  = \mathrm{Tr}(
  \mathbf{X}_{m}
  {\Sigma}_{\theta_{m}}  \mathbf{Q}_{m} {\Sigma}_{\theta_{m}} \mathbf{X}_{m}^{\top}
  \mathbb{E}_{\mathbf{e}_{m}}[\mathbf{e}_{m}^{\cal A}\mathbf{e}_{m}^{\top}]) \\
  =& \mathrm{Tr}(
  \mathbf{X}_{m}
  {\Sigma}_{\theta_{m}}  \mathbf{Q}_{m} {\Sigma}_{\theta_{m}} \mathbf{X}_{m}^{\top})
  = N_a \mathrm{Tr}(
  {\Sigma}_{\theta_{m}} \mathbf{Q}_{m} {\Sigma}_{\theta_{m}} \hat{\mathbf{Q}}_{m})
  = N_a \mathrm{Tr}(
  \mathbf{W}_{m,N_a}^{\rm im} \hat{\mathbf{Q}}_{m}).
\end{align*}
To derive all the terms related to $\theta_m^{\star}$, 
based on the Woodbury matrix identity,
$\mathbf{I} - \hat{\mathbf{Q}}_{m} 
  {\Sigma}_{\theta_{m}} = 
\mathbf{I} -  
  {\Sigma}_{\theta_{m}}
  \hat{\mathbf{Q}}_{m}
  = \gamma{\Sigma}_{\theta_{m}}$, we have 
\begin{align}\label{eq:inter_sigma}
  (\theta_{m}^{\star\top}\hat{\mathbf{Q}}_{m}
  {\Sigma}_{\theta_{m}} - \theta_{m}^{\star\top})=
  \theta_{m}^{\star\top}(\hat{\mathbf{Q}}_{m}
  {\Sigma}_{\theta_{m}} - \mathbf{I})
  = - \gamma \theta_{m}^{\star\top}{\Sigma}_{\theta_{m}}
\end{align} 
and then the terms related to $\theta^{\star}_m$ in \eqref{eq:R_im} can be computed by
\begin{align*}
  &\theta_{m}^{\star\top}
  \hat{\mathbf{Q}}_{m}
  {\Sigma}_{\theta_{m}}  \mathbf{Q}_{m} {\Sigma}_{\theta_{m}} \hat{\mathbf{Q}}_{m}\theta_{m}^{\star}
  - 2\theta_{m}^{\star\top}
  \mathbf{Q}_{m} {\Sigma}_{\theta_{m}}
  \hat{\mathbf{Q}}_{m}\theta_{m}^{\star}
  + \theta_{m}^{\star\top} \mathbf{Q}_{m} \theta_{m}^{\star}\\
  =& \theta_{m}^{\star\top}
  \big(
  (\hat{\mathbf{Q}}_{m}
  {\Sigma}_{\theta_{m}}-\mathbf{I})  \mathbf{Q}_{m} {\Sigma}_{\theta_{m}} \hat{\mathbf{Q}}_{m}
  + \mathbf{Q}_{m} (\mathbf{I}-{\Sigma}_{\theta_{m}}
  \hat{\mathbf{Q}}_{m})
  \big)
  \theta_{m}^{\star} \\
  \stackrel{(a)}{=}& \theta_{m}^{\star\top}
  \big(
  -\gamma {\Sigma}_{\theta_{m}}  \mathbf{Q}_{m} {\Sigma}_{\theta_{m}} \hat{\mathbf{Q}}_{m}
  + \mathbf{Q}_{m} \gamma {\Sigma}_{\theta_{m}}
  \big)
  \theta_{m}^{\star} \\
  \label{eq:theta_star_eq}
  \stackrel{(b)}{=} &\gamma^{-1} \theta_{m}^{\star\top}
  \big(
  -  \mathbf{W}_{m,N_a}^{\rm im}  \hat{\mathbf{Q}}_{m}
  + (\hat{\mathbf{Q}}_{m}+\gamma \mathbf{I})
  \mathbf{W}_{m,N_a}^{\mathrm{im}}
  \big)
  \theta_{m}^{\star} 
  \numberthis
\end{align*}
where $(a)$ follows from \eqref{eq:inter_sigma}, and $(b)$ follows from the definition of $\mathbf{W}_{m,N_a}^{\rm im}$.

Combining \eqref{eq:R_im} and \eqref{eq:theta_star_eq} and rearranging the equations, we obtain
\begin{align*}
 &\mathcal{R}_{N_a}^{\mathrm{im}}({\theta}_0)
{=}  \mathbb{E}\Big[
  {\theta}_0^{\top} 
  \mathbf{W}_{m,N_a}^{\mathrm{im}} {\theta}_0
  -2 \theta_{m}^{\star\top}
  \mathbf{W}_{m,N_a}^{\mathrm{im}} {\theta}_0+ \\
  &\hspace{1.6cm}\gamma^{-1} \theta_{m}^{\star\top}
  \big(
  -  \mathbf{W}_{m,N_a}^{\rm im}  \hat{\mathbf{Q}}_{m}
  + (\hat{\mathbf{Q}}_{m}+\gamma \mathbf{I})
  \mathbf{W}_{m,N_a}^{\mathrm{im}}
  \big)
  \theta_{m}^{\star} 
  +\frac{1}{N_a  \gamma^{2}} 
  \mathrm{Tr}(
  \mathbf{W}_{m,N_a}^{\mathrm{im}} \hat{\mathbf{Q}}_{m})
  \Big] + 1 \\
  \stackrel{(c)}{=}& \mathbb{E}\Big[
  \|{\theta}_0 - {\theta}_{m}^{\star}\|^2_{ \mathbf{W}_{m,N_a}^{\mathrm{im}}}
\!\!\!  + \gamma^{-1}\theta_{m}^{\star\top}
  \big(
  - \mathbf{W}_{m,N_a}^{\mathrm{im}} \hat{\mathbf{Q}}_{m}
  +\hat{\mathbf{Q}}_{m}
  \mathbf{W}_{m,N_a}^{\mathrm{im}} 
  \big)
  \theta_{m}^{\star} 
  +\frac{1}{N_a  \gamma^{2}} 
  \mathrm{Tr}(
  \mathbf{W}_{m,N_a}^{\mathrm{im}} \hat{\mathbf{Q}}_{m})
  \Big] + 1 \\
  \stackrel{(d)}{=}& \mathbb{E}\Big[
  \|{\theta}_0 - {\theta}_{m}^{\star}\|^2_{\mathbf{W}_{m,N_a}^{\mathrm{im}}}
 \!   +\frac{1}{N_a  \gamma^{2}} 
  \mathrm{Tr}(
  \mathbf{W}_{m,N_a}^{\mathrm{im}} \hat{\mathbf{Q}}_{m})
  \Big] + 1 
  \numberthis
\end{align*}
where
$(c)$ follows from rearranging the equations;
$(d)$ follows from the fact that
\begin{equation}
  \theta_{m}^{\star\top}
  \big(
  \mathbf{W}_{m,N_a}^{\mathrm{im}} \hat{\mathbf{Q}}_{m}
  \big)
  \theta_{m}^{\star} 
  = \big(\theta_{m}^{\star \top}
  ( \mathbf{W}_{m,N_a}^{\mathrm{im}} \hat{\mathbf{Q}}_{m})
  \theta_{m}^{\star} \big)^{\top}
  = \theta_{m}^{\star \top}
  \big(
  \hat{\mathbf{Q}}_{m}
  \mathbf{W}_{m,N_a}^{\mathrm{im}} 
  \big)
  \theta_{m}^{\star}. 
\end{equation}

Since $\lim_{N_a \to \infty} \frac{1}{N_a} \mathbb{E}[ \gamma^{-2} \mathrm{Tr}(\mathbf{W}_{m,N_a}^{\mathrm{im}} \hat{\mathbf{Q}}_{m,N_a})]  =0$,
from the definition of the population risk in~\eqref{eq:R}, the  population risk of iMAML is given by 
  \begin{subequations}
  \begin{align}\label{eq:biMAML_model_err}
  \mathcal{R}^{\mathrm{im}}({\theta_0} )
& \coloneqq 
    \lim_{N_a \to \infty} 
    \mathcal{R}_{N_a}^{\mathrm{im}}({\theta_0} )
    = \mathbb{E}_{m}\big[\|\theta_0-\theta^{\star}_{m}\|^2_{\mathbf{W}_{m}^{\mathrm{im}}}\big] + 1 \\
    \label{eq:W_tau_bimaml}
  {\rm with}~~~  &\mathbf{W}_{m}^{\mathrm{im}} =(\gamma ^{-1}\mathbf{Q}_{m}+\mathbf{I})^{-1}\mathbf{Q}_{m}(\gamma ^{-1}\mathbf{Q}_{m}+\mathbf{I})^{-1}
  \end{align}
  \end{subequations}
  whose minimizer is given by
  \begin{align}\label{eq:theta_0_star_biMAML_sln}
  \theta_{0}^{\mathrm{im}} 
    = \mathop{\arg\min}_{\theta_0}   \mathcal{R}^{\mathrm{im}}({\theta_0} )
    =\mathbb{E}_{m}\big[\mathbf{W}_{m}^{\mathrm{im}}\big]^{-1} \mathbb{E}_{m}\big[\mathbf{W}_{m}^{\mathrm{im}} {\theta}_{m}^{\star}\big].
  \end{align}

The above discussion provides proof for Proposition~\ref{prop:solutions}.


\section{Proof of Theorem~\ref{thm:maml_excess_risk_bound}} 
\label{sec:proof_main}
Section~\ref{sec:proof_solutions} gives solutions to the empirical and population risks. In this section, we provide proof to the main theorem, starting with the decomposition of the excess risk in Proposition~\ref{prop:excess_risk}.
Note that our proof of the bound on the variance  follows the idea of~\citep{Bartlett_benign_linear} by separately bounding the terms related to the first $k$ largest eigenvalues and the rest eigenvalues of the per-task weight matrices.

\subsection{Proof of Proposition \ref{prop:excess_risk}}

Next we analyze the excess risk defined in \eqref{eq:def_excess_risk} based on the solutions of MAML and iMAML.
First we restate the complete version of Proposition~2 in Lemma~\ref{lm:prop2_restate}. 

\begin{lemma}[Restatement of Proposition~2]
\label{lm:prop2_restate}
With probability at least $1 - \delta$, the excess risk of the MAML with the minimum-norm solution is bounded by 
    \begin{align*}
      \label{eq:excess_risk_bound_appendix}
      \mathcal{E}^{\cal A} (\hat{\theta}_0)
      \lesssim 
      \underbracket{\Big\|(\sum_{m=1}^{M} \hat{\mathbf{W}}_{m}^{\cal A})^{\dag}(\sum_{m=1}^{M} \hat{\mathbf{W}}_{m}^{\cal A}(\theta_{m}^{\star}-\theta_{0}))\Big\|_{\mathbf{W}^{\cal A}}^{2} }_{\mathcal{E}_{\theta^*_m}}
      +\underbracket{ \theta_0^{\top} \mathbf{B}^{\cal A} \theta_0}_{\mathcal{E}_{b}}
      + \underbracket{ {c_1} \sigma^2 \log \frac{1}{\delta} \mathrm{Tr}(\mathbf{C}^{\cal A} )}_{\mathcal{E}_{\epsilon_m}}
      \numberthis
    \end{align*}
where the weight matrix and the constants are defined as 
  \begin{align*}
  &\mathbf{W}^{\cal A} 
  \coloneqq  {\mathbb{E}_m[\mathbf{W}_m^{\cal A}]} ,
  ~~ \tilde{\mathbf{X}}^{\rm ma} \coloneqq 
  [\mathbf{X}_m^{\rm va} (\mathbf{I} - \alpha \hat{\mathbf{Q}}_m^{\rm tr} ) ],
  ~~ \tilde{\mathbf{X}}^{\rm im} \coloneqq 
  [\mathbf{X}_m^{\rm va} (\mathbf{I} + \gamma^{-1} \hat{\mathbf{Q}}_m^{\rm tr} )^{-1} ] \\
    &\mathbf{B}^{\cal A} \coloneqq  \Big(\tilde{\mathbf{X}}^{\cal A\top}(\tilde{\mathbf{X}}^{\cal A} \tilde{\mathbf{X}}^{\cal A \top})^{-1} \tilde{\mathbf{X}}^{\cal A}-\mathbf{I}\Big) 
    \mathbf{W}^{\cal A}
    \Big(\tilde{\mathbf{X}}^{\cal A\top}(\tilde{\mathbf{X}}^{\cal A} \tilde{\mathbf{X}}^{\cal A\top})^{-1} \tilde{\mathbf{X}}^{\cal A}-\mathbf{I}\Big), \\
    &\mathbf{C}^{\cal A} = \bC^{\cal A}_1 + \bC^{\cal A}_2, ~~
    \mathbf{C}_1^{\cal A}\coloneqq 
    (\tilde{\mathbf{X}}^{\cal A} \tilde{\mathbf{X}}^{\cal A\top})^{-1} \tilde{\mathbf{X}}^{\cal A} 
    \mathbf{W}^{\cal A} 
     \tilde{\mathbf{X}}^{\cal A \top}(\tilde{\mathbf{X}}^{\cal A} \tilde{\mathbf{X}}^{\cal A \top})^{-1}, \\
    &
    \mathbf{C}_2^{\rm ma} \coloneqq
    \frac{\alpha^2}{N_{\rm tr}}
     \mathbf{C}_1^{\rm ma}
     \mathrm{diag}[\mathbf{X}_m^{\rm va} 
     \hat{\mathbf{Q}}_m^{\rm tr} \mathbf{X}_m^{\rm va \top} ] \\
     &\mathbf{C}_2^{\rm im} \coloneqq
     \frac{1}{N_{\rm tr}}
    \mathbf{C}_1^{\rm im}
     \mathrm{diag}[\mathbf{X}_m^{\rm va} (\mathbf{I} + \gamma^{-1} \hat{\mathbf{Q}}_m^{\rm tr} )^{-1}
     \hat{\mathbf{Q}}_m^{\rm tr} (\mathbf{I} + \gamma^{-1} \hat{\mathbf{Q}}_m^{\rm tr} )^{-1} \mathbf{X}_m^{\rm va \top} ].
  \end{align*}
  Note that $\mathbf{C}_2^{\cal A}$ can be either $\mathbf{C}_2^{\rm ma}$ for MAML or $\mathbf{C}_2^{\rm im}$ for iMAML.
\end{lemma}

\begin{proof}
  The excess risk $\mathcal{E}^{\cal A}$ can be derived as
  \begin{align*}\label{eq.app-pf-lem4-1}
    & \mathcal{E}^{\cal A}(\hat{\theta}_{0})
      \coloneqq 
      \mathcal{R}(\hat{\theta}_0) - \mathcal{R}({\theta}_0) 
      =\mathbb{E}_{m}\big[\|\hat{\theta}_0-\theta^{\star}_{m}\|^2_{\mathbf{W}_{m}^{\cal A}}\big] 
      - \mathbb{E}_{m}\big[\|\theta_0-\theta^{\star}_{m}\|^2_{\mathbf{W}_{m}^{\cal A}}\big] \\
        =& \hat{\theta}_{0}^{\top} \mathbf{W} \hat{\theta}_{0}-\theta_{0}^{\top} \mathbf{W} \theta_{0}-2 (\hat{\theta}_{0}-\theta_{0} )^{\top} \mathbb{E}_{m} [\mathbf{W}_{m} \theta_{m}^{\star} ] 
        = \hat{\theta}_{0}^{\top} \mathbf{W} \hat{\theta}_{0}-\theta_{0}^{\top} \mathbf{W} \theta_{0}-2 (\hat{\theta}_{0}-\theta_{0} )^{\top} \mathbf{W} \theta_{0} \\
        =& \hat{\theta}_{0}^{\top} \mathbf{W} \hat{\theta}_{0}-2 \hat{\theta}_{0}^{\top} \mathbf{W} \theta_{0}+\theta_{0}^{\top} \mathbf{W} \theta_{0}
      =\|\hat{\theta}_{0}-\theta_{0}\|^{2}_{\mathbf{W}^{\cal A}} \\
      =&\Big\|\Big(\sum_{m=1}^{M} \hat{\mathbf{W}}_{m}\Big)^{\dag}\Big(\sum_{m=1}^{M} \hat{\mathbf{W}}_{m} \theta_{m}^{\star}\Big)+\Delta_{M}-\theta_{0}\Big\|^2_{\mathbf{W}^{\cal A}} \\
      \leq & 2\underbracket{\Big\|\Big({\sum}_{m} \hat{\mathbf{W}}_{m}\Big)^{\dag}\Big({\sum}_{m} \hat{\mathbf{W}}_{m} \theta_{m}^{\star}\Big)-\theta_{0}\Big\|^2_{\mathbf{W}^{\cal A}}}_{I_1}
      +2\underbracket{\big\|\Delta_{M}\big\|^{2}_{\mathbf{W}^{\cal A}}}_{I_2}. \numberthis
  \end{align*}
In \eqref{eq.app-pf-lem4-1}, $I_1$ can be bounded by
\begin{align*}\label{eq:bound_I1_prop2}
  I_1 
  =&\left\|\left({\sum}_{m} \hat{\mathbf{W}}_{m}\right)^{\dag}\left({\sum}_{m} \hat{\mathbf{W}}_{m} \theta_{m}^{\star}\right)-\theta_{0}\right\|^2_{\mathbf{W}^{\cal A}}\\
  =& \left\|\Big({\sum}_{m} \hat{\mathbf{W}}_{m}\Big)^{\dag}\Big({\sum}_{m} \hat{\mathbf{W}}_{m} (\theta_{m}^{\star}-\theta_{0})\Big)
  + \Big(\Big({\sum}_{m} \hat{\mathbf{W}}_{m}\Big)^{\dag}\Big({\sum}_{m} \hat{\mathbf{W}}_{m}) - \mathbf{I}\Big) \theta_{0}\right\|^2_{\mathbf{W}^{\cal A}} \\
  \leq &
  2\Big\|\Big({\sum}_{m} \hat{\mathbf{W}}_{m}\Big)^{\dag}\Big({\sum}_{m} \hat{\mathbf{W}}_{m} (\theta_{m}^{\star}-\theta_{0}) \Big)\Big\|^2_{\mathbf{W}^{\cal A}} 
  + 2\Big\|\Big(\Big({\sum}_{m} \hat{\mathbf{W}}_{m}\Big)^{\dag}\Big({\sum}_{m} \hat{\mathbf{W}}_{m}) - \mathbf{I}\Big) \theta_{0}\Big\|^2_{\mathbf{W}^{\cal A}} \\
  = &
  2\Big\|\Big({\sum}_{m} \hat{\mathbf{W}}_{m}\Big)^{\dag}\Big({\sum}_{m} \hat{\mathbf{W}}_{m} (\theta_{m}^{\star}-\theta_{0})\Big)\Big\|^2_{\mathbf{W}^{\cal A}} 
  + 2\theta_0^{\top} \mathbf{B} \theta_0 \numberthis
\end{align*}
with the matrix $\mathbf{B}$ defined as  
\begin{align*}
  \mathbf{B} 
  =& \Big(\Big({\sum}_{m} \hat{\mathbf{W}}_{m}\Big)^{\dag}\Big({\sum}_{m} \hat{\mathbf{W}}_{m}) - \mathbf{I}\Big) \mathbf{W}^{\cal A}
  \Big(\Big({\sum}_{m} \hat{\mathbf{W}}_{m}\Big)^{\dag}\Big({\sum}_{m} \hat{\mathbf{W}}_{m}\Big) - \mathbf{I}\Big) \\
    \stackrel{(a)}{=} &
  \big((\tilde{\mathbf{X}}^{\top}\tilde{\mathbf{X}})^{\dag} \tilde{\mathbf{X}}^{\top} \tilde{\mathbf{X}}-\mathbf{I}\big) 
  \mathbf{W}^{\cal A}
  \big((\tilde{\mathbf{X}}^{\top}\tilde{\mathbf{X}})^{\dag} \tilde{\mathbf{X}}^{\top} \tilde{\mathbf{X}}-\mathbf{I}\big) \\
  = &
  \big(\tilde{\mathbf{X}}^{\top}(\tilde{\mathbf{X}} \tilde{\mathbf{X}}^{\top})^{-1} \tilde{\mathbf{X}}-\mathbf{I}\big) 
  \mathbf{W}^{\cal A}
  \big(\tilde{\mathbf{X}}^{\top}(\tilde{\mathbf{X}} \tilde{\mathbf{X}}^{\top})^{-1} \tilde{\mathbf{X}}-\mathbf{I}\big) 
  \numberthis .
\end{align*}
And {(a) is from the relationship of $\hat{\mathbf{W}}$ and $\tilde{\mathbf{X}}$}, recall we use $[\cdot]$ to represent row concatenation of matrices or vectors. 

In \eqref{eq.app-pf-lem4-1}, $I_2$ can be bounded by
\begin{align*}
  I_2   &= \Big\| \Big(\sum_{m=1}^{M} \hat{\mathbf{W}}_{m}^{\cal A}
  \Big)^{\dag}
  \Big(\sum_{m=1}^{M}(\mathbf{I} -\alpha \hat{\mathbf{Q}}_{m}^{\rm tr}) \frac{1}{N_{2}} \mathbf{X}_{m}^{\rm va \top} \mathbf{e}_{m}^{\rm va}
  - (\mathbf{I}-\alpha \hat{\mathbf{Q}}_{m}^{\rm tr}) \hat{\mathbf{Q}}_{m}^{\rm va} \frac{\alpha}{N_{\rm tr}} \mathbf{X}_{m}^{\rm tr{\top}} \mathbf{e}_{m}^{\rm tr}
  \Big) \Big\|^2_{\mathbf{W}^{\cal A}} \nonumber \\
  &    \stackrel{(b)}{=}   [\mathbf{e}_m^{\rm va}]^{\top} \mathbf{C}_1^{\cal A} [\mathbf{e}_m^{\rm va}] + 
  [\mathbf{e}_m^{\rm tr}]^{\top} \mathbf{C}_2^{\cal A} [\mathbf{e}_m^{\rm tr}] 
  -2 [\mathbf{e}_m^{\rm va}]^{\top} C_3^{\cal A} [\mathbf{e}_m^{\rm tr}]\\
  &\leq 
  2[\mathbf{e}_m^{\rm va}]^{\top} \mathbf{C}_1^{\cal A} [\mathbf{e}_m^{\rm va}] + 
  2[\mathbf{e}_m^{\rm tr}]^{\top} \mathbf{C}_2^{\cal A} [\mathbf{e}_m^{\rm tr}] \\
  &= 2\mathrm{Tr} (\mathbf{C}_1^{\cal A} [\mathbf{e}_m^{\rm va}][\mathbf{e}_m^{\rm va}]^{\top} + \mathbf{C}_2^{\cal A} [\mathbf{e}_m^{\rm tr}] [\mathbf{e}_m^{\rm tr}]^{\top}) \\
  &= 2\mathrm{Tr} (\mathbf{C}_1^{\cal A} + \mathbf{C}_2^{\cal A})
  + 2\mathrm{Tr} \big(\mathbf{C}_1^{\cal A} ([\mathbf{e}_m^{\rm va}][\mathbf{e}_m^{\rm va}]^{\top} - \mathbf{I}) + \mathbf{C}_2^{\cal A} ([\mathbf{e}_m^{\rm tr}] [\mathbf{e}_m^{\rm tr}]^{\top} - \mathbf{I}) \big)
\end{align*}
where $(b)$ follows from expanding the quadratic terms, and 
\begin{align}
\label{eq:C1A}
  \mathbf{C}_1 ^{\cal A}
  &= \frac{1}{N^2}\tilde{\mathbf{X}}
  \Big(\sum_{m=1}^{M} \hat{\mathbf{W}}_{m}^{\cal A}
  \Big)^{\dag} {\mathbf{W}}^{\cal A}
  \Big(\sum_{m=1}^{M} \hat{\mathbf{W}}_{m}^{\cal A}
  \Big)^{\dag}
  \tilde{\mathbf{X}}^{\top} \nonumber\\
  &= \tilde{\mathbf{X}}
  \big(\tilde{\mathbf{X}}^{\top} \tilde{\mathbf{X}}
  \big)^{\dag} {\mathbf{W}}^{\cal A}
  \big(\tilde{\mathbf{X}}^{\top} \tilde{\mathbf{X}}
  \big)^{\dag}
  \tilde{\mathbf{X}}^{\top} 
  = \big(\tilde{\mathbf{X}}
  \tilde{\mathbf{X}}^{\top}\big)^{-1} \tilde{\mathbf{X}}
  {\mathbf{W}}^{\cal A}
  \tilde{\mathbf{X}}^{\top}
  \big(\tilde{\mathbf{X}}
  \tilde{\mathbf{X}}^{\top}\big)^{-1}, \\
  \mathbf{C}_2 ^{\rm ma}
  &= \frac{\alpha^2}{N_{\rm tr}^2}[\mathbf{X}_m^{\rm tr} \mathbf{X}_m^{\rm va \top} \tilde{\mathbf{X}}_m]
  \Big(\sum_{m=1}^{M} \hat{\mathbf{W}}_{m}^{\cal A}
  \Big)^{\dag} {\mathbf{W}}^{\cal A}
  \Big(\sum_{m=1}^{M} \hat{\mathbf{W}}_{m}^{\cal A}
  \Big)^{\dag}
  [\mathbf{X}_m^{\rm tr} \mathbf{X}_m^{\rm va \top} \tilde{\mathbf{X}}_m]^{\top} \nonumber\\
  &= \frac{\alpha^2}{N_{\rm tr}^2}[\mathbf{X}_m^{\rm tr} \mathbf{X}_m^{\rm va \top} \tilde{\mathbf{X}}_m]
  \tilde{\mathbf{X}}^{\top}
  \big(\tilde{\mathbf{X}}
  \tilde{\mathbf{X}}^{\top}\big)^{-2} \tilde{\mathbf{X}}
  {\mathbf{W}}^{\cal A}
  \tilde{\mathbf{X}}^{\top}
  \big(\tilde{\mathbf{X}}
  \tilde{\mathbf{X}}^{\top}\big)^{-2} \tilde{\mathbf{X}}
  [\mathbf{X}_m^{\rm tr} \mathbf{X}_m^{\rm va \top} \tilde{\mathbf{X}}_m]^{\top}. 
\end{align}

By taking the expectation w.r.t. $\mathbf{e}_m$, we need to bound $\mathrm{Tr}(\mathbf{C}_1), \mathrm{Tr}(\mathbf{C}_2)$.
Based on the cyclic property of trace,
$\mathrm{Tr}(\mathbf{C}_2^{\rm ma})$ can be further derived as 
\begin{align*}
    \mathrm{Tr}(\mathbf{C}_2^{\rm ma}) 
    &= \frac{\alpha^2}{N_{\rm tr}^2}
    \mathrm{Tr} \Big([\mathbf{X}_m^{\rm tr} \mathbf{X}_m^{\rm va \top} \tilde{\mathbf{X}}_m]
    \tilde{\mathbf{X}}^{\top}
    \big(\tilde{\mathbf{X}}
    \tilde{\mathbf{X}}^{\top}\big)^{-2} \tilde{\mathbf{X}}
     {\mathbf{W}}^{\cal A}
     \tilde{\mathbf{X}}^{\top}
     \big(\tilde{\mathbf{X}}
     \tilde{\mathbf{X}}^{\top}\big)^{-2} \tilde{\mathbf{X}}
     [\mathbf{X}_m^{\rm tr} \mathbf{X}_m^{\rm va \top} \tilde{\mathbf{X}}_m]^{\top} \Big)\\
    &= \frac{\alpha^2}{N_{\rm tr}^2} \mathrm{Tr} \Big(
    \tilde{\mathbf{X}}^{\top}
    \big(\tilde{\mathbf{X}}
    \tilde{\mathbf{X}}^{\top}\big)^{-2} \tilde{\mathbf{X}}
     {\mathbf{W}}^{\cal A}
     \tilde{\mathbf{X}}^{\top}
     \big(\tilde{\mathbf{X}}
     \tilde{\mathbf{X}}^{\top}\big)^{-2} \tilde{\mathbf{X}}
     \sum_{m=1}^{M}
     \tilde{\mathbf{X}}_m^{\top} \mathbf{X}_m^{\rm va} \mathbf{X}_m^{t \top} 
     \mathbf{X}_m^{\rm tr} \mathbf{X}_m^{\rm va \top} \tilde{\mathbf{X}}_m
     \Big)\\
     &= \frac{\alpha^2}{N_{\rm tr}^2} \mathrm{Tr} \Big(
    \tilde{\mathbf{X}}^{\top}
    \big(\tilde{\mathbf{X}}
    \tilde{\mathbf{X}}^{\top}\big)^{-2} \tilde{\mathbf{X}}
     {\mathbf{W}}^{\cal A}
     \tilde{\mathbf{X}}^{\top}
     \big(\tilde{\mathbf{X}}
     \tilde{\mathbf{X}}^{\top}\big)^{-2} \tilde{\mathbf{X}}
     \tilde{\mathbf{X}}^{\top}
     [\mathbf{X}_m^{\rm va} \mathbf{X}_m^{t \top} 
     \mathbf{X}_m^{\rm tr} \mathbf{X}_m^{\rm va \top} \tilde{\mathbf{X}}_m]
     \Big) .
  \end{align*}

Then $\mathrm{Tr}(\mathbf{C}_2^{\rm ma})$ can be further written as
\begin{align*}
  \mathrm{Tr}(\mathbf{C}_2^{\rm ma})&= \frac{\alpha^2}{N_{\rm tr}^2} \mathrm{Tr} \Big(
    \tilde{\mathbf{X}}^{\top}
    \big(\tilde{\mathbf{X}}
    \tilde{\mathbf{X}}^{\top}\big)^{-2} \tilde{\mathbf{X}}
     {\mathbf{W}}^{\cal A}
     \tilde{\mathbf{X}}^{\top}
     \big(\tilde{\mathbf{X}}
     \tilde{\mathbf{X}}^{\top}\big)^{-1} 
     \mathrm{diag} [\mathbf{X}_m^{\rm va} \mathbf{X}_m^{t \top} 
     \mathbf{X}_m^{\rm tr} \mathbf{X}_m^{\rm va \top} ] \tilde{\mathbf{X}}
     \Big) \\
     &= \frac{\alpha^2}{N_{\rm tr}^2}
     \mathrm{Tr}\Big(
     \big(\tilde{\mathbf{X}}
     \tilde{\mathbf{X}}^{\top}\big)^{-1} \tilde{\mathbf{X}}
     {\mathbf{W}}^{\cal A} 
     \tilde{\mathbf{X}}^{\top}
     \big(\tilde{\mathbf{X}}
     \tilde{\mathbf{X}}^{\top}\big)^{-1} 
     \mathrm{diag}[\mathbf{X}_m^{\rm va} \mathbf{X}_m^{\rm tr\top} 
     \mathbf{X}_m^{\rm tr} \mathbf{X}_m^{\rm va \top} ]
     \Big ) \\
     &= \frac{\alpha^2}{N_{\rm tr}}
     \mathrm{Tr}\Big(\big(\tilde{\mathbf{X}}
     \tilde{\mathbf{X}}^{\top}\big)^{-1} \tilde{\mathbf{X}}
     {\mathbf{W}}^{\cal A} 
     \tilde{\mathbf{X}}^{\top}
     \big(\tilde{\mathbf{X}}
     \tilde{\mathbf{X}}^{\top}\big)^{-1} 
     \mathrm{diag}[\mathbf{X}_m^{\rm va} 
     \hat{\mathbf{Q}}_m^{\rm tr} \mathbf{X}_m^{\rm va \top} ] 
     \Big) \numberthis .
\end{align*}

Since we have 
\begin{align*}
  \mathbb{E}_{\epsilon}[I_2] = &
\mathbb{E}_{\epsilon}\big[
  [\mathbf{e}_m^{\rm va}]^{\top} \mathbf{C}_1 [\mathbf{e}_m^{\rm va}]
  \big] + 
\mathbb{E}_{\epsilon}\big[
  [\mathbf{e}_m^{\rm tr}]^{\top} \mathbf{C}_2 [\mathbf{e}_m^{\rm tr}]
  \big]\\
  = &\mathrm{Tr}(\mathbf{C}_1 \operatorname{Cov}[[\mathbf{e}_m^{\rm va}]]) + 
  \mathrm{Tr}(\mathbf{C}_2 \operatorname{Cov}[[\mathbf{e}_m^{\rm tr}]])
  = \sigma^2 \mathrm{Tr}(\mathbf{C}_1+\mathbf{C}_2)
\end{align*}
by the subGaussian concentration inequality~\citep{vershynin2018high}, 
it holds with probability at least $1-\delta$ over $\epsilon$ that
\begin{align}\label{eq:bound_I2_prop2}
  2I_2 \leq {c_1} \sigma^2 \log \frac{1}{\delta} \mathrm{Tr}(\mathbf{C}_1 + \mathbf{C}_2).
\end{align}

Combining the bounds for \(I_1\) and \(I_2\) in \eqref{eq:bound_I1_prop2} and \eqref{eq:bound_I2_prop2} completes the proof.
\end{proof}


\subsection{Proof of Lemma \ref{lm:bound_task_para_dist}}

Define 
  \begin{align*}
    &\Delta_{\theta_{\mathcal{A}}} 
    \coloneqq 
    \begin{bmatrix}
    (\theta_1^{\star} - \theta_0^{\mathcal{A}})^{\top}, 
    \dots ,
    (\theta_{M}^{\star} - \theta_0^{\mathcal{A}})^{\top}
    \end{bmatrix}^{\top} \in \mathbb{R}^{dM}, \\
    &\mathbf{U}_{\mathcal{A}} 
    \coloneqq  
    \Big[
    \hat{\mathbf{W}}^{\mathcal{A}}_{1} 
    \Big(\sum_{m=1}^{M} \hat{\mathbf{W}}_{m}^{\mathcal{A}}\Big)^{\dag} ,
    \dots, 
    \hat{\mathbf{W}}^{\mathcal{A}}_{M}
    \Big(\sum_{m=1}^{M} \hat{\mathbf{W}}_{m}^{\mathcal{A}}\Big)^{\dag}
    \Big]^{\top}
    \in \mathbb{R}^{dM\times d}.
  \end{align*}
  Then we can derive that
  \begin{align*}
    \Big\|\Big(\sum_{m=1}^{M} \hat{\mathbf{W}}_{m}^{\cal A}\Big)^{\dag}\Big(\sum_{m=1}^{M} \hat{\mathbf{W}}_{m}^{\cal A}(\theta_{m}^{\star}-\theta_{0})\Big) \Big\|_{\mathbf{W}^{\cal A}}^{2}
    = & \|\mathbf{U}_{\mathcal{A}}^{\top} \Delta_{\theta_{\mathcal{A}}}\|^2_{\mathbf{W}^{\cal A}}.
  \end{align*}
  By the Hanson-Wright inequality, with probability at least $1 - \delta$ over $\theta_m^{\star}$, we have 
  \begin{align*}\label{eq:HW_ineq_general_ref}
 \left|\Big\| \mathbf{U}_{\mathcal{A}}^{\top} \Delta_{\theta_{\mathcal{A}}}\Big\|^2_{\mathbf{W}^{\cal A}}
    - \mathbb{E}_{\theta_{m}^{\star} \mid \hat{\mathbf{W}}_{m}^{\cal A}}\Big[\Big\| \mathbf{U}_{\mathcal{A}}^{\top} \Delta_{\theta_{\mathcal{A}}}\Big\|^2_{\mathbf{W}^{\cal A}}\Big]\right|  
    =\widetilde{\mathcal{O}}\Big(\frac{R^{2}}{M\sqrt{d}}  \Big).\numberthis
  \end{align*}
  
  To compute $\mathbb{E}_{\theta_{m}^{\star} \mid \hat{\mathbf{W}}_{m}^{\cal A}}\big[\big\| \mathbf{U}_{\mathcal{A}}^{\top} \Delta_{\theta_{\mathcal{A}}}\big\|^2_{\mathbf{W}^{\cal A}} \big]$, first recall $\mathrm{Cov}[\theta_m^{\star}] = \frac{R^2}{d} \mathbf{I}$, then we have
  \begin{align*}
  & \mathbb{E}_{\theta_{m}^{\star} \mid \hat{\mathbf{W}}_{m}^{\cal A}}[\Delta_{\theta_{\mathcal{A}}}^{\top}\mathbf{U}_{\mathcal{A}} 
  \mathbf{W}^{\cal A}
  \mathbf{U}_{\mathcal{A}}^{\top} \Delta_{\theta_{\mathcal{A}}}] 
  = \frac{R^{2}}{ d}\Big\langle\Big({\sum_{m=1}^{M} \hat{\mathbf{W}}_{m}^{\cal A}}\Big)^{\dag}
  \mathbf{W}^{\cal A}
  \Big({\sum_{m=1}^{M}\hat{\mathbf{W}}_{m}^{\cal A}}\Big)^{\dag}, 
  {\sum_{m=1}^{M} (\hat{\mathbf{W}}_{m}^{\cal A})^{2}}\Big\rangle 
  \label{eq:E_mse_general_3term}  \\
  =& \frac{R^{2}}{ d} 
  \mathrm{Tr}\Bigg(
  \tilde{\mathbf{X}}  \Big({\sum_{m=1}^{M} \hat{\mathbf{W}}_{m}^{\cal A}}\Big)^{\dag}
  \mathbf{W}^{\cal A}
  \Big({\sum_{m=1}^{M}\hat{\mathbf{W}}_{m}^{\cal A}}\Big)^{\dag} \tilde{\mathbf{X}}^{\top}
  \mathrm{diag} [\tilde{\mathbf{X}}_m\tilde{\mathbf{X}}_m^{\top}] 
  \Bigg) \\
  =& \frac{R^{2}}{ d} 
  \mathrm{Tr}\Big(
  \tilde{\mathbf{X}} 
  \big(\tilde{\mathbf{X}}^{\top}\tilde{\mathbf{X}}\big)^{\dag}
  \mathbf{W}^{\cal A}
  \big(\tilde{\mathbf{X}}^{\top}\tilde{\mathbf{X}}\big)^{\dag} \tilde{\mathbf{X}}^{\top}
  \mathrm{diag} [\tilde{\mathbf{X}}_m\tilde{\mathbf{X}}_m^{\top}] 
  \Big) \\
  =& \frac{R^{2}}{ d} \mathrm{Tr} (\mathbf{C}_1^{\cal A}
  \mathrm{diag} [\tilde{\mathbf{X}}_m\tilde{\mathbf{X}}_m^{\top}]
  )\leq \frac{R^{2}}{ d} \mathrm{Tr} (\mathbf{C}_1^{\cal A}
  ) \|\mathrm{diag} [{\mathbf{X}}_m^{\rm va}{\mathbf{X}}_m^{\rm va\top}]\| \\
  \leq & \frac{R^{2}}{ d} \mathrm{Tr} (\mathbf{C}_1^{\cal A}
  ) \max_{m\in [M]}
  \|{\mathbf{X}}_m^{\rm va}{\mathbf{X}}_m^{\rm va\top} \| 
  \leq \frac{R^{2}}{ d} \mathrm{Tr} (\mathbf{C}_1^{\cal A}
  ) \max_{m\in [M]}
  \|{\mathbf{X}}_m^{\rm va}{\mathbf{X}}_m^{\rm va\top} \| 
  \numberthis
  \end{align*}
  where from Lemma~\ref{lm:bound_Q_hat_norm}, with high probability
  $\|{\mathbf{X}}_m^{\rm va}{\mathbf{X}}_m^{\rm va\top} \|$ can be bounded by
  \begin{align*}\label{eq:XXT_bound}
    \|{\mathbf{X}}_m^{\rm va}{\mathbf{X}}_m^{\rm va\top} \|
    =\|{\mathbf{X}}_m^{\rm va\top} {\mathbf{X}}_m^{\rm va}\|
    \lesssim & 
    \Big(\sum_{i=1}^{d} \lambda_{mi}^2 + \lambda_{m1}^2 N_2 \Big)
    \leq 
    \mathcal{O}(N_{\rm va})
    . \numberthis
  \end{align*}
Combining \eqref{eq:HW_ineq_general_ref}, \eqref{eq:E_mse_general_3term} and \eqref{eq:XXT_bound} leads to the following with high probability
\begin{align*}
  \mathbb{E}_{\theta_{m}^{\star} \mid \hat{\mathbf{W}}_{m}^{\cal A}}[\Delta_{\theta_{\mathcal{A}}}^{\top}\mathbf{U}_{\mathcal{A}} 
  \mathbf{W}^{\cal A}
  \mathbf{U}_{\mathcal{A}}^{\top} \Delta_{\theta_{\mathcal{A}}}] 
  \leq \frac{R^{2}}{ d} \mathrm{Tr} (\mathbf{C}_1^{\cal A}
  ) \max_{m\in [M]}
  \|{\mathbf{X}}_m^{\rm va}{\mathbf{X}}_m^{\rm va\top} \| 
  \lesssim \frac{R^2 N_{\rm va}}{d}\mathrm{Tr}(\mathbf{C}_1^{\cal A}) 
\end{align*}
which proves that this term $\mathbb{E}_{\theta_{m}^{\star} \mid \hat{\mathbf{W}}_{m}^{\cal A}}[\Delta_{\theta_{\mathcal{A}}}^{\top}\mathbf{U}_{\mathcal{A}} \mathbf{W}^{\cal A}
\mathbf{U}_{\mathcal{A}}^{\top} \Delta_{\theta_{\mathcal{A}}}]$
  is non-dominant compared to $\mathrm{Tr}(\mathbf{C}_1^{\cal A})$.

\subsection{Proof of Lemma \ref{lm:bound_B}}
\begin{proof}
  Recall $\mathbf{B}\coloneqq  \Big(\tilde{\mathbf{X}}^{\top}(\tilde{\mathbf{X}} \tilde{\mathbf{X}}^{\top})^{-1} \tilde{\mathbf{X}}-\mathbf{I}\Big) 
  \mathbf{W}
  \Big(\tilde{\mathbf{X}}^{\top}(\tilde{\mathbf{X}} \tilde{\mathbf{X}}^{\top})^{-1} \tilde{\mathbf{X}}-\mathbf{I}\Big)$.
  First note that
  \begin{align}\label{eq:X0}
    \Big(\tilde{\mathbf{X}}^{\top}(\tilde{\mathbf{X}} \tilde{\mathbf{X}}^{\top})^{-1} \tilde{\mathbf{X}}-\mathbf{I}\Big) 
  \tilde{\mathbf{X}}^{\top}
  =\tilde{\mathbf{X}}^{\top} - \tilde{\mathbf{X}}^{\top}
  =\mathbf{0}.
  \end{align}
  Thus, for any $\mathbf{u}$ in the column space of
  $\tilde{\mathbf{X}}^{\top}$, $\mathbf{u}$ can be represented as $\mathbf{u} = \tilde{\mathbf{X}}^{\top} \bar{\mathbf{u}}, \bar{\mathbf{u}}\neq \mathbf{0}$, then we have 
  \begin{align}
      \Big(\tilde{\mathbf{X}}^{\top}(\tilde{\mathbf{X}} \tilde{\mathbf{X}}^{\top})^{-1} \tilde{\mathbf{X}}-\mathbf{I}\Big) 
      \mathbf{u}
    =\mathbf{0}.
    \end{align}
  And for any $\mathbf{u}$ orthogonal to the colomn space of
  $\tilde{\mathbf{X}}^{\top}$, $\tilde{\mathbf{X}}\mathbf{u}=\mathbf{0}$, therefore
  \begin{align}
    \Big(\tilde{\mathbf{X}}^{\top}(\tilde{\mathbf{X}} \tilde{\mathbf{X}}^{\top})^{-1} \tilde{\mathbf{X}}-\mathbf{I}\Big) \mathbf{u} = - \mathbf{u}.
  \end{align}
  Since any $\mathbf{u} \in \mathbb{R}^d$ can be represented as a combination of a vector in the colomn space of
  $\tilde{\mathbf{X}}^{\top}$ and a vector orthogonal to the colomn space of
  $\tilde{\mathbf{X}}^{\top}$,
  $\big(\tilde{\mathbf{X}}^{\top}(\tilde{\mathbf{X}} \tilde{\mathbf{X}}^{\top})^{-1} \tilde{\mathbf{X}}-\mathbf{I}\big)$ has eigenvalues whose absolute values are smaller than $1$, i.e. 
  \begin{align}
  \label{eq:norm_X_pinv}
    \big\|\tilde{\mathbf{X}}^{\top}(\tilde{\mathbf{X}} \tilde{\mathbf{X}}^{\top})^{-1} \tilde{\mathbf{X}}-\mathbf{I}\big\|\leq 1.
  \end{align}
   
  Then let $\mathbf{M}=\Big(\tilde{\mathbf{X}}^{\top}
  \big(\tilde{\mathbf{X}} \tilde{\mathbf{X}}^{\top}\big)^{-1} \tilde{\mathbf{X}}-\mathbf{I}\Big)$, expanding $\theta_0^{\top} \mathbf{B} \theta_0$, we have
\begin{align*}
\theta_0^{\top} \mathbf{B} \theta_0 
&=\theta_0^{\top}\Big(\tilde{\mathbf{X}}^{\top}\big(\tilde{\mathbf{X}} \tilde{\mathbf{X}}^{\top}\big)^{-1} \tilde{\mathbf{X}}-\mathbf{I}\Big) \mathbf{W}\Big(\tilde{\mathbf{X}}^{\top}\big(\tilde{\mathbf{X}} \tilde{\mathbf{X}}^{\top}\big)^{-1} \tilde{\mathbf{X}}-\mathbf{I}\Big) \theta_0 \\
& \stackrel{(a)}{=} \theta_0^{\top} \mathbf{M}\Big(\mathbf{W}-\frac{1}{M N_{\mathrm{va}}} \tilde{\mathbf{X}}^{\top} \tilde{\mathbf{X}}\Big) \mathbf{M} \theta_0 \\
&=\theta_0^{\top} \mathbf{M}\Big(\mathbf{W}-\frac{1}{M N_{\mathrm{va}}} \bar{\mathbf{X}}^{\top} \bar{\mathbf{X}}+\frac{1}{M N_{\mathrm{va}}} \bar{\mathbf{X}}^{\top} \bar{\mathbf{X}}-\frac{1}{M N_{\mathrm{va}}} \tilde{\mathbf{X}}^{\top} \tilde{\mathbf{X}}\Big) \mathbf{M} \theta_0 \\
& \stackrel{(b)}{\leq}
\Big\|\mathbf{W}-\frac{1}{M N_{\mathrm{va}}} \bar{\mathbf{X}}^{\top} \bar{\mathbf{X}}\Big\| \|\theta_0\|^2
+\frac{1}{M N_{\mathrm{va}}}
\Big\|\bar{\mathbf{X}}^{\top} \bar{\mathbf{X}}-\tilde{\mathbf{X}}^{\top} \tilde{\mathbf{X}} \Big\| \|\theta_0 \|^2 
\numberthis
\end{align*}
where $(a)$ follows from \eqref{eq:X0}, and $(b)$ follows from \eqref{eq:norm_X_pinv}.

Thus, due to Lemma~\ref{lm:concentrate_sample_cov}, there is an absolute constant $c$ such that for any $1\leq t\leq MN_{\rm va}$ with probability at least $1-e^{-t}$ over $\bZ^{\rm va}$, it holds that
\begin{align}
  \Big\|\mathbf{W}-\frac{1}{M N_{\mathrm{va}}} \bar{\mathbf{X}}^{\top} \bar{\mathbf{X}}\Big\| \|\theta_0 \|^2 
  \leq c\left\|\theta_0\right\|^2\|\mathbf{W}\| \max \left\{\sqrt{\frac{r(\mathbf{W})}{M N_{\mathrm{va}}}}, \frac{r(\mathbf{W})}{M N_{\mathrm{va}}}, \sqrt{\frac{t}{M N_{\mathrm{va}}}} 
  \right\}
\end{align}

where $r(\mathbf{W})$ is defined as
\begin{align}
  r(\mathbf{W}):=\frac{(\mathbb{E}\|\bar{\bx}\|)^2}{\|\mathbf{W}\|} \leq \frac{\mathbb{E}\left(\|\bar{\bx}\|^2\right)}{\|\mathbf{W}\|}=\frac{\mathrm{Tr}(\mathbf{W})}{\|\mathbf{W}\|}=r_0(\mathbf{W}) .
\end{align}

The bound on $\big\|\bar{\mathbf{X}}^{\top} \bar{\mathbf{X}}-\tilde{\mathbf{X}}^{\top} \tilde{\mathbf{X}} \big\|$ can be found in Lemma~\ref{lm:eigenvalues_diff_tilde_bar}, which shows
when $|\alpha| < \min_m \min \{1/\lambda_{m1}, 1/ \mu_1(\mathbf{\Lambda}_{m}^{\frac{1}{2}} 
  \hat{\mathbf{D}}_m^{\rm tr} 
  \mathbf{\Lambda}_{m}^{\frac{1}{2}})\}$,
with probability at least $1 - 2M e^{-t}$ over $\bZ^{\rm tr}$ and $\bZ^{\rm va}$ for any $1 \leq t \leq  N_{\rm va}$, it holds that
\begin{align}
  \frac{1}{M N_{\mathrm{va}}} 
  \big\|\bar{\mathbf{X}}^{\top} \bar{\mathbf{X}}-\tilde{\mathbf{X}}^{\top} \tilde{\mathbf{X}} \big\|
  \leq & 
    \frac{c |\alpha|}{M}  \sum_{m=1}^{M} \lambda_{m1}^2
   \max \left\{\sqrt{\frac{r(\bW_m)}{N_{\mathrm{tr}}}}, \frac{r(\bW_m)}{N_{\mathrm{tr}}}, \sqrt{\frac{t}{N_{\mathrm{tr}}}}, \frac{t}{N_{\mathrm{tr}}}\right\}.
\end{align}

Applying the union bound we have 
for MAML with $|\alpha| < \min_m \min \{1/\lambda_{m1}, 1/ \mu_1(\mathbf{\Lambda}_{m}^{\frac{1}{2}} 
  \hat{\mathbf{D}}_m^{\rm tr} 
  \mathbf{\Lambda}_{m}^{\frac{1}{2}})\}$ and for iMAML with $\gamma > 0$, for any $1 \leq t \leq  N_{\rm va}$,
with probability at least $1 - (2M + 1) e^{-t}$ over $\bZ^{\rm tr}$ and $\bZ^{\rm va}$,  
there exists $c > 1$  that
\begin{align}
  \theta_0^{\top} \mathbf{B} \theta_0 
  \lesssim  
  \|\theta_0\|^2 
  & \|\mathbf{W}\| \max \Bigg\{\sqrt{\frac{r(\mathbf{W})}{M N_{\mathrm{va}}}}, \frac{r(\mathbf{W})}{M N_{\mathrm{va}}}, \sqrt{\frac{t}{M N_{\mathrm{va}}}} 
  \Bigg\} 
   .
\end{align}

The proof is complete.
\end{proof}

\subsection{Proof of Lemma \ref{lm:bound_C}}
To prove Lemma~\ref{lm:bound_C}, we need to bound $\mathrm{Tr}(\mathbf{C}) = \mathrm{Tr}(\mathbf{C}_1) + \mathrm{Tr}(\mathbf{C}_2)$.
We first show in Lemma~\ref{lm:C2_in_C1} that $\mathrm{Tr}(\mathbf{C}_2)$ can be bounded as $\boldsymbol{\Theta}(\mathrm{Tr}(\mathbf{C}_1))$.
Then the key step is to bound $\mathrm{Tr}(\mathbf{C}_1)$.
To bound $\mathrm{Tr}(\mathbf{C}_1)$, first we show in Lemma~\ref{lm:decompose_tr_C1} that $\mathrm{Tr}(\mathbf{C}_1)$ can be decomposed into terms that are related to the first $k$ largest eigenvalues of $\bW$  and the term that is only related to the rest eigenvalues.
Next we bound the term related to the $d-k$ smallest eigenvalues of, as a function of $\mu_n(\mathbf{A})$, given in Lemma~\ref{lm:last_terms}.
And then we bound the term related to the $k$ largest eigenvalues, given in Lemma~\ref{lm:first_k_terms}.
Finally, we bound the eigenvalues of $\mu_n(\mathbf{A})$ in Lemma~\ref{lm:eigenv_A}.

\begin{lemma}[Bound on $\mathrm{Tr}(\mathbf{C}_2^{\cal A})$ in terms of $\mathrm{Tr}(\mathbf{C}_1^{\cal A})$]
\label{lm:C2_in_C1}
  Recall $\alpha$ is the step size for MAML, $\gamma$ is the regularization parameter for iMAML,
  and %
\begin{align}
    \mathrm{Tr}(\mathbf{C}_2^{\rm ma}) 
     &= \frac{\alpha^2}{N_{\rm tr}}
     \mathrm{Tr}\Big(\mathbf{C}_1^{\rm ma} 
     \mathrm{diag}[\mathbf{X}_m^{\rm va} 
     \hat{\mathbf{Q}}_m^{\rm tr} \mathbf{X}_m^{\rm va \top} ] 
     \Big) , \\
     \mathrm{Tr}(\mathbf{C}_2^{\rm im} )
     &=
     \frac{1}{N_{\rm tr}}
    \mathrm{Tr}\Big(\mathbf{C}_1^{\rm im}
     \mathrm{diag}[\mathbf{X}_m^{\rm va} (\mathbf{I} + \gamma^{-1} \hat{\mathbf{Q}}_m^{\rm tr} )^{-1}
     \hat{\mathbf{Q}}_m^{\rm tr} (\mathbf{I} + \gamma^{-1} \hat{\mathbf{Q}}_m^{\rm tr} )^{-1} \mathbf{X}_m^{\rm va \top} ] \Big).
  \end{align} 
  Let $c > c_{\lambda} + \max_{m}\lambda_{m1} (1 + c_{\sigma_x}t 
+ \sqrt{c_{\lambda}/\lambda_{m1}}) $,
  it holds with probability at least $1 - 2M e^{-t}$ that
  \begin{align}
  \mathrm{Tr}(\mathbf{C}_2^{\rm ma}) \leq &
   \mathrm{Tr}(\mathbf{C}_1^{\rm ma})
  c^2 {\alpha^2} \frac{N_{\rm va}}{N_{\rm tr}}, 
  \quad \text{and} \quad \mathrm{Tr}(\mathbf{C}_2^{\rm im}) \leq 
   \mathrm{Tr}(\mathbf{C}_1^{\rm im})
  c^2 \frac{N_{\rm va}}{N_{\rm tr}} .
  \end{align}
\end{lemma}

\begin{proof}
We can derive $\mathrm{Tr} (\mathbf{C}_2^{\rm ma}) $ by
\begin{align}\label{eq:bound_C_ma_2_derive1}
  & \mathrm{Tr} (\mathbf{C}_2^{\rm ma}) 
  = \frac{\alpha^2}{N_{\rm tr}} \mathrm{Tr}(\mathbf{C}_1 \mathrm{diag}[\mathbf{X}_m^{\rm va} \hat{\mathbf{Q}}_{m}^{\rm tr} \mathbf{X}_m^{\rm va \top} ]) 
  \stackrel{(a)}{\leq} 
  \frac{\alpha^2}{N_{\rm tr}} \mathrm{Tr} (\mathbf{C}_1^{\rm ma}) 
  \big\|\mathrm{diag}[\mathbf{X}_m^{\rm va} \hat{\mathbf{Q}}_{m}^{\rm tr} \mathbf{X}_m^{\rm va \top} ] \big\| \nonumber \\
  \stackrel{(b)}{=}  
  & \frac{\alpha^2}{N_{\rm tr}} \mathrm{Tr} (\mathbf{C}_1^{\rm ma}) 
  \max_m \big\|\mathbf{X}_m^{\rm va} \hat{\mathbf{Q}}_{m}^{\rm tr} \mathbf{X}_m^{\rm va \top}  \big\| 
  \stackrel{(c)}{\leq}  
  \frac{\alpha^2}{N_{\rm tr}} \mathrm{Tr} (\mathbf{C}_1^{\rm ma}) 
  \max_m \big \|\hat{\mathbf{Q}}_{m}^{\rm tr} \big \|
  \big\|\mathbf{X}_m^{\rm va} \mathbf{X}_m^{\rm va \top}  \big\| 
\end{align}
where $(a)$ follows from Lemma~\ref{lm:von_trace_ineq}, $(b)$ follows because the largest eigenvalue of a symmetric block diagonal matrix is the maximum largest eigenvalue of the block matrices, $(c)$ follows because for any unit vector $\bu$, 
$\bu^{\top} \mathbf{X}_m^{\rm va} \hat{\mathbf{Q}}_{m}^{\rm tr} \mathbf{X}_m^{\rm va \top} \bu
\leq \big \|\hat{\mathbf{Q}}_{m}^{\rm tr} \big \|
\bu^{\top} \mathbf{X}_m^{\rm va} \mathbf{X}_m^{\rm va \top} \bu
\leq \big \|\hat{\mathbf{Q}}_{m}^{\rm tr} \big \|
  \big\|\mathbf{X}_m^{\rm va} \mathbf{X}_m^{\rm va \top}  \big\| $.

Then because $ \big\|\mathbf{X}_m^{\rm va} \mathbf{X}_m^{\rm va \top}  \big\|
= \big\| \mathbf{X}_m^{\rm va \top}  \mathbf{X}_m^{\rm va} \big\|
= N_{\rm va} \big\| \hat{\mathbf{Q}}_m^{\rm va} \big\|$.
The bound on $\big\| \hat{\mathbf{Q}}_m^{\rm tr} \big\|$ and $\big\| \hat{\mathbf{Q}}_m^{\rm va} \big\|$ can be obtained by Lemma~\ref{lm:bound_Q_hat_norm}.
Applying the union bound over $\bZ^{\rm tr}$ and $\bZ^{\rm va}$, we have
that there exists a constant $c>0$ that depends on $\sigma_x$ such that, for all $t \geq 1$, with probability at least $1-2e^{-t}$
\begin{align*}
  \|\hat{\mathbf{Q}}_m^{\rm tr}\|
  \leq &
   \lambda_{m1} +
  c  \lambda_{m1} \max \Bigg\{ \sqrt{\frac{{r}(\mathbf{Q}_m)}{N_{\rm tr}}}, \frac{{r}(\mathbf{Q}_m)}{N_{\rm tr}}, \sqrt{\frac{t}{N_{\rm tr}}}, \frac{t}{N_{\rm tr}} \Bigg\}, \\
  \text{and}~~\|\hat{\mathbf{Q}}_m^{\rm va}\|
  \leq &
   \lambda_{m1} +
  c  \lambda_{m1} \max \Bigg\{ \sqrt{\frac{{r}(\mathbf{Q}_m)}{N_{\rm va}}}, \frac{{r}(\mathbf{Q}_m)}{N_{\rm va}}, \sqrt{\frac{t}{N_{\rm va}}}, \frac{t}{N_{\rm va}} \Bigg\} .
\end{align*}

Then applying the union bound over $M$ tasks, we have that there exists a constant $c_{\sigma_x} > 0$ that depends on $\sigma_x$,
and $c > c_{\lambda} + \max_{m}\lambda_{m1} (1 + c_{\sigma_x}t 
+ \sqrt{c_{\lambda}/\lambda_{m1}}) $ 
such that, for all $t \geq 1$, with probability at least $1-2M e^{-t}$
\begin{align}\label{eq:Q_hat_norm_bound_C2}
  \max_m \big \|\hat{\mathbf{Q}}_{m}^{\rm tr} \big \|
  \big\|\mathbf{X}_m^{\rm va} \mathbf{X}_m^{\rm va \top}  \big\| 
  \leq & 
  c^2 N_{\rm va} .
\end{align}

Combining the above results with \eqref{eq:bound_C_ma_2_derive1}  completes the proof for MAML.

Similarly, for iMAML, we have
\begin{align*}\label{eq:bound_C_ma_2_derive2}
  & \mathrm{Tr} (\mathbf{C}_2^{\rm im}) 
  = \frac{1}{N_{\rm tr}}\mathrm{Tr}\Big(\mathbf{C}_1^{\rm im}
     \mathrm{diag}[\mathbf{X}_m^{\rm va} (\mathbf{I} + \gamma^{-1} \hat{\mathbf{Q}}_m^{\rm tr} )^{-1}
     \hat{\mathbf{Q}}_m^{\rm tr} (\mathbf{I} + \gamma^{-1} \hat{\mathbf{Q}}_m^{\rm tr} )^{-1} \mathbf{X}_m^{\rm va \top} ] \Big) \\
  {\leq} &
  \frac{1}{N_{\rm tr}} \mathrm{Tr} (\mathbf{C}_1^{\rm im}) 
  \big\|\mathrm{diag}[\mathbf{X}_m^{\rm va} (\mathbf{I} + \gamma^{-1} \hat{\mathbf{Q}}_m^{\rm tr} )^{-1}
     \hat{\mathbf{Q}}_m^{\rm tr} (\mathbf{I} + \gamma^{-1} \hat{\mathbf{Q}}_m^{\rm tr} )^{-1} \mathbf{X}_m^{\rm va \top} ] \big\|  \\
  {=}  
  & \frac{1}{N_{\rm tr}} \mathrm{Tr} (\mathbf{C}_1^{\rm im}) 
  \max_m \big\|\mathbf{X}_m^{\rm va} (\mathbf{I} + \gamma^{-1} \hat{\mathbf{Q}}_m^{\rm tr} )^{-1}
     \hat{\mathbf{Q}}_m^{\rm tr} (\mathbf{I} + \gamma^{-1} \hat{\mathbf{Q}}_m^{\rm tr} )^{-1} \mathbf{X}_m^{\rm va \top}  \big\| \\
  {\leq}  &
  \frac{1}{N_{\rm tr}} \mathrm{Tr} (\mathbf{C}_1^{\rm im}) 
  \max_m \big \|(\mathbf{I} + \gamma^{-1} \hat{\mathbf{Q}}_m^{\rm tr} )^{-1}
     \hat{\mathbf{Q}}_m^{\rm tr} (\mathbf{I} + \gamma^{-1} \hat{\mathbf{Q}}_m^{\rm tr} )^{-1} \big \|
  \big\|\mathbf{X}_m^{\rm va} \mathbf{X}_m^{\rm va \top}  \big\| \\
  \leq & \frac{1}{N_{\rm tr}} \mathrm{Tr} (\mathbf{C}_1^{\rm im}) 
  \max_m \big \| \hat{\mathbf{Q}}_m^{\rm tr}  \big \|
  \big\|\mathbf{X}_m^{\rm va} \mathbf{X}_m^{\rm va \top}  \big\| .
  \numberthis
\end{align*}
Combining the above results with \eqref{eq:Q_hat_norm_bound_C2} on the same high probability event for $\mathbf{Z}$ completes the proof for iMAML.
\end{proof}

Lemma~\ref{lm:C2_in_C1} shows that $\mathrm{Tr}(\mathbf{C}_2)$ can be bounded as $\boldsymbol{\Theta}(\mathrm{Tr}(\mathbf{C}_1))$.
Then we proceed to bound $\mathrm{Tr}(\mathbf{C}_1)$. In Lemma~\ref{lm:decompose_tr_C1}, we decompose $\mathrm{Tr}(\mathbf{C}_1)$  into terms that are related to the first $k$ largest eigenvalues of $\mathbf{W}$ and the term that is only related to the rest eigenvalues of $\mathbf{W}$.

\begin{lemma}[Bound of $\mathrm{Tr}(\mathbf{C}_1)$ in terms of $\bar{\bX}$]
\label{lm:bound_C_1_permute}
  Recall  $\mathrm{Tr}(\mathbf{C}_1)$ and $\bar{\bX}$ is computed by
  \begin{align*}
    \mathrm{Tr}(\mathbf{C}_1)&
    = 
    \mathrm{Tr}\big(
     \tilde{\mathbf{X}}
        {\mathbf{W}}
        \tilde{\mathbf{X}}^{\top}
     \mathbf{A}^{-2}
     \big),
    ~~\text{and}~~\bar{\bX} = 
  [\mathbf{Z}_{m}^{\rm va} 
    \bar{\bLam}_{m} \bV_{m}^{\top} ]_{m } 
  \end{align*}
 Then we have with high probability
 \begin{align*}
   \mathrm{Tr}(\mathbf{C}_1)
   \leq &
   c \mathrm{Tr}\big(
     \bar{\mathbf{X}}
        {\mathbf{W}}
        \bar{\mathbf{X}}^{\top}
     \mathbf{A}^{-2}
     \big) .
 \end{align*}

\end{lemma}

\begin{proof}
\label{proof:bound_C_1_permute}

By Lemma~\ref{lm:von_trace_ineq} and the properties of trace, we have
\begin{align*}
   \mathrm{Tr}(\mathbf{C}_1)
   = &
   \mathrm{Tr}\big(
     \bar{\mathbf{X}}
        {\mathbf{W}}
        \bar{\mathbf{X}}^{\top}
     \mathbf{A}^{-2}
     \big)
     + \mathrm{Tr}\big(
             {\mathbf{W}}
      (\tilde{\mathbf{X}}- \bar{\mathbf{X}})^{\top}
     \mathbf{A}^{-2}(\tilde{\mathbf{X}} + \bar{\mathbf{X}})
     \big) \\
   \leq &
   \mathrm{Tr}\big( \mathbf{A}^{-2}
       \bar{\bX}
     \bW
    \bar{\bX}^{\top}
       \big) 
       + \mathrm{Tr}(\bW)
   \|(\tilde{\mathbf{X}}- \bar{\mathbf{X}})^{\top}
     \mathbf{A}^{-2}(\tilde{\mathbf{X}} + \bar{\mathbf{X}}) \| \\
     \leq &
   \mathrm{Tr}\big( \mathbf{A}^{-2}
       \bar{\bX}
     \bW
    \bar{\bX}^{\top}
       \big) 
       + \mathrm{Tr}(\bW) \mu_n^{-2}(\bA)
   \| \tilde{\mathbf{X}}- \bar{\mathbf{X}} \|
    \| \tilde{\mathbf{X}} + \bar{\mathbf{X}} \| \\
    \leq  &
   \mathrm{Tr}\big( \mathbf{A}^{-2}
       \bar{\bX}
     \bW
    \bar{\bX}^{\top}
       \big) 
       + \mathrm{Tr}(\bW) \mu_n^{-2}(\bA)
   \| \tilde{\mathbf{X}}- \bar{\mathbf{X}} \|
    \big(2 \| \bar{\mathbf{X}}\| + \| \tilde{\mathbf{X}} - \bar{\mathbf{X}} \|\big).
 \end{align*}
where $\|\tilde{\mathbf{X}}- \bar{\mathbf{X}}\|$
  is bounded by Lemma~\ref{lm:singularvalues_diff_tilde_bar}
  and $\| \bar{\mathbf{X}}\|$ is bounded by Lemma~\ref{lm:bound_Q_hat_norm}, which can be controlled by choosing proper hyperparameters $\gamma$ and $\alpha$ to make the first term dominate.
\end{proof}

\begin{lemma}[Decomposition of $\mathrm{Tr}\big(
     \bar{\mathbf{X}}
        {\mathbf{W}}
        \bar{\mathbf{X}}^{\top}
     \mathbf{A}^{-2}
     \big)$ in $\mathrm{Tr}(\bC_1)$]
\label{lm:decompose_tr_C1}
Recall $\tilde{\bX} = [\mathbf{Z}_{m}^{\rm va} 
    \tilde{\bLam}_{m} \bP_{m}]$,
    $\bar{\bX} = [\mathbf{Z}_{m}^{\rm va} 
    \bar{\bLam}_{m} \bV_{m}]$,
    $\bar{\bX}_{\rm P} = [\mathbf{Z}_{m}^{\rm va} 
    \bar{\bLam}_{m} \bP_{m}]$.
Define $\mathbf{A} = \tilde{\mathbf{X}}
  \tilde{\mathbf{X}}^{\top}$, and 
  $\bX_{\rm P} = [\mathbf{Z}_{m}^{\rm va} 
    \bar{\bLam}_{m} \bP_{m} ]$.
For both MAML and iMAML, $\mathrm{Tr}\big(
     \bar{\mathbf{X}}
        {\mathbf{W}}
        \bar{\mathbf{X}}^{\top}
     \mathbf{A}^{-2}
     \big)$ in $\mathrm{Tr}(\bC_1)$ can be bounded by 
  \begin{align*}
    &\mathrm{Tr}\big(
     \bar{\mathbf{X}}
        {\mathbf{W}}
        \bar{\mathbf{X}}^{\top}
     \mathbf{A}^{-2}
     \big)
     \leq 
    c\mathrm{Tr}\Big(
     (\bar{\mathbf{X}}_{\rm P}
     \bLam_{W,0:k}
     \bar{\mathbf{X}}_{\rm P}^{\top}
        + \bar{\mathbf{X}}
        \bV_{W,k:d} \bLam_{W,k:d} \bV_{W,k:d}^{\top}
        \bar{\mathbf{X}}^{\top}
      )
     \mathbf{A}^{-2}
     \Big) .
  \end{align*}
\end{lemma}

\begin{proof}

Recall the singular value decomposition of $\mathbf{W}$ as $ \mathbf{W}= \bV_W \mathbf{\Lambda}_W \bV_W^{\top}$, then for any $0\leq k \leq d$,
$\mathbf{W}$ can be computed by
\begin{align*}\label{eq:W_decompose}
  \mathbf{W}& = 
  \bV_{W, 0:k} \mathbf{\Lambda}_{W, 0:k} \bV_{W, 0:k}^{\top}
  +\bV_{W, k:d} \mathbf{\Lambda}_{W, k:d} \bV_{W, k:d}^{\top}. 
  \numberthis
\end{align*}

Therefore we have
\begin{align*}
  &\mathrm{Tr}\big(
   \bar{\mathbf{X}}
      {\mathbf{W}}
      \bar{\mathbf{X}}^{\top}
   \mathbf{A}^{-2}
   \big)
   = 
  \mathrm{Tr}\Big(
   (\bar{\mathbf{X}}
   \bV_{W,0:k} \bLam_{W,0:k} \bV_{W,0:k}^{\top}
   \bar{\mathbf{X}}^{\top}
      + \bar{\mathbf{X}}
      \bV_{W,k:d} \bLam_{W,k:d} \bV_{W,k:d}^{\top}
      \bar{\mathbf{X}}^{\top}
    )
   \mathbf{A}^{-2}
   \Big) 
\end{align*}
where $\bar{\mathbf{X}}
\bV_{W,0:k} \bLam_{W,0:k} \bV_{W,0:k}^{\top}
\bar{\mathbf{X}}^{\top}$ can be further decomposed by
\begin{align*}
  &\bar{\mathbf{X}}
\bV_{W,0:k} \bLam_{W,0:k} \bV_{W,0:k}^{\top}
   = \bar{\mathbf{X}}_{\rm P}
   \bLam_{W,0:k}
   \bar{\mathbf{X}}_{\rm P}^{\top}
       \\
   & + [\mathbf{Z}_{m}^{\rm va} 
   \bar{\bLam}_{m} (\bV_{m}\bV_{W,0:k} - \bP_{m, 0:k})] 
   \bLam_{W,0:k} 
   [\mathbf{Z}_{m}^{\rm va} 
   \bar{\bLam}_{m} (\bV_{m}\bV_{W,0:k} - \bP_{m, 0:k})]^{\top}.
\end{align*}
By Lemma~\ref{lm:von_trace_ineq}, we have the last term can be bounded by
\begin{align*}
    &\mathrm{Tr}\big(\mathbf{A}^{-2} [\mathbf{Z}_m^{\rm va} 
  \bar{\mathbf{\Lambda}}_{m} (\bV_m^{\top} \bV_{W,0:k} + \bP_{m,0:k}) ]
   \mathbf{\Lambda}_{W,0:k} 
  [\mathbf{Z}_m^{\rm va} 
  \bar{\mathbf{\Lambda}}_{m} (\bV_m^{\top} \bV_{W,0:k} - \bP_{m,0:k}) ]^{\top}
  \big) \\
  \leq & 
  \mathrm{Tr}\big(   \mathbf{\Lambda}_{W,0:k}     \big)
    \mu_n(\mathbf{A})^{-2}
   \| [\mathbf{Z}_m^{\rm va} 
  \bar{\mathbf{\Lambda}}_{m} (\bV_m^{\top} \bV_{W,0:k} + \bP_{m,0:k}) ] \|
  \|[\mathbf{Z}_m^{\rm va} 
  \bar{\mathbf{\Lambda}}_{m} (\bV_m^{\top} \bV_{W,0:k} - \bP_{m,0:k}) ]^{\top}\|
 \end{align*}
 where $\| [\mathbf{Z}_m^{\rm va} 
  \bar{\mathbf{\Lambda}}_{m} (\bV_m^{\top} \bV_{W,0:k} + \bP_{m,0:k}) ] \|$ can be further bounded with high probability by
  \begin{align*}
    &\| [\mathbf{Z}_m^{\rm va} 
  \bar{\mathbf{\Lambda}}_{m} (\bV_m^{\top} \bV_{W,0:k} + \bP_{m,0:k}) ] \| \\
    = &
    \big\| [\mathbf{Z}_m^{\rm va} 
  \bar{\mathbf{\Lambda}}_{m} (\bV_m^{\top} \bV_{W,0:k} + \bP_{m,0:k}) ]^{\top} [\mathbf{Z}_m^{\rm va} 
  \bar{\mathbf{\Lambda}}_{m} (\bV_m^{\top} \bV_{W,0:k} + \bP_{m,0:k}) ] \big\|^{\frac{1}{2}}\\
    =  & 
    \Big \| \sum_{m=1}^{M}
    (\bV_m^{\top} \bV_{W,0:k} + \bP_{m,0:k})^{\top}
    \bar{\mathbf{\Lambda}}_{m}^{\top}
    \mathbf{Z}_m^{{\rm va}{\top}} \mathbf{Z}_m^{\rm va} 
  \bar{\mathbf{\Lambda}}_{m} (\bV_m^{\top} \bV_{W,0:k} + \bP_{m,0:k})
   \Big \|^{\frac{1}{2}} \\
   \lesssim &
    \sqrt{N_{\rm va}}\Big( 
   \sum_{m=1}^{M}\mathrm{Tr}\big(\bW_m \big)\Big) ^{\frac{1}{2}}
  \end{align*}
  where the last inequality follows from Lemma~\ref{lm:concentrate_sample_cov}.

  Similarly, $\| [\mathbf{Z}_m^{\rm va} 
  \bar{\mathbf{\Lambda}}_{m} (\bV_m^{\top} \bV_{W,0:k} - \bP_{m,0:k}) ] \|$ can be further bounded with high probability by
  \begin{align*}
    &\| [\mathbf{Z}_m^{\rm va} 
  \bar{\mathbf{\Lambda}}_{m} (\bV_m^{\top} \bV_{W,0:k} - \bP_{m,0:k}) ] \| \\
    = &
    \big\| [\mathbf{Z}_m^{\rm va} 
  \bar{\mathbf{\Lambda}}_{m} (\bV_m^{\top} \bV_{W,0:k} - \bP_{m,0:k}) ]^{\top} [\mathbf{Z}_m^{\rm va} 
  \bar{\mathbf{\Lambda}}_{m} (\bV_m^{\top} \bV_{W,0:k} - \bP_{m,0:k}) ] \big\|^{\frac{1}{2}}\\
    =  & 
    \Big \| \sum_{m=1}^{M}
    (\bV_m^{\top} \bV_{W,0:k} - \bP_{m,0:k})^{\top}
    \bar{\mathbf{\Lambda}}_{m}^{\top}
    \mathbf{Z}_m^{{\rm va}{\top}} \mathbf{Z}_m^{\rm va} 
  \bar{\mathbf{\Lambda}}_{m} (\bV_m^{\top} \bV_{W,0:k} - \bP_{m,0:k})
   \Big \|^{\frac{1}{2}} \\
  \lesssim &
   \sqrt{M N_{\rm va} } \max_m  
   \mathrm{Tr}^{\frac{1}{2}}\big(\bW_m \big) \|\bV_m^{\top}\bV_{W,0:k} - \bP_{m,0:k} \|.
  \end{align*}
  Based on the assumption the last term is smaller compared to the rest terms.
\end{proof}

Then we bound the term related to the $d-k$ smallest eigenvalues of $\mathbf{W}$ as a function of $\mu_n(\mathbf{A})$, given in Lemma~\ref{lm:last_terms}.
\begin{lemma}[Bound on $\mathrm{Tr}(\bar{\mathbf{X}}\bV_{W,k:d}\bLam_{W,k:d}\bV_{W,k:d}^{\top}
\bar{\mathbf{X}}^{\top}
\mathbf{A}^{-2})$ in $\mathrm{Tr}(\mathbf{C}_1)$]
\label{lm:last_terms}
With probability at least $1 - e^{-t}$ over $\mathbf{Z}$, and for $c \geq t$, it holds that
  \begin{align*}
    \mathrm{Tr}(\bar{\mathbf{X}}
      \bV_{W,k:d}\bLam_{W,k:d}\bV_{W,k:d}^{\top}
      \bar{\mathbf{X}}^{\top}
      \mathbf{A}^{-2}) 
      &\leq 
      c M N_{\rm va} \mu_n^{-2}(\mathbf{A})
      \sum_{i>k} \mu_i^2(\mathbf{W})
  \end{align*}
  {where $\mu_n$ is the smallest eigenvalue of a matrix.}
\end{lemma}

\begin{proof}
By Von Neumann's trace inequality in Lemma~\ref{lm:von_trace_ineq}, 
$\mathrm{Tr}(\bar{\mathbf{X}}
      \bV_{W,k:d}\bLam_{W,k:d}\bV_{W,k:d}^{\top}
      \bar{\mathbf{X}}^{\top}
      \mathbf{A}^{-2}) $ is bounded by
  \begin{align*}
    \mathrm{Tr}(\bar{\mathbf{X}}
      \bV_{W,k:d}\bLam_{W,k:d}\bV_{W,k:d}^{\top}
      \bar{\mathbf{X}}^{\top}
      \mathbf{A}^{-2}) 
     \leq \mathrm{Tr}(\bV_{W,k:d}\bLam_{W,k:d}\bV_{W,k:d}^{\top}
     \bar{\mathbf{X}}^{\top}
     \bar{\mathbf{X}})
     \mu_n^{-2}(\mathbf{A}) .
  \end{align*}

To bound $\mathrm{Tr}(\bV_{W,k:d}\bLam_{W,k:d}\bV_{W,k:d}^{\top}
     \bar{\mathbf{X}}^{\top}
     \bar{\mathbf{X}})$, 
we first rewrite it as
\begin{align*}
  &\mathrm{Tr}(\bV_{W,k:d}\bLam_{W,k:d}\bV_{W,k:d}^{\top}
     \bar{\mathbf{X}}^{\top}
     \bar{\mathbf{X}}) 
 = MN_{\rm va}
 \mathrm{Tr}\Big( (\bV_{W,k:d}\bLam_{W,k:d}\bV_{W,k:d}^{\top})^2 \Big) \\
  &+ \mathrm{Tr}\Big( \bV_{W,k:d}\bLam_{W,k:d}\bV_{W,k:d}^{\top}
  (\bar{\mathbf{X}}^{\top}
  \bar{\mathbf{X}} - M N_{\rm va} \bV_{W,k:d}\bLam_{W,k:d}\bV_{W,k:d}^{\top})
   \Big)\\
  \label{eq:trace_last}
  {=}&
  MN_{\rm va}
 \mathrm{Tr}\Big( \bLam_{W,k:d}^2 \Big) 
  + \Big\| 
  \bar{\mathbf{X}}\bV_{W,k:d}\bLam_{W,k:d}^{\frac{1}{2}} \Big\|_{\rm F}^2
  - M N_{\rm va} \Big\| \bV_{W,k:d}\bLam_{W,k:d}\bV_{W,k:d}^{\top}
   \Big\|_{\rm F}^2\\ 
  {=}&MN_{\rm va}
 \Big( \sum_{i>k} \mu_i^2(\mathbf{W}) \Big) 
  +\underbracket {\big\| 
  \bar{\mathbf{X}}\bV_{W,k:d}\bLam_{W,k:d}^{\frac{1}{2}} \big\|_{\rm F}^2
  - \mathbb{E}\Big[\big\| 
  \bar{\mathbf{X}}\bV_{W,k:d}\bLam_{W,k:d}^{\frac{1}{2}} \big\|_{\rm F}^2\Big]}_{I_1}
   \numberthis
\end{align*}
where the last equation follows because
\begin{align*}
  &\mathbb{E}\Big[\big\| 
  \bar{\mathbf{X}}\bV_{W,k:d}\bLam_{W,k:d}^{\frac{1}{2}} \big\|_{\rm F}^2\Big]
  =\mathbb{E}\Big[\mathrm{Tr}(\bV_{W,k:d}\bLam_{W,k:d}\bV_{W,k:d}^{\top} \bar{\bX}^{\top}\bar{\bX})\Big] \\
  =& MN_{\rm va}\mathbb{E}\Big[\mathrm{Tr}(\bV_{W,k:d}\bLam_{W,k:d}\bV_{W,k:d}^{\top} \mathbf{W})\Big]
  = MN_{\rm va}\mathbb{E}\Big[\mathrm{Tr}\big((\bV_{W,k:d}\bLam_{W,k:d}\bV_{W,k:d}^{\top})^2 \big)\Big] \\
  =& M N_{\rm va} \Big\| \bV_{W,k:d}\bLam_{W,k:d}\bV_{W,k:d}^{\top}
  \Big\|_{\rm F}^2.
\end{align*}

Let $\bar{\bx}_{m,n}$ be the $n$-th row of $\bar{\bX}_m$, $I_1$ can be further bounded with probability at least $1-e^{-t}$ by
\begin{align*}
  |I_1|
  = & MN_{\rm va} \Bigg|\frac{1}{MN_{\rm va}}\sum_{m=1}^M\sum_{n=1}^{N_{\rm va}}
  \big\| 
  \bar{\bx}_{m,n}\bV_{W,k:d}\bLam_{W,k:d}^{\frac{1}{2}} \big\|_{\rm F}^2
  - \mathbb{E}\Big[\big\| 
  \bar{\bx}_{m,n}\bV_{W,k:d}\bLam_{W,k:d}^{\frac{1}{2}} \big\|_{\rm F}^2\Big]\Bigg| \\
  \leq & MN_{\rm va} \Big\| \bV_{W,k:d}\bLam_{W,k:d}\bV_{W,k:d}^{\top}
  \Big\|_{\rm F}^2 
  \max \Big\{ \sqrt{\frac{t}{MN_{\rm va}}}, \frac{t}{MN_{\rm va}} \Big\}
\end{align*}
where the last inequality follows because $\| 
\bar{\bx}_{m,n}\bV_{W,k:d}\bLam_{W,k:d}^{\frac{1}{2}} \|_{\rm F}^2$ are  sub-exponential for $m \in [M], n\in [N_{\rm va}]$.

Also because $\| \bV_{W,k:d}\bLam_{W,k:d}\bV_{W,k:d}^{\top}
\|_{\rm F}^2 = \mathrm{Tr}(\bLam_{W,k:d}^2) = \sum_{i>k} \mu_i^2(\mathbf{W})$, we have with probability at least $1 - e^{-t}$
\begin{align*}
  \mathrm{Tr}(\bar{\mathbf{X}}
      \bV_{W,k:d}\bLam_{W,k:d}\bV_{W,k:d}^{\top}
      \bar{\mathbf{X}}^{\top}
      \mathbf{A}^{-2}) 
      &\leq 
      c M N_{\rm va} \mu_n^{-2}(\mathbf{A})
      \sum_{i>k} \mu_i^2(\mathbf{W}).
\end{align*}
This completes the proof.
\end{proof}

Next we bound the term related to the $k$ largest eigenvalues of $\mathbf{W}$, given in Lemma~\ref{lm:first_k_terms}.
\begin{lemma}[Bound on terms in $\mathrm{Tr}(\mathbf{C}_1)$ related to the first $k$ eigenvalues]
\label{lm:first_k_terms}
Recall 
\begin{align*}
  &\bar{\mathbf{X}}_{\rm P}
= [\mathbf{Z}_{m}^{\rm va} 
    \bar{\bLam}_{m} \bP_{m} 
   ]_{m},\quad
  \bar{\bX}_{{\rm P}, 0:k} \coloneqq
[\mathbf{Z}_{m}^{\rm va} 
    \bar{\bLam}_{m} \bP_{m, 0:k} 
   ]_{m}
\end{align*}
There exists $c$ with $0\leq k \leq c$ such that 
 with probability at least $1-2e^{MN_{\rm va}/c}$, the following holds
  \begin{align*}
    \mathrm{Tr}\big( \bar{\mathbf{X}}_{{\rm P}, 0:k} \bLam_{W, 0:k} \bar{\mathbf{X}}_{{\rm P}, 0:k}^{\top}
     {\bA}^{-2} \big)
     &\leq
     \frac{c k}{MN_{\rm va}}. 
  \end{align*}
\end{lemma}

\begin{proof}
Recall $\bar{\bLam}_{{\rm P}, m} = \bP_{m}^{\top}
    \bar{\bLam}_{m} \bP_{m}$, 
$\bar{\bX}_{{\rm P} }$ and $\bar{\bX}_{{\rm P}, 0:k}$ can be written as
\begin{align}
  \bar{\bX}_{{\rm P}}
  = & [\mathbf{Z}_{m}^{\rm va} 
  \bP_{m} \bP_{m}^{\top}
    \bar{\bLam}_{m} \bP_{m} 
   ]_{m}
   = [\mathbf{Z}_{{\rm P}, m}^{\rm va}
    \bar{\bLam}_{{\rm P}, m} 
   ]_{m}, \quad\quad
   \bar{\mathbf{X}}_{{\rm P}, 0:k}
  = 
   [\mathbf{Z}_{{\rm P}, m, 0:k}^{\rm va}
    \bar{\bLam}_{{\rm P}, m,  0:k} 
   ]_{m} .
\end{align}

Derive $\mathrm{Tr}\big( \bar{\mathbf{X}}_{{\rm P}, 0:k} \bLam_{W, 0:k} \bar{\mathbf{X}}_{{\rm P}, 0:k}^{\top} {\bA}^{-2} \big)$ as follows
\begin{align*}
  & \mathrm{Tr}\big( \bar{\mathbf{X}}_{{\rm P}, 0:k} \bLam_{W, 0:k} \bar{\mathbf{X}}_{{\rm P}, 0:k}^{\top} {\bA}^{-2} \big)
  = \mathrm{Tr}
  \big( [\mathbf{Z}_{{\rm P}, m, 0:k}^{\rm va}
    \bar{\bLam}_{{\rm P}, m, 0:k} 
   ]_{m} 
   \bLam_{W, 0:k}
   [\mathbf{Z}_{{\rm P}, m, 0:k}^{\rm va}
    \bar{\bLam}_{ {\rm P}, m, 0:k} 
   ]_{m}^{\top} {\bA}^{-2} \big) \\
   =& \sum_{i=1}^{k}
   \lambda_{W,i} 
   [\mathbf{z}_{{\rm P}, m, i}^{\rm va}
    \bar{\lambda}_{{\rm P}, m, i} 
   ]_{m}^{\top}
   {\bA}^{-2}
   [\mathbf{z}_{{\rm P}, m, i}^{\rm va}
    \bar{\lambda}_{{\rm P}, m, i} 
   ]_{m} 
   = \sum_{i=1}^{k}
   \lambda_{W,i} 
   \bar{\bx}_{{\rm P}, i} ^{\top}
   {\bA}^{-2}
   \bar{\bx}_{{\rm P}, i}  
\end{align*}

Based on Lemma~\ref{lm:zzt_A}, 
let ${\bA}_{-j} = {\bA} - \bar{\bx}_{{\rm P}, j} \bar{\bx}_{{\rm P}, j}^{\top} \succ 0$,
we have
\begin{align*}
 & \bar{\bx}_{{\rm P}, j} ^{\top}
   \bA^{-2}
   \bar{\bx}_{{\rm P}, j}   
= \bar{\bx}_{{\rm P}, j}^{\top}
( \bar{\bx}_{{\rm P}, j} \bar{\bx}_{{\rm P}, j}^{\top}
+ {\bA}_{-j})^{-2}
\bar{\bx}_{{\rm P}, j}  
= \frac{\bar{\bx}_{{\rm P}, j}^{\top} 
{\bA}_{-j}^{-2}\bar{\bx}_{{\rm P}, j}}
{(1+ \bar{\bx}_{{\rm P}, j}^{\top}
{\bA}_{-j}^{-1}\bar{\bx}_{{\rm P}, j})^2}\\
\leq &
\frac{\bar{\bx}_{{\rm P}, j}^{\top}
{\bA}_{-j}^{-2}\bar{\bx}_{{\rm P}, j}}
{( \bar{\bx}_{{\rm P}, j}^{\top}
{\bA}_{-j}^{-1}\bar{\bx}_{{\rm P}, j})^2}
\leq 
\frac{\mu_{n}^{-2}({\bA}_{-j})
\|\bar{\bx}_{{\rm P}, j}\|^2}
{\mu_{k+1}^{-2}({\bA}_{-j})
\|\Pi_{\mathscr{L}_j} \bar{\bx}_{{\rm P}, j} \|^4}
\end{align*}
where by Lemma~\ref{lm:z_piz_norm_bound}, there exists $c_{z1}$ that, with probability at least $1-3e^{-t}$, it holds that 
\begin{align}
  \|\bar{\bx}_{{\rm P}, j}\|^2 = 
  \sum_{m=1}^{M} 
    \bar{\lambda}_{{\rm P}, m, i} ^2
   \|\mathbf{z}_{{\rm P}, m, i}^{\rm va}\|^2
   \leq \sum_{m=1}^{M} \bar{\lambda}_{{\rm P}, m, i} ^2 
   \big(N_{\rm va}+a \sigma_x^{2}(t+\sqrt{N_{\rm va} t})\big)
   \leq c_{z1} N_{\rm va} \sum_{m=1}^{M} 
    \bar{\lambda}_{{\rm P}, m, i} ^2.
\end{align}

And $\mathscr{L}_j$ is the span of the $MN_{\rm va}-k$ eigenvectors 
with the smallest eigenvalues of ${\bA}_{-j}$, and 
$\Pi_{\mathscr{L}_j}$ represents the projection to $\mathscr{L}_j$.
Let $M = \Pi_{\mathscr{L}_j^{\perp }}^{\top} \Pi_{\mathscr{L}_j^{\perp}} $. 
By Lemma~\ref{lm:z_piz_norm_bound}, with probability at least $1-3e^{-t}$, it holds that 
\begin{align*}
  \|\Pi_{\mathscr{L}_j^{\perp}} \bar{\bx}_{{\rm P},j} \|^2
  = \bar{\bx}_{{\rm P},j}^{\top} M \bar{\bx}_{{\rm P},j}
  \leq c_{z1} (2k + 4t) c_P  \frac{1}{M} \sum_{m=1}^{M} 
    \bar{\lambda}_{{\rm P}, m, i} ^2 . 
\end{align*}

Therefore 
\begin{align*}
  & \|\Pi_{\mathscr{L}_j} \bar{\bx}_{{\rm P},j} \|^2
  =\|\bar{\bx}_{{\rm P},j} \|^2 
  - \|\Pi_{\mathscr{L}_j{\perp}} \bar{\bx}_{{\rm P},j} \|^2 \\
  \geq & c_{z1} (M N_{\rm va} - (2k + 4t) c_P) \frac{1}{M} \sum_{m=1}^{M} 
    \bar{\lambda}_{{\rm P}, m, i} ^2
  \geq (M N_{\rm va}/ c_{z2}) \frac{1}{M} \sum_{m=1}^{M} 
    \bar{\lambda}_{{\rm P}, m, i} ^2
\end{align*}

Since ${\bA}_{-j} = {\bA} - \bar{\bx}_{{\rm P}, j} \bar{\bx}_{{\rm P}, j}^{\top} \preceq {\bA}$, which, combined with Lemma~\ref{lm:weyl}, leads to
$\mu_{k+1}({\bA}_{-j}) < \mu_{k+1}({\bA}) = \mu_1({\bA}_{k})$. 

Since
$\mu_n(\mathbf{A}_{-j}) \geq \mu_n (\mathbf{A}_k)$ 
, we have
\begin{align*}
  \bar{\bx}_{{\rm P}, j} ^{\top}
   \bA^{-2}
   \bar{\bx}_{{\rm P}, j}  
  \leq 
  \frac{\mu_{n}^{-2}({\bA}_{-j})
\|\bar{\bx}_{{\rm P}, j}\|^2}
{\mu_{k+1}^{-2}({\bA}_{-j})
\|\Pi_{\mathscr{L}_j} \bar{\bx}_{{\rm P}, j} \|^4}
  \stackrel{(a)}{\leq} 
  c_1 \frac{\mu_n({\bA}_{ k}) }
  {\mu_1({\bA}_{ k}) MN_{\rm va}}
  \stackrel{(b)}{\leq} 
  c_2 \frac{1}{ M N_{\rm va} }
\end{align*}
where $(a)$ is because $\mu_n({\bA}_{-j}) \geq  \mu_n ({\bA}_{k})$ 
and $\mu_{k+1}({\bA}_{-j}) <  \mu_1({\bA}_{k})$.
And $(b)$ is from Lemma~\ref{lm:eigenv_A}.
\end{proof}

Finally in Lemma~\ref{lm:eigenv_A}, we bound the eigenvalues of $\mathbf{A}$ to complete the bound on the term related to the $d-k$ smallest eigenvalues of $\mathbf{W}$.
\begin{lemma}[Bound on eigenvalues of $\mathbf{A}$]
\label{lm:eigenv_A}
Recall that 
  \begin{align*}
    \mathbf{A}& = \tilde{\mathbf{X}} \tilde{\mathbf{X}}^{\top}
  = [\mathbf{Z}_{m_1}^{\rm va} 
    \tilde{\bLam}_{m_1} \bV_{m_1}^{\top} 
    \bV_{m_2} \tilde{\bLam}_{m_2}^{\top} 
    \mathbf{Z}_{m_2}^{{\rm va} {\top} } ]_{m_1 m_2} \\
  \bar{\bA} &= \bar{\mathbf{X}} \bV_{W}\bV_{W}^{\top} \bar{\mathbf{X}}^{\top}
= [\mathbf{Z}_{m_1}^{\rm va} 
\bar{\bLam}_{m_1} \bV_{m_1}^{\top} 
\bV_{W}\bV_{W}^{\top}
\bV_{m_2} \bar{\bLam}_{m_2} 
\mathbf{Z}_{m_2}^{{\rm va} {\top} } ]_{m_1 m_2} \\
\bar{\bA}_{\rm P} &= \bar{\mathbf{X}}_{\rm P} \bar{\mathbf{X}}_{\rm P}^{\top}
= [\mathbf{Z}_{m_1}^{\rm va} 
    \bar{\bLam}_{m_1} \bP_{m_1} 
    \bP_{m_2}^{\top} \bar{\bLam}_{m_2} 
    \mathbf{Z}_{m_2}^{{\rm va} {\top} } ]_{m_1 m_2}
  . \nonumber
  \end{align*}
Let $\mu_i(\cdot)$ denote  the $i$-th largest eigenvalue of a matrix, 
and let $n = MN_{\rm va}$.
Define $\overline{\mathbf{W}}_{{\rm P},M} \coloneqq
\frac{1}{M} \sum_{m=1}^{M} \bP_m^{\top} \bar{\bLam}_m^2 \bP_m $,
$\bar{\bA}_{{\rm P}, k} \coloneqq
\bar{\mathbf{X}}_{{\rm P}, k:d} \bar{\mathbf{X}}_{{\rm P}, k:d}^{\top} $,
$\overline{\mathbf{W}}_{{\rm P},M,k} \coloneqq
\frac{1}{M} \sum_{m=1}^{M} \bP_{m,k:d}^{\top} \bar{\bLam}_m^2 \bP_{m,k:d} $.
Then 
there exists  constants $b,c\geq 1, c_0 \geq 0$ that 
if $r_0(\overline{\bW}_{M,k}) \geq b MN_{\rm va}$,
with probability at least $1 - 2e^{-MN_{\rm va}/c}$
\begin{align}\label{eq:mu_A}
&\mu_n(\mathbf{A}) 
  \geq  \mu_n(\bar{\bA})  - c_0
  \geq  \mu_n(\bar{\bA}_{k}) - c_0
  \geq \frac{1}{c} \mu_{1}({\mathbf{W}}_{k}) r_0({\mathbf{W}}_{k}) \\
  &\mu_1(\bar{\bA}_{k}) 
  \leq c \mu_{1}(\bW_k) r_0(\bW_k)\\
  &\mu_n(\mathbf{A}) 
  \geq  \mu_n(\bar{\bA}_{\rm P})  - 2c_0
  \geq  \mu_n(\bar{\bA}_{{\rm P},k}) - 2c_0
  \geq \frac{1}{c} \mu_{1}(\bW_k) r_0(\bW_k) \\
  &\mu_1(\bar{\bA}_{{\rm P},k})  
  \leq \mu_1(\bar{\bA}_{k}) + c_0
  \leq c \mu_{1}(\bW_k) r_0(\bW_k).
\end{align}

\end{lemma}

\begin{proof}

First $\bA$ can be written as
\begin{align*}
  \bA = \bar{\bA} + \tilde{\mathbf{X}} \tilde{\mathbf{X}}^{\top} - \bar{\mathbf{X}} \bar{\mathbf{X}}^{\top}
  = \bar{\mathbf{A}}_{\rm P} +
  \bar{\mathbf{X}} \bar{\mathbf{X}}^{\top} - \bar{\mathbf{X}}_{\rm P} \bar{\mathbf{X}}_{\rm P}^{\top} +
  \tilde{\mathbf{X}} \tilde{\mathbf{X}}^{\top} - \bar{\mathbf{X}} \bar{\mathbf{X}}^{\top}.
\end{align*}
Therefore 
\begin{align*}
\bar{\bA}_{\rm P} - 2c_0 \mathbf{I} \preceq 
\bar{\bA} - c_0 \mathbf{I} \preceq 
\bA \preceq  \bar{\bA} + c_0 \mathbf{I} 
\preceq \bar{\bA}_{\rm P} + 2c_0 \mathbf{I}
\end{align*}
and
 $ c_0 = \max\{\|\tilde{\mathbf{X}} \tilde{\mathbf{X}}^{\top} - \bar{\mathbf{X}} \bar{\mathbf{X}}^{\top} \| 
  , \|\bar{\mathbf{X}} \bar{\mathbf{X}}^{\top} - \bar{\mathbf{X}}_{\rm P} \bar{\mathbf{X}}_{\rm P}^{\top}\|\} $,
where $\|\tilde{\mathbf{X}} \tilde{\mathbf{X}}^{\top} - \bar{\mathbf{X}} \bar{\mathbf{X}}^{\top} \|$
can be bounded by 
\begin{align*}
  \|\tilde{\mathbf{X}} \tilde{\mathbf{X}}^{\top} - \bar{\mathbf{X}} \bar{\mathbf{X}}^{\top} \|
  \leq &
  \|(\tilde{\mathbf{X}} + \bar{\mathbf{X}}) 
  (\tilde{\mathbf{X}} - \bar{\mathbf{X}})^{\top} \|
  \leq 
  \| \tilde{\mathbf{X}}- \bar{\mathbf{X}} \|
    \big(2 \| \bar{\mathbf{X}}\| + \| \tilde{\mathbf{X}} - \bar{\mathbf{X}} \|\big).
 \end{align*}
where $\|\tilde{\mathbf{X}}- \bar{\mathbf{X}}\|$
  is bounded by Lemma~\ref{lm:singularvalues_diff_tilde_bar}
  and $\| \bar{\mathbf{X}}\|$ is bounded by Lemma~\ref{lm:bound_Q_hat_norm}.

For sufficiently small $|\alpha|$ and $\gamma^{-1}$, and $c_1 > 1$, we can control
\begin{align}
  \big\| \bar{\mathbf{X}}^{\top} \bar{\mathbf{X}} -\tilde{\mathbf{X}}^{\top} \tilde{\mathbf{X}}  \big \|
  \leq \frac{1}{c_1}\mu_{1}(\bW_k) r_0(\bW_k)
\end{align}

Furthermore, by the bounded task heterogeneity assumption, 
$\|\bar{\mathbf{X}} \bar{\mathbf{X}}^{\top} - \bar{\mathbf{X}}_{\rm P} \bar{\mathbf{X}}_{\rm P}^{\top}\|$ can be bounded as
\begin{align*}
  \|\bar{\mathbf{X}} \bar{\mathbf{X}}^{\top} - \bar{\mathbf{X}}_{\rm P} \bar{\mathbf{X}}_{\rm P}^{\top}\| 
  = & \| (\bar{\mathbf{X}} + \bar{\mathbf{X}}_{\rm P}) ^{\top}
  (\bar{\mathbf{X}} - \bar{\mathbf{X}}_{\rm P}) 
  \| \\
  \leq &
  \sum_{m=1}^{M}
  \|\mathbf{I}
  +\bP_m^{\top} \bV_m^{\top} \bV_W  \|
  \| \bar{\bX}_{m} ^{\top} \bar{\bX}_{m} \|
  \|\mathbf{I}
  -\bV_W^{\top} \bV_m\bP_m \| \\
  \leq & 2 M \max_{m} \|\bP_m^{\top} 
  - \bV_W^{\top} \bV_m \|
  \| \bar{\bX}_{m} ^{\top} \bar{\bX}_{m} \|
  \leq \frac{1}{c_1} \mu_{1}(\bW_k) r_0(\bW_k) .
\end{align*}
Then we have there exists $c_1 > 1$ that 
\begin{align}\label{eq:c_0_bound}
  c_0 \leq \frac{1}{c_1} \mu_{1}(\bW_k) r_0(\bW_k).
\end{align}

Next we bound $\|\bar{\bA} \|$ and $\|\bar{\bA}_{\rm P} \|$. 
For $\|\bar{\bA} \|$ we have
\begin{align*}
  & \|\bar{\bA} \|
  = \|\bar{\mathbf{X}}\bV_{W}\bV_{W}^{\top} \bar{\mathbf{X}}^{\top} \| 
  = \|\bV_{W}^{\top} \bar{\mathbf{X}}^{\top} \bar{\mathbf{X}}\bV_{W}\| \\
  \leq &
   MN_{\rm va} \Big\|\bLam_{W}\Big\| 
  + \Big\| M N_{\rm va} \bLam_{W}
  - \bV_{W}^{\top} \bar{\mathbf{X}}^{\top} \bar{\mathbf{X}}\bV_{W} \Big\| \\
  \leq &
  M N_{\rm va} \mu_1(\bLam_{W})
  + \Big\| M N_{\rm va} \bLam_{W}
  - \bV_{W}^{\top} \bar{\mathbf{X}}^{\top} \bar{\mathbf{X}}\bV_{W} \Big\| .
\end{align*}
From Lemma~\ref{lm:bound_Q_hat_norm}, we have there exists a constant $c$ that with probability at least $1 - e^{-t}$
\begingroup\allowdisplaybreaks
\begin{align*}
  & \|\bar{\bA} \|
  \leq c  M N_{\rm va} \mu_1(\bLam_{W})
  +
  c MN_{\rm va} \mu_1(\bLam_{W}) c_{r_0}({r}_0(\bLam_{W}), N, t) , \\
  &  \|\bar{\bA} \|
  \geq c  M N_{\rm va} \mu_1(\bLam_{W})
  -
  c MN_{\rm va} \mu_1(\bLam_{W}) c_{r_0}({r}_0(\bLam_{W}), N, t) .
\end{align*}
\endgroup
Similarly,  because 
\begin{align*}
  \bW_k =
\bV_{W, k:d}^{\top} \bW \bV_{W, k:d} ,
~~\mathrm{Tr}(\bW_k) = \mathrm{Tr}(\bLam_{W,k:d}),
~~\mu_1(\bW_k) = \mu_1(\bLam_{W,k:d}) = \mu_{k+1}(\bW).
\end{align*}
If $r_k(\bW) \geq b MN_{\rm va}$,
then 
there exists a constant $c$ that depends on $\sigma_x$ such that
with probability at least $1 - 2e^{-MN_{\rm va}/c}$
\begin{align*}
& \|\bar{\bA}_{k} \| 
\geq \frac{1}{c} \mu_{k+1}(\bW) r_k(\bW),
~~\|\bar{\bA}_{k} \|  \leq c  \mu_{k+1}(\bW) r_k(\bW) .
\end{align*}

\end{proof}


\section{Auxiliary Lemmas} 
\label{sec:supporting_lemmas}
\subsection{Algebraic properties} 
\label{sub:basic_prop}

\begin{lemma}
  \label{lm:zzt_A}
  (Lemma~20 in \citep{Bartlett_benign_linear})
  Suppose $k<n, \bA \in \mathbb{R}^{n \times n}$ is an invertible matrix, and $\bZ \in \mathbb{R}^{n \times k}$ is such that $\bZ \bZ^{\top}+ \bA$ is invertible. Then
  \begin{align}
  \bZ^{\top}(\bZ \bZ^{\top}+ \bA)^{-2} 
  \bZ=( \bI+ \bZ^{\top} \bA^{-1} \bZ)^{-1} \bZ^{\top} \bA^{-2} \bZ(\bI+\bZ^{\top} \bA^{-1} \bZ)^{-1} .
  \end{align}
    
\end{lemma}

\begin{lemma}[Weyl's inequality~\citep{weyl1912asymptotische}]
\label{lm:weyl}
  Let $\mathbf{B}=\mathbf{A}+{\bE}, \bA, \bE$ be  $n \times n$ Hermitian matrices. Let $\mu_{i}(\cdot)$ denote the $i$-th largest eigenvalues of a matrix. Then, we have
\begin{align*}
  \mu_{i}(\mathbf{A}) + \mu_n(\bE) \leq \mu_{i}(\mathbf{B}) \leq \mu_{i}(\mathbf{A}) + \mu_1(\bE), \quad \forall i \in[n] .
\end{align*}

\end{lemma}

\begin{lemma}[Von Neumann's trace inequality~\citep{Mirsky1975}]
  \label{lm:von_trace_ineq}
  If $\bA,\bB \in \mathbb{R}^{n\times n}$.
  Let $\sigma_{i}(\cdot)$ denote the $i$-th largest singular values of a matrix.
  $\sigma_1(\bA) \geq \cdots \geq \sigma_n(\bA) , ~~
  \sigma_1(\bB) \geq \cdots \geq \sigma_n(\bB) $ respectively, then
  \begin{align}
    &|\mathrm{Tr}(\bA \bB)|
    \leq \sum_{i=1}^n \sigma_i(\bA) \sigma_i(\bB)
    \leq \sigma_1(\bB) \sum_{i=1}^n \sigma_i(\bA)  .
  \end{align}
  
\end{lemma}
\subsection{Concentration inequalities} 
\label{sub:concentrate_ineq}

\begin{lemma}
  \label{lm:seq_subGaussian}
  (Corollary~23 in \citep{Bartlett_benign_linear})
  There is a universal constant c such that for any non-increasing sequence $\{\lambda_{i}\}_{i=1}^{\infty}$ of non-negative numbers such that $\sum_{i=1}^{\infty} \lambda_{i}<\infty$, and any independent, centered, $\sigma$-subexponential random variables $\{\xi_{i}\}_{i=1}^{\infty}$, and any $t>0$, with probability at least $1-2 e^{-t}$
  \begin{align*}
    \Big|\sum_{i=1}^{\infty} \lambda_{i} \xi_{i}\Big| 
    \leq 
    c \sigma \max \Bigg\{t \lambda_{1}, \sqrt{t \sum_{i=1}^{\infty} \lambda_{i}^{2}}\Bigg\}.
  \end{align*}
\end{lemma}

\begin{lemma}
\label{lm:z_piz_norm_bound}
(Corollary~24 in \citep{Bartlett_benign_linear}) Suppose $\mathbf{z} \in \mathbb{R}^{n}$ is a centered random vector with independent $\sigma^2$-subGaussian entries with unit variances, $\mathscr{L}$ is a random subspace of $\mathbb{R}^{n}$ of codimension $k$, and $\mathscr{L}$ is independent of $\mathbf{z}$. Then for some constant $a$ and any $t>0$, with probability at least $1-3 e^{-t}$,
\begin{align*}
 \|\mathbf{z}\|^{2} \leq n+a \sigma^{2}(t+\sqrt{n t}), ~~~~~~~\|\Pi_{\mathscr{L}} \mathbf{z} \|^{2} \geq n-a \sigma^{2}(k+t+\sqrt{n t})
\end{align*}
where $\Pi_{\mathscr{L}}$ is the orthogonal projection on $\mathscr{L}$.
\end{lemma}

\begin{lemma}[Theorem 9 in~\citep{koltchinskii2017concentration}]
\label{lm:concentrate_sample_cov}
  Let $\bx, \bx_1, \ldots, \bx_n$ be i.i.d. weakly square integrable centered random vectors in   a separable Banach space with covariance $\mathbf{\Sigma}$ and sample covariance $\hat{\mathbf{\Sigma}}$. If $\mathbf{x}$ is subgaussian and pregaussian, 
  define $r(\mathbf{\Sigma}):={(\mathbb{E}[\|\mathbf{x}\|])^2}/{\|\mathbf{\Sigma}\|} $,
  then there exists a constant $c>0$ such that, for all $t \geq 1$, with probability at least $1-e^{-t}$
  \begin{align}
    \|\hat{\mathbf{\Sigma}}-\mathbf{\Sigma}\| \leq c\|\mathbf{\Sigma}\| \max \Big\{ \sqrt{\frac{{r}(\mathbf{\Sigma})}{n}}, \frac{{r}(\mathbf{\Sigma})}{n}, \sqrt{\frac{t}{n}}, \frac{t}{n} \Big\}.
  \end{align}
\end{lemma}

  

\subsection{Other supporting lemmas} 
\label{sub:other_lemma}

\begin{lemma}[Bound of $\|\tilde{\bX}^{\top}\tilde{\bX} - \bar{\bX}^{\top}\bar{\bX}\|$]
\label{lm:eigenvalues_diff_tilde_bar}
  Recall that $\bar{\mathbf{X}}^{\top} \bar{\mathbf{X}}$ and 
$\tilde{\mathbf{X}}^{\top} \tilde{\mathbf{X}}$ are computed by
\begin{align*}
  \bar{\mathbf{X}}^{\top} \bar{\mathbf{X}}
  = & N_{\mathrm{va}}\sum_{m=1}^{M}
  \bV_m \bar{\mathbf{\Lambda}}_m
  \hat{\mathbf{D}}_m^{\rm va}
  \bar{\mathbf{\Lambda}}_m \bV_m^{\top}, 
  ~\text{and}~~~~
  \tilde{\mathbf{X}}^{\top} \tilde{\mathbf{X}}
  =  N_{\mathrm{va}}\sum_{m=1}^{M}
  \bV_m \tilde{\mathbf{\Lambda}}_m^{\top}
  \hat{\mathbf{D}}_m^{\rm va}
  \tilde{\mathbf{\Lambda}}_m \bV_m^{\top}.
\end{align*}

For MAML, for $|\alpha| < \min_m \min \{1/\lambda_{m1}, 1/ \mu_1( 
  \hat{\mathbf{Q}}_m^{\rm tr} )\}$,
  and for $1 \leq t \leq N_{\rm va}$,
  there exists $c > 1$ such that
with probability at least $1 - 2M e^{-t}$ 
\begin{align*}
\big\| \bar{\mathbf{X}}^{{\rm ma}\top} \bar{\mathbf{X}}^{\rm ma}-\tilde{\mathbf{X}}^{{\rm ma}\top} \tilde{\mathbf{X}}^{\rm ma}\big \|
  \leq & c |\alpha| N_{\rm va}  \sum_{m=1}^{M} \lambda_{m1}^2
   c_{r_0}({r}_0({\mathbf{\Lambda}_m}), N_{\rm tr}, t).
\end{align*}
For iMAML, for $\gamma > 0$,
and for $1 \leq t \leq N_{\rm va}$,
  there exists $c > 1$ such that
with probability at least $1 - 2M e^{-t}$ 
\begin{align*}
\big\| \bar{\mathbf{X}}^{{\rm im}\top} \bar{\mathbf{X}}^{\rm im}-\tilde{\mathbf{X}}^{{\rm im}\top} \tilde{\mathbf{X}}^{\rm im} \big \|
  \leq & c \gamma^{-1} N_{\rm va}  \sum_{m=1}^{M} \lambda_{m1}^2
   c_{r_0}({r}_0({\mathbf{\Lambda}_m}), N_{\rm tr}, t).
\end{align*}
\end{lemma}

\begin{proof}
\label{proof:eigenvalues_diff_tilde_bar}

First we have the following relationship
 \begin{align*}
  \bar{\mathbf{X}}^{\top} \bar{\mathbf{X}}-\tilde{\mathbf{X}}^{\top} \tilde{\mathbf{X}}  
   = \frac{1}{2}
  \Big( (\bar{\mathbf{X}}+ \tilde{\mathbf{X}}) ^{\top}
  (\bar{\mathbf{X}} - \tilde{\mathbf{X}})
  + (\bar{\mathbf{X}}- \tilde{\mathbf{X}}) ^{\top}
  (\bar{\mathbf{X}} + \tilde{\mathbf{X}})
  \Big).
 \end{align*}

 Therefore we have 
 \begingroup \allowdisplaybreaks
\begin{align*}
  & 
  \big\| \bar{\mathbf{X}}^{\top} \bar{\mathbf{X}}-\tilde{\mathbf{X}}^{\top} \tilde{\mathbf{X}} \big \|
  \leq  
  \big\| (\bar{\mathbf{X}}+ \tilde{\mathbf{X}}) ^{\top}
  (\bar{\mathbf{X}} - \tilde{\mathbf{X}}) \big\|
  =   \Big\| \sum_{m=1}^{M}  \bV_m
  (\bar{\mathbf{\Lambda}}_m + \tilde{\mathbf{\Lambda}}_m) ^{\top}
  \mathbf{Z}_m^{\top} \mathbf{Z}_m
  (\bar{\mathbf{\Lambda}}_m - \tilde{\mathbf{\Lambda}}_m)
   \bV_m^{\top}  \Big\| 
\end{align*}

For MAML, we have
\begin{align*}
  & \big\| \bar{\mathbf{X}}^{{\rm ma}\top} \bar{\mathbf{X}}^{\rm ma}-\tilde{\mathbf{X}}^{{\rm ma}\top} \tilde{\mathbf{X}}^{\rm ma} \big \| \\
  \leq & N_{\rm va} \Big\|\sum_{m=1}^{M}  \bV_m
  \Big( (\mathbf{I} - \alpha \mathbf{\Lambda}_{m}) +
      (\mathbf{I}-\alpha \mathbf{\Lambda}_{m}^{\frac{1}{2}} 
  \hat{\mathbf{D}}_m^{\rm tr} 
  \mathbf{\Lambda}_{m}^{\frac{1}{2}}) \Big) 
  {\mathbf{\Lambda}}_m^{\frac{1}{2}}
  \hat{\mathbf{D}}_m^{\rm va} {\mathbf{\Lambda}}_m^{\frac{1}{2}}
  \big(\alpha {\mathbf{\Lambda}}_m^{\frac{1}{2}}
   (\hat{\mathbf{D}}_m^{\rm tr} - \mathbf{I})
   {\mathbf{\Lambda}}_m^{\frac{1}{2}} \big)
   \bV_m^{\top}\Big\|   \\
   \leq & N_{\rm va} \sum_{m=1}^{M}  
  \underbracket{\Big\| (\mathbf{I} - \alpha \mathbf{\Lambda}_{m}) + 
        (\mathbf{I}-\alpha \mathbf{\Lambda}_{m}^{\frac{1}{2}} 
    \hat{\mathbf{D}}_m^{\rm tr} 
    \mathbf{\Lambda}_{m}^{\frac{1}{2}}) \Big\|}_{I_1}
  \underbracket{\Big\| {\mathbf{\Lambda}}_m^{\frac{1}{2}}
    \hat{\mathbf{D}}_m^{\rm va} {\mathbf{\Lambda}}_m^{\frac{1}{2}}\Big\|}_{I_2}
   \underbracket{ \Big\| \alpha {\mathbf{\Lambda}}_m^{\frac{1}{2}}
   (\hat{\mathbf{D}}_m^{\rm tr} - \mathbf{I})
   {\mathbf{\Lambda}}_m^{\frac{1}{2}} 
   \Big\|}_{I_3}   
\end{align*}
where we choose $\alpha $ such that $\|\alpha \mathbf{\Lambda}_{m}^{\frac{1}{2}} 
  \hat{\mathbf{D}}_m^{\rm tr} 
  \mathbf{\Lambda}_{m}^{\frac{1}{2}}\| < 1$ and $\|\alpha \mathbf{\Lambda}_m\| < 1$.
  Therefore 
  \begin{align}
    I_1 = \big\| (\mathbf{I} - \alpha \mathbf{\Lambda}_{m}) + 
      (\mathbf{I}-\alpha \mathbf{\Lambda}_{m}^{\frac{1}{2}} 
  \hat{\mathbf{D}}_m^{\rm tr} 
  \mathbf{\Lambda}_{m}^{\frac{1}{2}}) \big\| \leq 4.
  \end{align}
  
Also based on Lemma~\ref{lm:concentrate_sample_cov} we can bound $I_2$ and $I_3$
since $\mathbf{\Lambda}_{m}^{\frac{1}{2}} 
  \hat{\mathbf{D}}_m^{\rm tr} 
  \mathbf{\Lambda}_{m}^{\frac{1}{2}}$
  and $\mathbf{\Lambda}_{m}^{\frac{1}{2}} 
  \hat{\mathbf{D}}_m^{\rm va} 
  \mathbf{\Lambda}_{m}^{\frac{1}{2}}$ are the sample covariances of $\mathbf{\Lambda}_m$.

There exists a constant $c$ that for all $t \geq 1$, with probability at least $1 - e^{-t}$ we have
\begin{align}
  I_2 = & \big\| {\mathbf{\Lambda}}_m^{\frac{1}{2}}
    \hat{\mathbf{D}}_m^{\rm va} {\mathbf{\Lambda}}_m^{\frac{1}{2}}\big\|
    \leq  \big\| \mathbf{\Lambda}_m \big\| + 
    \big\| \mathbf{\Lambda}_m - {\mathbf{\Lambda}}_m^{\frac{1}{2}}
    \hat{\mathbf{D}}_m^{\rm va} {\mathbf{\Lambda}}_m^{\frac{1}{2}} \big\| \nonumber \\
    \leq & \lambda_{m1} + 
    c\lambda_{m1} c_{r_0}({r}_0({\mathbf{\Lambda}_m}), N_{\rm va}, t) 
\end{align}
and for all $t \geq 1$, with probability at least $1 - e^{-t}$ we have
\begin{align}
  I_3 = & \big\| \alpha {\mathbf{\Lambda}}_m^{\frac{1}{2}}
   (\hat{\mathbf{D}}_m^{\rm tr} - \mathbf{I})
   {\mathbf{\Lambda}}_m^{\frac{1}{2}} 
   \big\|
    \leq 
    c |\alpha|  
    \lambda_{m1} c_{r_0}({r}_0({\mathbf{\Lambda}_m}), N_{\rm tr}, t)
\end{align}

Combining the bounds for $I_1$, $I_2$, $I_3$ and applying union bound over training and validation data for all tasks, 
when $|\alpha| < \min_m \min \{1/\lambda_{m1}, 1/ \mu_1(\mathbf{\Lambda}_{m}^{\frac{1}{2}} 
  \hat{\mathbf{D}}_m^{\rm tr} 
  \mathbf{\Lambda}_{m}^{\frac{1}{2}})\}$,
  and for $1 \leq t \leq N_{\rm va}$,
  there exists $c > 1$ such that
the following holds with probability at least $1 - 2M e^{-t}$ 
\begin{align*}
\big\| \bar{\mathbf{X}}^{{\rm ma}\top} \bar{\mathbf{X}}^{\rm ma}-\tilde{\mathbf{X}}^{{\rm ma}\top} \tilde{\mathbf{X}}^{\rm ma} \big \|
  \leq & c |\alpha| N_{\rm va}  \sum_{m=1}^{M} \lambda_{m1}^2
   c_{r_0}({r}_0({\mathbf{\Lambda}_m}), N_{\rm tr}, t).
\end{align*}
\endgroup

Similarly, for iMAML, we have
\begin{align*}
  & \big\| \bar{\mathbf{X}}^{{\rm im}\top} \bar{\mathbf{X}}^{\rm im}-\tilde{\mathbf{X}}^{{\rm im}\top} \tilde{\mathbf{X}}^{\rm im} \big \| 
  \leq \Big\| \sum_{m=1}^{M}  \bV_m
  (\bar{\mathbf{\Lambda}}_m + \tilde{\mathbf{\Lambda}}_m) ^{\top}
  \mathbf{Z}_m^{\top} \mathbf{Z}_m
  (\bar{\mathbf{\Lambda}}_m - \tilde{\mathbf{\Lambda}}_m)
   \bV_m^{\top}  \Big\| \\
  = & N_{\rm va} \Big\|\sum_{m=1}^{M}  \bV_m
  \Big( (\mathbf{I} + \gamma^{-1} \mathbf{\Lambda}_{m})^{-1} + 
      (\mathbf{I}+ \gamma^{-1} \mathbf{\Lambda}_{m}^{\frac{1}{2}} 
  \hat{\mathbf{D}}_m^{\rm tr} 
  \mathbf{\Lambda}_{m}^{\frac{1}{2}})^{-1} \Big) 
  {\mathbf{\Lambda}}_m^{\frac{1}{2}}
  \hat{\mathbf{D}}_m^{\rm va} {\mathbf{\Lambda}}_m^{\frac{1}{2}} \\
  &\quad \cdot \big((\mathbf{I} + \gamma^{-1} \mathbf{\Lambda}_{m})^{-1} 
  (\gamma^{-1} {\mathbf{\Lambda}}_m^{\frac{1}{2}}
   (\hat{\mathbf{D}}_m^{\rm tr} - \mathbf{I})
   {\mathbf{\Lambda}}_m^{\frac{1}{2}}  )
      (\mathbf{I}+ \gamma^{-1} \mathbf{\Lambda}_{m}^{\frac{1}{2}} 
  \hat{\mathbf{D}}_m^{\rm tr} 
  \mathbf{\Lambda}_{m}^{\frac{1}{2}})^{-1} \big)
   \bV_m^{\top}\Big\|   \\
   \leq & N_{\rm va} \sum_{m=1}^{M}  
  \underbracket{\Big\| (\mathbf{I} + \gamma^{-1} \mathbf{\Lambda}_{m})^{-1} + 
      (\mathbf{I}+ \gamma^{-1} \mathbf{\Lambda}_{m}^{\frac{1}{2}} 
  \hat{\mathbf{D}}_m^{\rm tr} 
  \mathbf{\Lambda}_{m}^{\frac{1}{2}})^{-1} \Big\|}_{I_4}
  \underbracket{\Big\| {\mathbf{\Lambda}}_m^{\frac{1}{2}}
    \hat{\mathbf{D}}_m^{\rm va} {\mathbf{\Lambda}}_m^{\frac{1}{2}}\Big\|}_{I_2} \\
  & \quad \cdot \underbracket{ \Big\| (\mathbf{I} + \gamma^{-1} \mathbf{\Lambda}_{m})^{-1} 
  (\gamma^{-1} {\mathbf{\Lambda}}_m^{\frac{1}{2}}
   (\hat{\mathbf{D}}_m^{\rm tr} - \mathbf{I})
   {\mathbf{\Lambda}}_m^{\frac{1}{2}}  )
      (\mathbf{I}+ \gamma^{-1} \mathbf{\Lambda}_{m}^{\frac{1}{2}} 
  \hat{\mathbf{D}}_m^{\rm tr} 
  \mathbf{\Lambda}_{m}^{\frac{1}{2}})^{-1}
   \Big\|}_{I_5}   
\end{align*}
where $I_4$ can be bounded by
\begin{align}\label{eq:I_4_imaml_bound}
  I_4 & = \Big\| (\mathbf{I} + \gamma^{-1} \mathbf{\Lambda}_{m})^{-1} + 
      (\mathbf{I}+ \gamma^{-1} \mathbf{\Lambda}_{m}^{\frac{1}{2}} 
  \hat{\mathbf{D}}_m^{\rm tr} 
  \mathbf{\Lambda}_{m}^{\frac{1}{2}})^{-1} \Big\| 
  \leq 2.
\end{align}
And $I_5$ can be bounded by
\begin{align*}\label{eq:I_5_imaml_bound}
  I_5  =& \Big\| (\mathbf{I} + \gamma^{-1} \mathbf{\Lambda}_{m})^{-1} 
  (\gamma^{-1} {\mathbf{\Lambda}}_m^{\frac{1}{2}}
   (\hat{\mathbf{D}}_m^{\rm tr} - \mathbf{I})
   {\mathbf{\Lambda}}_m^{\frac{1}{2}}  )
      (\mathbf{I}+ \gamma^{-1} \mathbf{\Lambda}_{m}^{\frac{1}{2}} 
  \hat{\mathbf{D}}_m^{\rm tr} 
  \mathbf{\Lambda}_{m}^{\frac{1}{2}})^{-1}
   \Big\| \\
   \leq & 
   \Big\| (\mathbf{I} + \gamma^{-1} \mathbf{\Lambda}_{m})^{-1} \Big\|
  \Big\|\gamma^{-1} {\mathbf{\Lambda}}_m^{\frac{1}{2}}
   (\hat{\mathbf{D}}_m^{\rm tr} - \mathbf{I})
   {\mathbf{\Lambda}}_m^{\frac{1}{2}} \Big\|
    \Big\|  (\mathbf{I}+ \gamma^{-1} \mathbf{\Lambda}_{m}^{\frac{1}{2}} 
  \hat{\mathbf{D}}_m^{\rm tr} 
  \mathbf{\Lambda}_{m}^{\frac{1}{2}})^{-1}
   \Big\| \\
   \leq &
   \Big\|\gamma^{-1} {\mathbf{\Lambda}}_m^{\frac{1}{2}}
   (\hat{\mathbf{D}}_m^{\rm tr} - \mathbf{I})
   {\mathbf{\Lambda}}_m^{\frac{1}{2}} \Big\|
   \numberthis
\end{align*}

Based on Lemma~\ref{lm:concentrate_sample_cov}, we can bound $I_5$ similarly as $I_3$. 
There exists a constant $c$ that for all $t \geq 1$, with probability at least $1 - e^{-t}$ we have
\begin{align}
  I_5 
  \leq & \Big\|\gamma^{-1} {\mathbf{\Lambda}}_m^{\frac{1}{2}}
   (\hat{\mathbf{D}}_m^{\rm tr} - \mathbf{I})
   {\mathbf{\Lambda}}_m^{\frac{1}{2}} \Big\|
   \leq c \gamma^{-1}
   \lambda_{m1} c_{r_0}({r}_0({\mathbf{\Lambda}_m}), N_{\rm tr}, t) .
\end{align}

Combining the bounds for $I_4$, $I_2$, $I_5$ and applying union bound over training and validation data for all tasks, 
for $\gamma > 0$, and for $1 \leq t \leq N_{\rm va}$,
  there exists $c > 1$ such that
the following holds with probability at least $1 - 2M e^{-t}$ 
\begin{align*}
\big\| \bar{\mathbf{X}}^{{\rm im}\top} \bar{\mathbf{X}}^{\rm im}-\tilde{\mathbf{X}}^{{\rm im}\top} \tilde{\mathbf{X}}^{\rm im} \big \|
  \leq & c \gamma^{-1} N_{\rm va}  \sum_{m=1}^{M} \lambda_{m1}^2
   c_{r_0}({r}_0({\mathbf{\Lambda}_m}), N_{\rm tr}, t).
\end{align*}
This completes the proof for Lemma~\ref{lm:eigenvalues_diff_tilde_bar}.
\end{proof}

\begin{lemma}[Bound of $\|\tilde{\bX} - \bar{\bX}\|$]
\label{lm:singularvalues_diff_tilde_bar}
  Recall that $ \bar{\mathbf{X}}$ and 
$ \tilde{\mathbf{X}}$ are defined as
\begin{align*}
  \bar{\mathbf{X}}
  = & [{\mathbf{Z}}_m^{\rm va}
    \bar{\mathbf{\Lambda}}_m \bV_m^{\top}], 
  ~\text{and}~~~~
  \tilde{\mathbf{X}}
  =  [{\mathbf{Z}}_m^{\rm va}
    \tilde{\mathbf{\Lambda}}_m \bV_m^{\top}].
\end{align*}
Define $c_{r_0}({r}({\mathbf{\Lambda}}), N, t) \coloneqq \max \big\{ \sqrt{\frac{{r}(\mathbf{\mathbf{\Lambda}})}{N}}, \frac{{r}({\mathbf{\Lambda}})}{N}, \sqrt{\frac{t}{N}}, \frac{t}{N} \big\}$.
For MAML, and for $1 \leq t \leq N_{\rm va}$,
  there exists $c > 1$ such that
with probability at least $1 - 2M e^{-t}$ 
\begin{align*}
\big\| \bar{\mathbf{X}}^{\rm ma}-\tilde{\mathbf{X}}^{\rm ma}\big \|
  \leq & c |\alpha| N_{\rm va}  \Big(\sum_{m=1}^{M} \lambda_{m1}^2
     c_{r_0}^2({r}({\mathbf{\Lambda}_m}), N_{\rm tr}, t)\Big)^{\frac{1}{2}}.
\end{align*}
For iMAML, and for $\gamma > 0$,
and for $1 \leq t \leq N_{\rm va}$,
  there exists $c > 1$ such that
with probability at least $1 - 2M e^{-t}$ 
\begin{align*}
\big\| \bar{\mathbf{X}}^{\rm im}-\tilde{\mathbf{X}}^{\rm im} \big \|
  \leq & c \gamma^{-1} N_{\rm va}  \Big(\sum_{m=1}^{M} \lambda_{m1}^2
   c_{r_0}^2({r}({\mathbf{\Lambda}_m}), N_{\rm tr}, t)\Big)^{\frac{1}{2}}.
\end{align*}
\end{lemma}

\begin{proof}
\label{proof:singularvalues_diff_tilde_bar}

First we have the following relationship
 \begin{align*}
  \|\tilde{\mathbf{X}} - \bar{\mathbf{X}} \|
   = 
  \big\| (\bar{\mathbf{X}}- \tilde{\mathbf{X}}) ^{\top}
  (\bar{\mathbf{X}} - \tilde{\mathbf{X}}) \big\|^{\frac{1}{2}}.
 \end{align*}

 \begingroup \allowdisplaybreaks

For MAML, we have
\begin{align*}
  & \big\| (\bar{\mathbf{X}}^{\rm ma}- \tilde{\mathbf{X}}^{\rm ma}) ^{\top}
  (\bar{\mathbf{X}}^{\rm ma} - \tilde{\mathbf{X}}^{\rm ma}) \big \| \\
  = & N_{\rm va} \Big\|\sum_{m=1}^{M}  \bV_m
  \big(\alpha {\mathbf{\Lambda}}_m^{\frac{1}{2}}
   (\hat{\mathbf{D}}_m^{\rm tr} - \mathbf{I})
   {\mathbf{\Lambda}}_m^{\frac{1}{2}} \big)
  {\mathbf{\Lambda}}_m^{\frac{1}{2}}
  \hat{\mathbf{D}}_m^{\rm va} {\mathbf{\Lambda}}_m^{\frac{1}{2}}
  \big(\alpha {\mathbf{\Lambda}}_m^{\frac{1}{2}}
   (\hat{\mathbf{D}}_m^{\rm tr} - \mathbf{I})
   {\mathbf{\Lambda}}_m^{\frac{1}{2}} \big)
   \bV_m^{\top}\Big\|   \\
   \leq & N_{\rm va} \sum_{m=1}^{M}  
  \underbracket{ \Big\| \alpha {\mathbf{\Lambda}}_m^{\frac{1}{2}}
   (\hat{\mathbf{D}}_m^{\rm tr} - \mathbf{I})
   {\mathbf{\Lambda}}_m^{\frac{1}{2}} 
   \Big\|}_{I_1} 
  \underbracket{\Big\| {\mathbf{\Lambda}}_m^{\frac{1}{2}}
    \hat{\mathbf{D}}_m^{\rm va} {\mathbf{\Lambda}}_m^{\frac{1}{2}}\Big\|}_{I_2}
   \underbracket{ \Big\| \alpha {\mathbf{\Lambda}}_m^{\frac{1}{2}}
   (\hat{\mathbf{D}}_m^{\rm tr} - \mathbf{I})
   {\mathbf{\Lambda}}_m^{\frac{1}{2}} 
   \Big\|}_{I_1}   
\end{align*}

where based on Lemma~\ref{lm:concentrate_sample_cov} we can bound $I_1$ and $I_2$
since $\mathbf{\Lambda}_{m}^{\frac{1}{2}} 
  \hat{\mathbf{D}}_m^{\rm tr} 
  \mathbf{\Lambda}_{m}^{\frac{1}{2}}$
  and $\mathbf{\Lambda}_{m}^{\frac{1}{2}} 
  \hat{\mathbf{D}}_m^{\rm va} 
  \mathbf{\Lambda}_{m}^{\frac{1}{2}}$ are the sample covariances of $\mathbf{\Lambda}_m$.

There exists a constant $c$ that for all $t \geq 1$, with probability at least $1 - e^{-t}$ we have
\begin{align}
  I_2 = & \big\| {\mathbf{\Lambda}}_m^{\frac{1}{2}}
    \hat{\mathbf{D}}_m^{\rm va} {\mathbf{\Lambda}}_m^{\frac{1}{2}}\big\|
    \leq  \big\| \mathbf{\Lambda}_m \big\| + 
    \big\| \mathbf{\Lambda}_m - {\mathbf{\Lambda}}_m^{\frac{1}{2}}
    \hat{\mathbf{D}}_m^{\rm va} {\mathbf{\Lambda}}_m^{\frac{1}{2}} \big\| \nonumber \\
    \leq & \lambda_{m1} + 
    c\lambda_{m1} c_{r_0}({r}_0({\mathbf{\Lambda}_m}), N_{\rm va}, t) 
\end{align}
and for all $t \geq 1$, with probability at least $1 - e^{-t}$ we have
\begin{align}
  I_1 = & \big\| \alpha {\mathbf{\Lambda}}_m^{\frac{1}{2}}
   (\hat{\mathbf{D}}_m^{\rm tr} - \mathbf{I})
   {\mathbf{\Lambda}}_m^{\frac{1}{2}} 
   \big\|
    \leq 
    c |\alpha|  
    \lambda_{m1} c_{r_0}({r}_0({\mathbf{\Lambda}_m}), N_{\rm tr}, t)
\end{align}

Combining the bounds for $I_1$, $I_2$ and applying union bound over training and validation data for all tasks, 
we have for $1 \leq t \leq N_{\rm va}$,
  there exists $c > 1$ such that
the following holds with probability at least $1 - 2M e^{-t}$ 
\begin{align*}
\big\| (\bar{\mathbf{X}}^{\rm ma}- \tilde{\mathbf{X}}^{\rm ma}) ^{\top}
  (\bar{\mathbf{X}}^{\rm ma} - \tilde{\mathbf{X}}^{\rm ma}) \big \|  
  \leq & c |\alpha|^2 N_{\rm va}  \sum_{m=1}^{M} \lambda_{m1}^3
   c^2_{\rm sample}({r}({\mathbf{\Lambda}_m}), N_{\rm tr}, t).
\end{align*}
\endgroup

Similarly, for iMAML, we have
\begin{align*}
  & \big\| (\bar{\mathbf{X}}^{\rm im}- \tilde{\mathbf{X}}^{\rm im}) ^{\top}
  (\bar{\mathbf{X}}^{\rm im} - \tilde{\mathbf{X}}^{\rm im}) \big \|
  = \Big\| \sum_{m=1}^{M}  \bV_m
  (\bar{\mathbf{\Lambda}}_m - \tilde{\mathbf{\Lambda}}_m) ^{\top}
  \mathbf{Z}_m^{\top} \mathbf{Z}_m
  (\bar{\mathbf{\Lambda}}_m - \tilde{\mathbf{\Lambda}}_m)
   \bV_m^{\top}  \Big\| \\
  = & N_{\rm va} \Big\|\sum_{m=1}^{M}  \bV_m
  \big((\mathbf{I} + \gamma^{-1} \mathbf{\Lambda}_{m})^{-1} 
  (\gamma^{-1} {\mathbf{\Lambda}}_m^{\frac{1}{2}}
   (\hat{\mathbf{D}}_m^{\rm tr} - \mathbf{I})
   {\mathbf{\Lambda}}_m^{\frac{1}{2}}  )
      (\mathbf{I}+ \gamma^{-1} \mathbf{\Lambda}_{m}^{\frac{1}{2}} 
  \hat{\mathbf{D}}_m^{\rm tr} 
  \mathbf{\Lambda}_{m}^{\frac{1}{2}})^{-1} \big) 
  {\mathbf{\Lambda}}_m^{\frac{1}{2}}
  \hat{\mathbf{D}}_m^{\rm va} {\mathbf{\Lambda}}_m^{\frac{1}{2}} \\
  &\quad \cdot \big((\mathbf{I} + \gamma^{-1} \mathbf{\Lambda}_{m})^{-1} 
  (\gamma^{-1} {\mathbf{\Lambda}}_m^{\frac{1}{2}}
   (\hat{\mathbf{D}}_m^{\rm tr} - \mathbf{I})
   {\mathbf{\Lambda}}_m^{\frac{1}{2}}  )
      (\mathbf{I}+ \gamma^{-1} \mathbf{\Lambda}_{m}^{\frac{1}{2}} 
  \hat{\mathbf{D}}_m^{\rm tr} 
  \mathbf{\Lambda}_{m}^{\frac{1}{2}})^{-1} \big)
   \bV_m^{\top}\Big\|   \\
   \leq & N_{\rm va} \sum_{m=1}^{M}  
  \underbracket{ \Big\| (\mathbf{I} + \gamma^{-1} \mathbf{\Lambda}_{m})^{-1} 
  (\gamma^{-1} {\mathbf{\Lambda}}_m^{\frac{1}{2}}
   (\hat{\mathbf{D}}_m^{\rm tr} - \mathbf{I})
   {\mathbf{\Lambda}}_m^{\frac{1}{2}}  )
      (\mathbf{I}+ \gamma^{-1} \mathbf{\Lambda}_{m}^{\frac{1}{2}} 
  \hat{\mathbf{D}}_m^{\rm tr} 
  \mathbf{\Lambda}_{m}^{\frac{1}{2}})^{-1}
   \Big\|}_{I_3}   
  \underbracket{\Big\| {\mathbf{\Lambda}}_m^{\frac{1}{2}}
    \hat{\mathbf{D}}_m^{\rm va} {\mathbf{\Lambda}}_m^{\frac{1}{2}}\Big\|}_{I_2} \\
  & \quad \quad \cdot \underbracket{ \Big\| (\mathbf{I} + \gamma^{-1} \mathbf{\Lambda}_{m})^{-1} 
  (\gamma^{-1} {\mathbf{\Lambda}}_m^{\frac{1}{2}}
   (\hat{\mathbf{D}}_m^{\rm tr} - \mathbf{I})
   {\mathbf{\Lambda}}_m^{\frac{1}{2}}  )
      (\mathbf{I}+ \gamma^{-1} \mathbf{\Lambda}_{m}^{\frac{1}{2}} 
  \hat{\mathbf{D}}_m^{\rm tr} 
  \mathbf{\Lambda}_{m}^{\frac{1}{2}})^{-1}
   \Big\|}_{I_3}   
\end{align*}
where $I_3$ can be bounded by
\begin{align*}\label{eq:I_3_imaml_bound}
  I_3  =& \Big\| (\mathbf{I} + \gamma^{-1} \mathbf{\Lambda}_{m})^{-1} 
  (\gamma^{-1} {\mathbf{\Lambda}}_m^{\frac{1}{2}}
   (\hat{\mathbf{D}}_m^{\rm tr} - \mathbf{I})
   {\mathbf{\Lambda}}_m^{\frac{1}{2}}  )
      (\mathbf{I}+ \gamma^{-1} \mathbf{\Lambda}_{m}^{\frac{1}{2}} 
  \hat{\mathbf{D}}_m^{\rm tr} 
  \mathbf{\Lambda}_{m}^{\frac{1}{2}})^{-1}
   \Big\| \\
   \leq & 
   \Big\| (\mathbf{I} + \gamma^{-1} \mathbf{\Lambda}_{m})^{-1} \Big\|
  \Big\|\gamma^{-1} {\mathbf{\Lambda}}_m^{\frac{1}{2}}
   (\hat{\mathbf{D}}_m^{\rm tr} - \mathbf{I})
   {\mathbf{\Lambda}}_m^{\frac{1}{2}} \Big\|
    \Big\|  (\mathbf{I}+ \gamma^{-1} \mathbf{\Lambda}_{m}^{\frac{1}{2}} 
  \hat{\mathbf{D}}_m^{\rm tr} 
  \mathbf{\Lambda}_{m}^{\frac{1}{2}})^{-1}
   \Big\| \\
   \leq &
   \Big\|\gamma^{-1} {\mathbf{\Lambda}}_m^{\frac{1}{2}}
   (\hat{\mathbf{D}}_m^{\rm tr} - \mathbf{I})
   {\mathbf{\Lambda}}_m^{\frac{1}{2}} \Big\|
   \numberthis
\end{align*}

Based on Lemma~\ref{lm:concentrate_sample_cov}, we can bound $I_3$ similarly as $I_1$. 
There exists a constant $c$ that for all $t \geq 1$, with probability at least $1 - e^{-t}$ we have
\begin{align}
  I_3 
  \leq & \Big\|\gamma^{-1} {\mathbf{\Lambda}}_m^{\frac{1}{2}}
   (\hat{\mathbf{D}}_m^{\rm tr} - \mathbf{I})
   {\mathbf{\Lambda}}_m^{\frac{1}{2}} \Big\|
   \leq c \gamma^{-1}
   \lambda_{m1} c_{r_0}({r}({\mathbf{\Lambda}_m}), N_{\rm tr}, t).
\end{align}

Combining the bounds for $I_2$, $I_3$ and applying union bound over training and validation data for all tasks, 
for $\gamma > 0$, and for $1 \leq t \leq N_{\rm va}$,
  there exists $c > 1$ such that
the following holds with probability at least $1 - 2M e^{-t}$ 
\begin{align*}
\big\| (\bar{\mathbf{X}}^{\rm im}- \tilde{\mathbf{X}}^{\rm im}) ^{\top}
  (\bar{\mathbf{X}}^{\rm im} - \tilde{\mathbf{X}}^{\rm im}) \big \|
  \leq & c \gamma^{-2} N_{\rm va}  \sum_{m=1}^{M} \lambda_{m1}^3
   c_{r_0}^2({r}({\mathbf{\Lambda}_m}), N_{\rm tr}, t).
\end{align*}
This completes the proof for Lemma~\ref{lm:singularvalues_diff_tilde_bar}.
\end{proof}

\begin{lemma}[Bound of $\|\bLam^{\frac{1}{2}} \bZ^{\top} \bZ \bLam^{\frac{1}{2}}\|$, $\|\bZ \bLam \bZ^{\top}\|$ and $\|\bZ \bLam^{\frac{1}{2}}\|$]
\label{lm:bound_Q_hat_norm}
  Let $\bZ \in \mathbb{R}^{N\times d}$, consists of centered, independent, $\sigma_x$-subGaussian entries.
  And $\bLam = \mathrm{diag}(\lambda_1, \dots, \lambda_d) \in \mathbb{R}^{d\times d}$ be a positive definite diagonal matrix with $\lambda_1 \geq \lambda_2 \geq \dots \geq \lambda_d$.
  Then $\|\bLam^{\frac{1}{2}} \bZ^{\top} \bZ \bLam^{\frac{1}{2}}\|=\|\bZ \bLam \bZ^{\top}\| $
  , and there exists a constant $c>0$ such that, for all $t \geq 1$, with probability at least $1-e^{-t}$
\begin{align*}
  \|\bZ \bLam \bZ^{\top}\|
  \leq &
  N \lambda_{1} +
  c N \lambda_{1} c_{r_0}({r}_0({\mathbf{\Lambda}}), N, t) ,~~
  \|\bZ \bLam^{\frac{1}{2}}\|
  \leq 
  \sqrt{N \lambda_{1}}
  \Big(1 +
    c c_{r_0}({r}_0({\mathbf{\Lambda}}), N, t)\Big)^{\frac{1}{2}}.
\end{align*}

\end{lemma}

\begin{proof}
\label{proof:bound_Q_hat_norm}
\begin{align*}
  \|\bLam^{\frac{1}{2}} \bZ^{\top} \bZ \bLam^{\frac{1}{2}}\|
  = \|\bLam^{\frac{1}{2}} \bZ^{\top} \bZ \bLam^{\frac{1}{2}} - N \bLam + N \bLam  \|
  \leq &
  N  \Big\|\frac{1}{N}\bLam^{\frac{1}{2}} \bZ^{\top} \bZ \bLam^{\frac{1}{2}} - \bLam \Big\| + N \|\bLam  \|.
\end{align*}

By Lemma~\ref{lm:concentrate_sample_cov}, we have
there exists a constant $c>0$ such that, for all $t \geq 1$, with probability at least $1-e^{-t}$
  \begin{align*}
    \Big\|\frac{1}{N}\bLam^{\frac{1}{2}} \bZ^{\top} \bZ \bLam^{\frac{1}{2}} - \bLam \Big\|
    \leq c\|\bLam\| c_{r_0}({r}_0({\mathbf{\Lambda}}), N, t).
  \end{align*}

Therefore, there exists a constant $c>0$ such that, for all $t \geq 1$, with probability at least $1-e^{-t}$
\begin{align*}
  \big\| \bLam^{\frac{1}{2}} \bZ^{\top} \bZ \bLam^{\frac{1}{2}}\big\|
  \leq &
  N \lambda_{1} +
  c N \lambda_{1} c_{r_0}({r}_0({\mathbf{\Lambda}}), N, t).
\end{align*}

Because $\|\bZ \bLam \bZ^{\top}\| = \|\bLam^{\frac{1}{2}} \bZ^{\top} \bZ \bLam^{\frac{1}{2}} \|$, and
$\|\bZ \bLam^{\frac{1}{2}}\| = \|\bLam^{\frac{1}{2}} \bZ^{\top} \bZ \bLam^{\frac{1}{2}} \|^{\frac{1}{2}}$, it leads to the conclusion.
\end{proof}

\end{document}